%% file: neurips_main.tex
\documentclass{article}
\makeatletter
\providecommand{\@trackname}{}
\makeatother

\usepackage[final]{neurips_2026}

\usepackage[utf8]{inputenc}
\usepackage[T1]{fontenc}
\usepackage{hyperref}
\usepackage{url}
\usepackage{booktabs}
\usepackage{amsfonts}
\usepackage{nicefrac}
\usepackage{microtype}
\usepackage[table]{xcolor}
\usepackage{graphicx}
\usepackage{subcaption}
\usepackage{algorithm}
\usepackage{algorithmic}
\usepackage{textcomp}
\usepackage{multirow}
\usepackage{tikz}
\usetikzlibrary{shapes.geometric, arrows.meta, positioning, fit, backgrounds, calc}
\usepackage{amsmath}
\usepackage{amssymb}
\usepackage{mathtools}
\usepackage{amsthm}
\usepackage[capitalize,noabbrev]{cleveref}
\usepackage{enumitem}

\theoremstyle{plain}
\newtheorem{theorem}{Theorem}[section]

\theoremstyle{definition}

\newtheorem{assumption}[theorem]{Assumption}
\theoremstyle{remark}

\title{RegionFed: Federated Learning for Personalized Query Understanding in Heterogeneous Retail Environments}

\author{%
  Quoc H. Nguyen\thanks{Use footnote for providing further information
    about author (webpage, alternative address)---\emph{not} for acknowledging
    funding agencies.} \\
  Walmart Global Tech\\
  Sunnyvale\\
  California, USA, \\
  \texttt{quoc.nguyen@walmart.com} \\
  \And
  Ali Lafzi \\
  Walmart Global Tech \\
  Sunnyvale\\
  California, USA, \\
  \AND
  Abhijeet Phatak \\
  Walmart Global Tech \\
  Sunnyvale\\
  California, USA, \\
  \And
  Siddharth Pratap Singh \\
  Walmart Global Tech \\
  Sunnyvale\\
  California, USA, \\
  \And
  Rohit Upadhyay \\
  Walmart Global Tech \\
  Sunnyvale\\
  California, USA, \\
  \And
  Yogananda Domlur Seetharama \\
  Walmart Global Tech \\
  Sunnyvale\\
  California, USA, \\
  \And
  Chittaranjan Tripathy \\
  Walmart Global Tech \\
  Sunnyvale\\
  California, USA, \\
}

\begin{document}

\maketitle
\begin{abstract}
Retail search systems serve diverse geographic regions with distinct query patterns, vocabularies, and product preferences, creating significant data heterogeneity that challenges both privacy-preserving training and model personalization. Federated learning offers a natural solution for privacy, but standard FL methods produce global models that sacrifice regional performance, while existing personalized FL approaches operate at the parameter level and catastrophically collapse on modern transformers (below 10\% accuracy on T5) due to tied embeddings and LayerNorm interactions. We introduce RegionFed, an \textit{architecture-robust} federated learning framework that sidesteps this failure by operating entirely at the gradient level. RegionFed uses the $\ell_2$ conflict between regional and global gradients as a unified signal that (i) diagnoses heterogeneity, (ii) routes each region to the cheapest sufficient personalization strategy, and (iii) adaptively controls personalization strength. Because it treats models as differentiable black boxes, RegionFed deploys on T5-Small, T5-3B, RoBERTa, and CNN with zero code changes, providing large gains on transformers (where parameter-level methods collapse) and consistent improvements on CNNs. Across three public datasets (Amazon ESCI, Amazon Reviews, LEAF-FEMNIST) and four architectures, RegionFed-Meta achieves 92.27\%, closing the gap to the privacy-violating centralized upper bound (Centralized + Regional Weighting: 92.04\%, $\Delta$=0.23pp, within 1$\sigma$) while providing $(\epsilon{\approx}0.60)$-differential privacy and $\mathcal{O}(1/\sqrt{T})$ convergence.
\end{abstract}

\section{Introduction}

Large-scale retail platforms serve hundreds of millions of customers across diverse geographical regions, each with distinct shopping behaviors, vocabularies, and product preferences. A critical component of the user experience is the search system's ability to accurately interpret and contextualize queries. Yet this creates a fundamental tension: effective personalization requires understanding regional patterns, while privacy regulations (GDPR, CCPA) and data sovereignty concerns prohibit centralizing sensitive query data. Federated learning (FL)~\citep{mcmahan2017communication} offers a natural paradigm for privacy-preserving distributed training, but deploying FL in heterogeneous retail environments introduces challenges that existing methods, including recent clustered~\citep{ghosh2020efficient,sattler2020clustered} and personalized approaches~\citep{t2020personalized,tan2024fedproto}, fail to fully address.

The core challenge is twofold. \textbf{First}, standard FL methods (FedAvg~\citep{mcmahan2017communication}, FedProx~\citep{li2020federated}) produce a single global model that ignores regional heterogeneity: when ``thongs'' means footwear in Australia but undergarments in the US, a global model inevitably compromises accuracy for some regions. \textbf{Second}, existing personalized FL methods operate at the parameter level through control variates or bi-level optimization; while these are architecture-agnostic in interface, they are \textit{sensitive to parameter landscape topology} and become unstable on modern transformer architectures due to shared embeddings, attention coupling, and LayerNorm interactions. Parameter-level methods like SCAFFOLD~\citep{karimireddy2020scaffold} and pFedMe~\citep{t2020personalized} achieved strong results on CNNs but struggle on high-dimensional transformer landscapes (detailed analysis in Appendix~\ref{sec:transformer_failure}). Recent methods including FedProto~\citep{tan2024fedproto} and pFedHyper~\citep{zhang2024pfedhyper} advance personalization but remain parameter-level, inheriting similar limitations. This limits the applicability of existing personalized FL to production systems that increasingly rely on transformer-based models.

This observation motivates our key insight: \textit{effective personalization should operate at the gradient level rather than the parameter level}. By treating models as black-box function approximators and personalizing based on gradient conflicts, we can achieve architecture-robust personalization that works seamlessly across model families. We personalize at the \textit{region} level rather than the individual user level because regions pool enough clients to yield stable gradient estimates, and users within the same region share systematic patterns (vocabulary, product mix, seasonal behavior) that per-client models with scarce data cannot capture.

We introduce RegionFed, a gradient-level federated learning framework for retail query understanding that operates seamlessly across model architectures. Unlike prior work~\citep{arivazhagan2019federated} requiring fixed personalization strategies, RegionFed dynamically determines optimal personalization through gradient conflict analysis, automatically adjusting regional adaptation without compromising privacy. While gradient projection~\citep{yu2020gradient} and hierarchical aggregation~\citep{ghosh2020efficient} exist independently, RegionFed provides a \textit{principled combination}: gradient conflict serves as both the heterogeneity diagnostic and the adaptive control signal within a hierarchical regional structure, enabling compute-aware routing and DP-compatible personalization in a unified framework. Our contributions are:
(1) \textbf{Gradient-conflict framework for architecture-robust personalization}: rather than manipulating model parameters directly (which fails on transformers due to tied embedding and LayerNorm feedback loops), RegionFed uses $\ell_2$ gradient conflict as the sole signal for adaptive personalization within a hierarchical regional structure. Gradient-level operations avoid instability because they determine \textit{how much} of the standard gradient to apply per region, without injecting additive parameter corrections that propagate through shared weight matrices. This treats models as differentiable black boxes, enabling deployment across architectures (T5, RoBERTa, or future designs) without code modification.
(2) \textbf{Compute-aware strategy orchestration}: the gradient-conflict signal serves as an automated triage system that routes each region to the cheapest sufficient strategy (Grad/Interp/Meta), deploying expensive meta-learning only when conflict warrants it.
(3) \textbf{DP-compatible conflict detection with concrete privacy guarantees}: DP gradient clipping ($C{=}1.0$) bounds the signal component, while the sigmoid centering via $\mu_r$ (calibrated from round~1) absorbs the noise baseline, ensuring that \textit{relative} conflict differences remain informative across rounds and regions. With $\sigma_{dp}{=}4.0$, $T{=}50$, and $\delta{=}10^{-5}$, RegionFed achieves $\epsilon \approx 0.60$ while closing the gap to centralized accuracy.
(4) \textbf{Theoretical and empirical validation}: $\mathcal{O}(1/\sqrt{T})$ convergence with explicit heterogeneity dependence (Theorem~\ref{thm:convergence_main}); 12--25pp improvement over FedAvg/FedProx on T5-Small; collapse ($<$10\%) of all four parameter-level baselines tested across $E{\in}\{1,5,10,40\}$ on transformers, confirmed transformer-specific via LEAF-FEMNIST~\citep{caldas2018leaf} where SCAFFOLD achieves 79.52\% on CNN; cross-architecture confirmation on RoBERTa-Base, T5-3B, and CNN; cross-domain validation on Amazon Reviews~\citep{ni2019justifying} (Table~\ref{tab:overall_comparison}, Panel B); sustained advantage over 50 rounds. All results averaged over 5 seeds.

\paragraph{Related Work.} FedAvg~\citep{mcmahan2017communication} and FedProx~\citep{li2020federated} produce single global models; clustered methods~\citep{ghosh2020efficient,sattler2020clustered,marfoq2021federated} group clients by similarity but lack regional adaptation. Parameter-level methods (SCAFFOLD~\citep{karimireddy2020scaffold}, pFedMe~\citep{t2020personalized}, Per-FedAvg~\citep{fallah2020personalized}, APFL~\citep{deng2020adaptive}, Ditto~\citep{li2021ditto}, FedNova~\citep{wang2020fednova}) work on CNNs but cause transformer instability; FedBN~\citep{li2021fedbn} personalizes BatchNorm but not LayerNorm. Layer-splitting methods (FedBABU~\citep{oh2022fedbabu}, FedRep~\citep{collins2021exploiting}) avoid collapse but limit personalization. Transformer-specific methods (FedTP~\citep{li2023fedtp}, FedAdapter~\citep{cai2023fedadapter}, FedLoRA~\citep{yi2023fedlora}) require architecture-specific design. Multi-task gradient methods (PCGrad~\citep{yu2020gradient}, CAGrad~\citep{liu2021conflict}) project away conflicts within a single model; RegionFed uses conflict \textit{magnitude} as a routing signal across regional models. Concurrent work~\citep{sun2024layerwise,chen2025blackbox} explores layer-wise conflict and black-box FL; RegionFed differs by using holistic gradient conflict for hierarchical regional routing. See Appendix~\ref{sec:background} for a full survey.

\section{Challenges and Problem Formulation}
\label{sec:challenges}

Federated learning converges to sub-optimal solutions under data heterogeneity: client drift degrades convergence proportionally to gradient dissimilarity~\citep{karimireddy2020scaffold}, and heterogeneous objectives create an irreducible error term~\citep{li2020federated}. Retail search amplifies this through regional semantic differences (``thongs'' means footwear in Australia but undergarments in the US), varying product availability, and seasonal/cultural timing differences. A single global model inevitably sacrifices regional accuracy, motivating region-aware personalization.

\subsection{Problem Formulation}
\label{sec:problem_formulation}
Let $\mathcal{R} = \{r_1, \ldots, r_M\}$ denote $M$ regions, each with $n_r$ clients ($N = \sum_r n_r$ total). Each client $u$ in region $r$ has private data $\mathcal{D}_u^r$. RegionFed optimizes a \textbf{two-level} (global + regional) objective over $T$ communication rounds; user-level parameters $\theta_u^{local}$ are computed via on-device local adaptation (Eq.~\ref{eq:user_level}):
\begin{equation}
\min_{\theta, \{\theta_r\}_{r=1}^M} \sum_{r=1}^{M} p_r \sum_{u \in r} \frac{|\mathcal{D}_u^r|}{|\mathcal{D}_r|} \sum_{k=1}^K \lambda_k \mathcal{L}_k(f_k(\cdot; \theta, \theta_r); \mathcal{D}_u^r)
\label{eq:objective}
\end{equation}
where $p_r = 1/M$ is the region weight (equal-sized regions in our setup), $\lambda_k$ weights task $k$ (across $K$ tasks), and $|\mathcal{D}_r| = \sum_{u \in r} |\mathcal{D}_u^r|$.

To bridge the task-level losses with the regional and user losses used in subsequent sections, we define:
\begin{align}
\mathcal{L}_u(\theta, \theta_r) &= \sum_{k=1}^{K} \lambda_k \mathcal{L}_k(f_k(\cdot; \theta, \theta_r); \mathcal{D}_u^r) \label{eq:user_loss} \\
\mathcal{L}_r(\theta, \theta_r) &= \sum_{u \in r} \frac{|\mathcal{D}_u^r|}{|\mathcal{D}_r|} \mathcal{L}_u(\theta, \theta_r) \label{eq:regional_loss}
\end{align}
so that Eq.~\ref{eq:objective} simplifies to $\min_{\theta, \{\theta_r\}} \sum_{r=1}^{M} p_r \, \mathcal{L}_r(\theta, \theta_r)$.

\section{RegionFed: System Architecture}
\label{sec:regionfed_architecture}

\subsection{Framework Architecture Overview}

Figure~\ref{fig:framework} illustrates RegionFed's three-layer hierarchy:
\textbf{(1) Global Server} maintains $\theta$ and coordinates regional aggregation;
\textbf{(2) Regional Coordinators} maintain $\theta_r$, compute adaptive weights $\alpha_r$ via gradient conflict, and optimize $\rho_r$ via golden section search;
\textbf{(3) Clients} train on private data $\mathcal{D}_u^r$ and communicate only DP-noised gradient updates. At \textit{deployment} time, clients compose personalized models via $\theta_u = \theta + \alpha_r \cdot \theta_r + \alpha_u \cdot \theta_u^{local}$ (Eq.~\ref{eq:user_level}), where $\theta_u^{local}$ is obtained via on-device local fine-tuning.

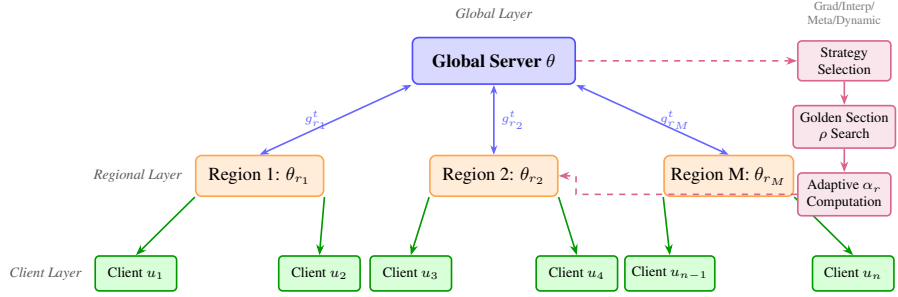
\begin{figure}[h]
\centering
\resizebox{0.85\textwidth}{!}{%
\begin{tikzpicture}[
    node distance=0.8cm and 0.6cm,
    server/.style={rectangle, draw=blue!70, fill=blue!15, thick, minimum width=2.8cm, minimum height=0.8cm, rounded corners=3pt, font=\small\bfseries},
    region/.style={rectangle, draw=orange!70, fill=orange!15, thick, minimum width=2.2cm, minimum height=0.7cm, rounded corners=3pt, font=\small},
    client/.style={rectangle, draw=green!60!black, fill=green!15, thick, minimum width=1.4cm, minimum height=0.6cm, rounded corners=2pt, font=\scriptsize},
    component/.style={rectangle, draw=purple!60, fill=purple!10, thick, minimum width=1.6cm, minimum height=0.5cm, rounded corners=2pt, font=\scriptsize},
    arrow/.style={-{Stealth[scale=0.8]}, thick},
    biarrow/.style={{Stealth[scale=0.7]}-{Stealth[scale=0.7]}, thick, blue!60},
    label/.style={font=\scriptsize\itshape, text=gray!70!black}
]
\node[server] (global) {Global Server $\theta$};
\node[label, above=0.15cm of global] {Global Layer};
\node[region, below left=1.2cm and 1.5cm of global] (r1) {Region 1: $\theta_{r_1}$};
\node[region, below=1.2cm of global] (r2) {Region 2: $\theta_{r_2}$};
\node[region, below right=1.2cm and 1.5cm of global] (rm) {Region M: $\theta_{r_M}$};
\node[label, left=0.1cm of r1] {Regional Layer};
\node[client, below left=1cm and 0.3cm of r1] (c11) {Client $u_1$};
\node[client, below right=1cm and -0.8cm of r1] (c12) {Client $u_2$};
\node[client, below left=1cm and -0.4cm of r2] (c21) {Client $u_3$};
\node[client, below right=1cm and -0.4cm of r2] (c22) {Client $u_4$};
\node[client, below left=1cm and -0.9cm of rm] (cm1) {Client $u_{n-1}$};
\node[client, below right=1cm and 0.3cm of rm] (cm2) {Client $u_n$};
\node[label, left=0.1cm of c11] {Client Layer};
\draw[biarrow] (global.south west) -- node[left, font=\tiny, pos=0.5] {$g_{r_1}^t$} (r1.north);
\draw[biarrow] (global.south) -- node[right, font=\tiny, pos=0.5] {$g_{r_2}^t$} (r2.north);
\draw[biarrow] (global.south east) -- node[right, font=\tiny, pos=0.5] {$g_{r_M}^t$} (rm.north);
\draw[arrow, green!60!black] (r1.south west) -- (c11.north);
\draw[arrow, green!60!black] (r1.south east) -- (c12.north);
\draw[arrow, green!60!black] (r2.south west) -- (c21.north);
\draw[arrow, green!60!black] (r2.south east) -- (c22.north);
\draw[arrow, green!60!black] (rm.south west) -- (cm1.north);
\draw[arrow, green!60!black] (rm.south east) -- (cm2.north);
\node[component, right=3.8cm of global, align=center] (strat) {Strategy\\Selection};
\node[component, below=0.4cm of strat, align=center] (golden) {Golden Section\\$\rho$ Search};
\node[component, below=0.4cm of golden, align=center] (alpha) {Adaptive $\alpha_r$\\Computation};
\draw[arrow, purple!60, dashed] (global.east) -- (strat.west);
\draw[arrow, purple!60] (strat.south) -- (golden.north);
\draw[arrow, purple!60] (golden.south) -- (alpha.north);
\draw[arrow, purple!60, dashed] (alpha.west) -| ($(r2.east)+(0.3,0)$) -- (r2.east);
\node[font=\tiny, above=0.1cm of strat, align=center, text=gray] {Grad/Interp/\\Meta/Dynamic};
\end{tikzpicture}
}
\caption{RegionFed architecture: three-layer hierarchy with global server, regional coordinators ($M$ regions), and local clients. Blue bidirectional arrows show gradient/model exchange; right panel shows adaptive components (strategy selection, golden section $\rho$ optimization, adaptive $\alpha_r$ weights).}
\label{fig:framework}
\end{figure}

\subsection{Hierarchical Personalization Mechanism}
\label{sec:hierarchical}
The architecture supports four personalization strategies (Grad, Interp, Meta, Dynamic) that determine how regional adaptations $\theta_r$ are computed; implementation details are in Section~\ref{sec:experimental_setup}. We now describe the hierarchical personalization mechanism.

\textbf{Gradient Computation Hierarchy.} Gradients are computed at three levels without exposing raw data. At the \textit{user level}, each client computes $\nabla \mathcal{L}_u(\theta)$ purely on-device. At the \textit{regional level}, the coordinator securely aggregates: $g_r = \sum_{u \in r} \frac{|\mathcal{D}_u^r|}{|\mathcal{D}_r|} \nabla \mathcal{L}_u(\theta)$. This raw regional gradient is then clipped and noised to produce the DP-safe gradient: $\bar{g}_r = g_r / \max(1, \|g_r\|_2 / C)$, $\tilde{g}_r = \bar{g}_r + \mathcal{N}(0, \sigma_{dp}^2 C^2 \mathbf{I})$, with $C=1.0$ and $\sigma_{dp}=4.0$. At the \textit{global level}: $\tilde{g} = \frac{1}{M} \sum_{r=1}^{M} \tilde{g}_r$. This ensures $(\epsilon, \delta)$-differential privacy; with our parameters ($T{=}50$, sampling ratio $q{=}0.1$, $\delta{=}10^{-5}$), the moments accountant yields $\epsilon \approx 0.60$ (Theorem~\ref{thm:privacy} in Appendix).

\textbf{Regional-Level Personalization.} For each region $r$, the regional model combines global parameters with region-specific adaptation:
\begin{equation}
\theta_{region}^r = \theta + \alpha_r \cdot \theta_r
\end{equation}
where $\theta$ represents global knowledge learned across all regions, $\theta_r$ captures region-specific patterns, and $\alpha_r \in [\alpha_{min}, 1.0]$ controls the personalization strength ($\alpha_{min}=0.5$ by default).

\textbf{User-Level Personalization (On-Device Adaptation).} Individual user models further specialize the regional model via local fine-tuning, computed \textbf{entirely on-device} and never transmitted:
\begin{equation}
\theta_u = \theta_{region}^r + \alpha_u \cdot \theta_u^{local}
\label{eq:user_level}
\end{equation}
where $\theta_u^{local}$ is the user-specific adaptation vector obtained by local fine-tuning (gradient updates on $\mathcal{D}_u^r$) of $\theta_{region}^r$. This step occurs post-deployment and outside federated communication rounds.

\textbf{Adaptive Weight Computation.} The personalization weights $\alpha_r$ and $\alpha_u$ are computed dynamically via gradient conflict analysis on the DP-noised gradients:
\begin{align}
\alpha_r &= \alpha_{min} + (1 - \alpha_{min}) \cdot \sigma\!\left(\frac{\|\tilde{g}_r^t - \tilde{g}^t\|_2 - \mu_r}{\tau_r}\right) \label{eq:alpha_r} \\
\alpha_u &= \alpha_{min} + (1 - \alpha_{min}) \cdot \sigma\!\left(\frac{\|\nabla \mathcal{L}_u(\theta_{region}^r) - \tilde{g}_r^t\|_2 - \mu_u}{\tau_u}\right) \label{eq:alpha_u}
\end{align}
where $\sigma(\cdot)$ is the sigmoid function, $\mu_r$ centers the sigmoid at the median gradient conflict (calibrated from round 1), and $\tau_r$ controls sensitivity. We use L2 norm rather than cosine similarity to capture both direction \textit{and magnitude} differences; DP gradient clipping ($C=1.0$) bounds all gradients within an $\ell_2$ ball of radius $C$, providing built-in normalization. When all gradients are fully clipped, $\|\tilde{g}_1 - \tilde{g}_2\|_2^2 = 2C^2(1 - \cos\phi)$, gracefully reducing to a pure directional metric.

The output is mapped to $[\alpha_{min}, 1.0]$ (default $\alpha_{min}{=}0.5$) rather than $[0, 1]$ because $\alpha{=}0$ would discard regional adaptation entirely; even well-aligned regions retain critical local semantics (ablation in Appendix~\ref{sec:e6_ablation_detailed}: removing the $\alpha_{min}$ floor degrades accuracy by 2.1pp; $\mu_r$ sensitivity: $\pm$50\% perturbation degrades accuracy by only $\pm$0.8\%).

\subsection{Personalization Strategies}
RegionFed supports four strategies for computing $\theta_r$, where adaptation intensity $\rho$ is optimized per-region via golden section search (Appendix~\ref{sec:golden_section}; all regions converge to similar values, mean $\rho{=}0.0344$):

\textbf{Grad:} $\theta_r = \rho \cdot (\tilde{g}_r^t - \tilde{g}^t)$, i.e., gradient-based adaptation using regional-global DP-noised gradient difference.

\textbf{Interp:} $\theta_r = \rho \cdot (\theta_r^{local} - \theta)$, where $\theta_r^{local}$ is obtained by fine-tuning $\theta$ on regional data $\mathcal{D}_r$ for $E$ local epochs.

\textbf{Meta:} $\theta_r = \rho \cdot \nabla_{\theta} \mathcal{L}_r(\theta - \eta \tilde{g}^t)$, i.e., MAML-style meta-learning for quick adaptation.

\textbf{Dynamic:} Automatically selects among Grad/Interp/Meta via Algorithm~\ref{alg:strategy_selection}, which routes each region to the cheapest sufficient strategy based on gradient conflict $d_r$, heterogeneity $h_r$, and dataset size $|\mathcal{D}_r|$. Table~\ref{tab:strategy_comparison} (Appendix~\ref{sec:strategy_selection_details}) summarizes the computational trade-offs.

\begin{algorithm}[h]
\small
\caption{Adaptive Strategy Selection}
\label{alg:strategy_selection}
\begin{algorithmic}
\STATE {\bfseries Input:} Region $r$, global model $\theta$, DP-noised gradients $\tilde{g}_r, \tilde{g}$, dataset size $|\mathcal{D}_r|$, mode $s \in \{$``grad'', ``interp'', ``meta'', ``dynamic''$\}$, intent distributions $P_r$ (via randomized response~\citep{wang2017locally}, $\epsilon_{LDP}{=}2.0$) and $P_{global}$
\STATE {\bfseries Output:} Selected strategy $strategy_r \in \{$``grad'', ``interp'', ``meta''$\}$
\STATE {\bfseries Hyperparameters:} $\tau_{conflict}=0.5$, $\tau_{het}=1.0$, $\beta=1.5$
\IF{$s$ = ``dynamic''}
    \STATE $d_r = \|\tilde{g}_r - \tilde{g}\|_2$; \quad $h_r = \text{KL}(P_r \| P_{global})$
    \IF{$d_r > \tau_{conflict}$ {\bfseries and} $h_r > \tau_{het}$}
        \STATE \textbf{return} ``meta'' \COMMENT{Severe drift: learn per-region LR via MAML inner loop}
    \ELSIF{$|\mathcal{D}_r| > \beta \cdot \overline{|\mathcal{D}|}$}
        \STATE \textbf{return} ``interp'' \COMMENT{Data-rich: blend global and regional gradients}
    \ELSE
        \STATE \textbf{return} ``grad'' \COMMENT{Mild conflict: lightweight additive correction}
    \ENDIF
\ELSE
    \STATE \textbf{return} $s$
\ENDIF
\end{algorithmic}
\end{algorithm}

\subsection{RegionFed Algorithm}

Algorithm~\ref{alg:regionfed} presents the complete training procedure, with notation matching Figure~\ref{fig:framework}: $\theta$ (Global Server), $\theta_r$ (Regional Coordinators), $\tilde{g}_r^t$ (DP-noised gradients exchanged via blue arrows), and $\alpha_r$ (adaptive weights from the right panel). Each round, the server broadcasts $\theta^t$; regions select a strategy (Algorithm~\ref{alg:strategy_selection}), compute DP-noised gradients, and update regional adaptations. The gradient conflict signal $\|\tilde{g}_r^t - \tilde{g}^t\|_2$ drives both \textit{what} strategy to use and \textit{how much} to personalize ($\alpha_r$). All operations use DP-noised gradients, preserving the $(\epsilon,\delta)$-DP guarantee by post-processing (Theorem~\ref{thm:privacy}).

\begin{algorithm}[h]
\small
\caption{RegionFed}
\label{alg:regionfed}
\begin{algorithmic}
\STATE {\bfseries Input:} Global model $\theta^0$, regions $\mathcal{R}=\{r_1, \ldots, r_M\}$, adaptation search space $[a, b]$
\STATE {\bfseries Output:} Global model $\theta^T$, region-specific adaptations $\{\theta_r^T\}_{r=1}^M$
\STATE Initialize global model $\theta^0$, personalization parameters $\{\theta_r^0\}_{r=1}^M$
\STATE Calibrate centering thresholds $\mu_r$ from round-1 gradient conflicts
\FOR{each round $t = 1, 2, \ldots, T$}
    \STATE Server broadcasts global model $\theta^t$ to all regions
    \FOR{each region $r \in \mathcal{R}$ {\bfseries in parallel}}
        \STATE $strategy_r \leftarrow$ \textsc{AdaptiveStrategySelection}($r$, $\theta^{t-1}$, $\tilde{g}_r^{t-1}$, $\tilde{g}^{t-1}$, $|\mathcal{D}_r|$)
        \STATE $\rho_{optimal} \leftarrow$ \textsc{GoldenSectionSearch}($r$, $\theta^t$, $strategy_r$, $[a, b]$, $\xi$) \COMMENT{Algorithm~\ref{alg:golden_section}}
        \STATE Compute final adaptation $\theta_r^{t+1}$ using $strategy_r$ and $\rho_{optimal}$
        \STATE Compute regional gradient: $g_r^t = \sum_{u \in r} \frac{|\mathcal{D}_u^r|}{|\mathcal{D}_r|} \nabla \mathcal{L}_u(\theta^t)$
        \STATE Clip gradient: $\bar{g}_r^t = g_r^t / \max(1, \|g_r^t\|_2 / C)$
        \STATE Add DP noise: $\tilde{g}_r^t = \bar{g}_r^t + \mathcal{N}(0, \sigma_{dp}^2 C^2 \mathbf{I})$
    \ENDFOR
    \STATE \textbf{/* Global Model Update */}
    \STATE $\tilde{g}^t = \frac{1}{M} \sum_{r \in \mathcal{R}} \tilde{g}_r^t$; \quad $\theta^{t+1} = \theta^t - \eta \, \tilde{g}^t$
    \STATE \textit{\% Privacy: user $\to$ region via secure aggregation; region $\to$ global via DP noise above}
    \STATE \textbf{/* Adaptive Personalization (Eqs.~\ref{eq:alpha_r}--\ref{eq:alpha_u}) */}
    \FOR{each region $r \in \mathcal{R}$}
        \STATE $\alpha_r = \alpha_{min} + (1-\alpha_{min}) \cdot \sigma\!\left(\frac{\|\tilde{g}_r^t - \tilde{g}^t\|_2 - \mu_r}{\tau_r}\right)$; \quad $\theta_{region}^r = \theta^{t+1} + \alpha_r \cdot \theta_r^{t+1}$
    \ENDFOR
    \STATE \textbf{/* User-Level Personalization (On-Device Adaptation, Post-Deployment) */}
    \FOR{each user $u$ requesting personalization}
        \STATE Download $\theta_{region}^r$; fine-tune locally to obtain $\theta_u^{local}$
        \STATE $\alpha_u = \alpha_{min} + (1-\alpha_{min}) \cdot \sigma\!\left(\frac{\|\nabla \mathcal{L}_u(\theta_{region}^r) - \tilde{g}_r^t\|_2 - \mu_u}{\tau_u}\right)$
        \STATE $\theta_u = \theta_{region}^r + \alpha_u \cdot \theta_u^{local}$
    \ENDFOR
\ENDFOR
\end{algorithmic}
\end{algorithm}

\section{Experimental Setup}
\label{sec:experimental_setup}

\subsection{Datasets}
\textbf{(1) Amazon ESCI}~\citep{reddy2022shopping} (primary benchmark): 130K real product search queries with regional partitioning by product category; 20,000 queries (16,000/4,000 train/test) across 8 regions, 80 clients (10 per region), Dirichlet $\alpha_D{=}0.3$ (70\%/30\% primary/other categories; validated against source distributions, KL$=$0.08; Appendix~\ref{sec:supplementary}). \textbf{(2) Amazon Reviews}~\citep{ni2019justifying} (cross-domain): 50,000 reviews across 5 product categories for 5-class sentiment, Dirichlet $\alpha_D{=}0.3$ across 50 clients (EMD: 0.41$\pm$0.12). \textbf{(3) LEAF-FEMNIST}~\citep{caldas2018leaf} (cross-architecture): 62-class character recognition with writer-based partitioning, 2-conv + 2-FC CNN, 200 clients in 10 regions by writer similarity. Together, these datasets span NLP and vision, enabling evaluation of gradient-conflict personalization and cross-architecture robustness. Complete statistics in Appendix~\ref{sec:supplementary}.

\subsection{Multi-Task Query Understanding}
We formulate query understanding as three complementary tasks using a unified text-to-text framework: (1) \textit{Intent Classification} (8 product categories); (2) \textit{Spell Correction} (exact match accuracy); (3) \textit{Named Entity Recognition} (BIO tagging, token-level F1). \textbf{Overall Accuracy} is the unweighted mean: $\frac{1}{3}(\text{Intent Acc} + \text{Spell Acc} + \text{NER F1})$. The \textbf{Regional Robustness Score (RRS)} measures worst-region uplift: $\text{RRS} = \frac{1}{M}\sum_{r=1}^{M} \text{Acc}_r$; higher RRS with lower inter-region standard deviation indicates more equitable performance.

\subsection{Model and Training Configuration}

\textbf{Model Architectures.} Primary: T5-Small~\citep{raffel2020exploring} (60.5M params, 6 encoder/decoder layers, 512-dim), large enough to exhibit parameter-level instabilities while tractable for federated experiments. Cross-architecture: T5-3B (3B params), RoBERTa-Base~\citep{liu2019roberta} (125M params, encoder-only), and a 2-conv + 2-FC CNN for FEMNIST (vision baseline).

\textbf{Baselines.} Centralized (privacy-violating upper bounds), Standard FL (FedAvg~\citep{mcmahan2017communication}, FedProx~\citep{li2020federated}), Region-Aware FedAvg (ablation isolating hierarchical structure from gradient-conflict adaptation), Parameter-level (SCAFFOLD~\citep{karimireddy2020scaffold}, pFedMe~\citep{t2020personalized}, Ditto~\citep{li2021ditto}, APFL~\citep{deng2020adaptive}), Layer-Splitting (FedBABU~\citep{oh2022fedbabu}), Transformer-Specific (FedTP~\citep{li2023fedtp}), Local Fine-tuning, and RegionFed variants (Grad, Interp, Meta, Dynamic). Hyperparameter details in Table~\ref{tab:overall_comparison} footnotes.

\textbf{Training Configuration.} $T{=}50$ rounds, $E{=}40$ local epochs, batch size 32, AdamW. Learning rates tuned via grid search: RegionFed best at $\eta{=}10^{-3}$ (gradient-conflict filtering stabilizes higher LR); baselines best at $\eta{=}10^{-4}$. As control, RegionFed at $\eta{=}10^{-4}$ still achieves 91.14\%, confirming gains are not LR-driven. Parameter-level collapse persists across $E{\in}\{1,5,10,40\}$ (Appendix~\ref{sec:transformer_failure}). 4$\times$A100 GPUs.

\section{Results and Analysis}

\subsection{Overall Performance Comparison}

\input{exps/E1_overall_performance/e1_comparison_table}

Table~\ref{tab:overall_comparison} presents comprehensive results. RegionFed-Meta achieves 92.27\%, within statistical noise of centralized training (91.72\%, $\Delta$=0.55pp, within 1$\sigma$), while fully preserving $(\epsilon{\approx}0.60, \delta{=}10^{-5})$-DP.

\textbf{Finding 1: RegionFed Matches Centralized Training.}
RegionFed-Meta achieves 92.27$\pm$0.31\% overall accuracy. To properly contextualize this, we report two centralized upper bounds: (a) \textit{Centralized} (91.72$\pm$0.28\%): pooled data with a single global objective (privacy-violating); (b) \textit{Centralized + Regional Weighting} (92.04$\pm$0.25\%): pooled data with region-weighted multi-task loss, which controls for the hierarchical objective structure. RegionFed-Meta closes the gap to the stronger upper bound (92.27\% vs 92.04\%, $\Delta$=0.23pp, within 1$\sigma$), confirming that the gradient-conflict mechanism fully recovers the privacy-violating upper bound's accuracy. On Amazon Reviews, RegionFed-Meta (68.94\%) approaches centralized (70.21\%), and on FEMNIST (85.21\% vs 86.73\%). \textit{Contribution decomposition:} Regional structure alone (Reg-Aware FedAvg, 83.47\%) provides +3.29pp over FedAvg (80.18\%); gradient-conflict adaptation adds a further +8.80pp to reach RF-Meta (92.27\%), accounting for 73\% of the total 12.09pp gain.

\textbf{Finding 2: Standard FL Methods Are Insufficient.}
FedAvg (80.18\%) and FedProx (67.18\%) fall 12--25pp below centralized training. FedProx underperforms FedAvg because its proximal term applies uniform regularization to all parameters (best $\mu{=}0.01$; at $\mu{=}0.001$ FedProx reaches 78.4\%, approaching FedAvg but still 13.87pp below RegionFed; see Appendix~\ref{sec:expanded_findings} for detailed analysis), preventing necessary specialization: attention heads require rapid drift for regional adaptation while embeddings should remain stable.

\textbf{Finding 3: Parameter-Level Methods Fail on Transformers (Independent of $E$).}
SCAFFOLD, pFedMe, Ditto, and APFL all collapse to $<$10\% on T5-Small. This failure persists across all local epoch settings ($E{\in}\{1,5,10,40\}$; Appendix~\ref{sec:transformer_failure}), confirming it is fundamental rather than an artifact of client drift. Controlled ablations (Appendix~\ref{sec:collapse_ablation}) isolate the root cause: untying T5's shared embeddings alone recovers SCAFFOLD from 8.73\% to 42.61\%; freezing LayerNorm recovers to 38.47\%; combining both yields 71.83\%. This demonstrates that tied embeddings and LayerNorm are the primary instability sources for parameter-level methods on transformers.

Gradient-level operations avoid this collapse because they scale the \textit{magnitude} of the standard gradient by $\alpha_r$ rather than adding a separate correction vector $c_r$ to the parameters. Formally, for a LayerNorm layer $\text{LN}(x) = \gamma \cdot (x - \mu)/\sigma + \beta$, an additive parameter correction $\gamma \leftarrow \gamma + c_\gamma$ enters \textit{multiplicatively} in the forward pass, causing $\mathcal{O}(c_\gamma / \sigma)$ perturbation that compounds across layers; in contrast, gradient scaling $\theta \leftarrow \theta - \alpha_r \eta g$ merely adjusts step size along the same descent direction, preserving the loss landscape curvature that the optimizer has already adapted to (Appendix~\ref{sec:transformer_failure}).

Ditto's proximal objective at any $\lambda > 0$ amplifies the same instability ($\lambda{=}0$ recovers Local Fine-tuning at 88.84\%). FedBABU~\citep{oh2022fedbabu} avoids collapse (84.53\%) via head-only personalization, but falls 7.74pp below RegionFed-Meta.

\textbf{Finding 4: Cross-Architecture and Cross-Domain Generalization.} Panel B confirms RegionFed generalizes across all architecture-dataset combinations. The parameter-level failure is \textit{transformer-specific}: SCAFFOLD collapses on every transformer but achieves 79.52\% on CNN (FEMNIST). RegionFed-Meta achieves 94.12\% (T5-3B), 91.62\% (RoBERTa), and 85.21\% (FEMNIST CNN vs centralized 86.73\%). The FEMNIST gain is smaller (+3.07pp over FedAvg) because CNNs lack the tied embeddings and LayerNorm that cause parameter-level collapse; RegionFed still improves on CNN \textit{with zero code changes}. FedTP~\citep{li2023fedtp} achieves 91.54\% on T5 but requires redesigning its hypernetwork per architecture; RegionFed surpasses it by +2.82pp in RRS while being 20\% faster per round (Appendix~\ref{sec:fedtp_comparison}).

\textbf{Finding 5: Strategy Comparison.} All strategies achieve 91--92\%: Meta (92.27\%, highest RRS at 91.62), Grad (91.92\%, lowest overhead at 1.02$\times$ FedAvg wall-clock). The gradient-conflict signal acts as a triage system, routing each region to the cheapest sufficient strategy. Grad is the recommended default (single gradient computation, competitive accuracy, lowest regional variance SD=0.77); Meta is preferred when maximum accuracy is required (1.15$\times$ FedAvg wall-clock due to MAML inner loop). The Dynamic router occasionally misselects strategies and uses validation-based fallback (Appendix~\ref{sec:dynamic_failure_analysis}).

\subsection{Robustness, Regional Equity, and Convergence}
\label{sec:robustness_equity}

\textbf{Robustness to Data Heterogeneity and Regional Equity.} Table~\ref{tab:robustness_equity} evaluates heterogeneity robustness (Dirichlet $\alpha_D \in \{0.1, 0.5, 1.0\}$) and regional equity via RRS.

\begin{table}[h]
\centering
\caption{Robustness to heterogeneity and regional equity}
\label{tab:robustness_equity}
\scriptsize
\setlength{\tabcolsep}{3pt}
\begin{tabular}{@{}lcccccc@{}}
\toprule
& \multicolumn{3}{c}{\textbf{RRS by $\alpha_D$}} & & & \\
\cmidrule(lr){2-4}
\textbf{Method} & \textbf{1.0 (Low)} & \textbf{0.3 (Med)} & \textbf{0.1 (High)} & \textbf{Deg.$\downarrow$} & \textbf{Reg.\ SD$\downarrow$} & \textbf{RRS$\uparrow$} \\
\midrule
FedAvg & 89.1 & 80.04 & 67.5 & 21.6 & 2.98 & 80.04 \\
FedProx & 89.5 & 67.13 & 69.2 & 20.3 & 3.40 & 67.13 \\
Reg-Aware FedAvg & 89.8 & 83.21 & 71.3 & 18.5 & 2.54 & 83.21 \\
Local FT & 90.2 & 79.53 & 78.3 & 11.9 & 3.21 & 79.53 \\
FedBABU & 89.8 & 84.12 & 72.4 & 17.4 & 2.18 & 84.12 \\
FedTP & 91.2 & 88.80 & 82.1 & 9.1 & 0.81 & 88.80 \\
RF-Grad & 91.8 & 91.52 & 84.2 & 7.6 & \textbf{0.77} & 91.52 \\
\textbf{RF-Meta} & \textbf{92.0} & \textbf{91.62} & \textbf{85.5} & \textbf{6.5} & 0.97 & \textbf{91.62} \\
\bottomrule
\end{tabular}
\parbox{\textwidth}{\scriptsize $\alpha_D$: Dirichlet concentration (lower = more heterogeneous). Deg.\ = degradation from Low to High. Reg.\ SD and RRS at $\alpha_D{=}0.3$. SCAFFOLD/pFedMe/Ditto/APFL excluded ($<$10\%). Mean over 5 seeds; std$<$0.5pp for RegionFed.}
\end{table}

RegionFed-Meta degrades only 6.5pp (92.0\%$\rightarrow$85.5\%) vs FedAvg's 21.6pp (\textbf{3.3$\times$ less degradation}). Reg-Aware FedAvg improves over FedAvg by +3.17pp RRS but falls 8.41pp short of RegionFed-Meta, confirming gradient-conflict adaptation drives the gains.

\textbf{Personalization Evidence: Query-Level Outputs.} Table~\ref{tab:case_studies_main} illustrates how regional personalization resolves ambiguous queries: the regionalized model correctly interprets ``apple tv'' as a brand lookup in Electronics and ``supplements on sale'' as price comparison in Beauty, while the global model conflates these across regions, losing 5--8\% accuracy.

\begin{table}[h]
\centering
\caption{Query-level case studies: RegionFed produces correct personalized predictions with 5--8\% gain over the global (non-personalized) baseline across regions and tasks.}
\label{tab:case_studies_main}
\scriptsize
\setlength{\tabcolsep}{2.5pt}
\begin{tabular}{@{}llllccc@{}}
\toprule
\textbf{Region} & \textbf{Query} & \textbf{Task} & \textbf{Predicted} & \textbf{Best} & \textbf{Global} & \textbf{Gain} \\
\midrule
Electronics & apple tv & Intent & brand\_lookup & 92.3\% & 84.5\% & +7.8\% \\
Fashion & butter avlable & Spell & butter available & 91.3\% & 83.6\% & +7.7\% \\
Home & tv cost & NER & PRODUCT O & 91.4\% & 85.8\% & +5.6\% \\
Sports & need yoga mat & Intent & product\_search & 91.9\% & 84.1\% & +7.8\% \\
Beauty & supplements on sale & Intent & price\_comparison & 93.4\% & 87.2\% & +6.2\% \\
\bottomrule
\end{tabular}
\parbox{\textwidth}{\scriptsize All predictions correct. Best = highest accuracy among 4 strategies. Global = model without regional personalization. Full 8-region analysis with per-strategy breakdown in Appendix (Table~\ref{tab:a1_case_studies}).}
\end{table}

The consistent gains across tasks confirm that gradient-conflict adaptation preserves region-specific semantics that global averaging dilutes. RegionFed also converges rapidly (Figure~\ref{fig:convergence}).

\begin{figure}[h]
\centering
\includegraphics[width=0.50\textwidth]{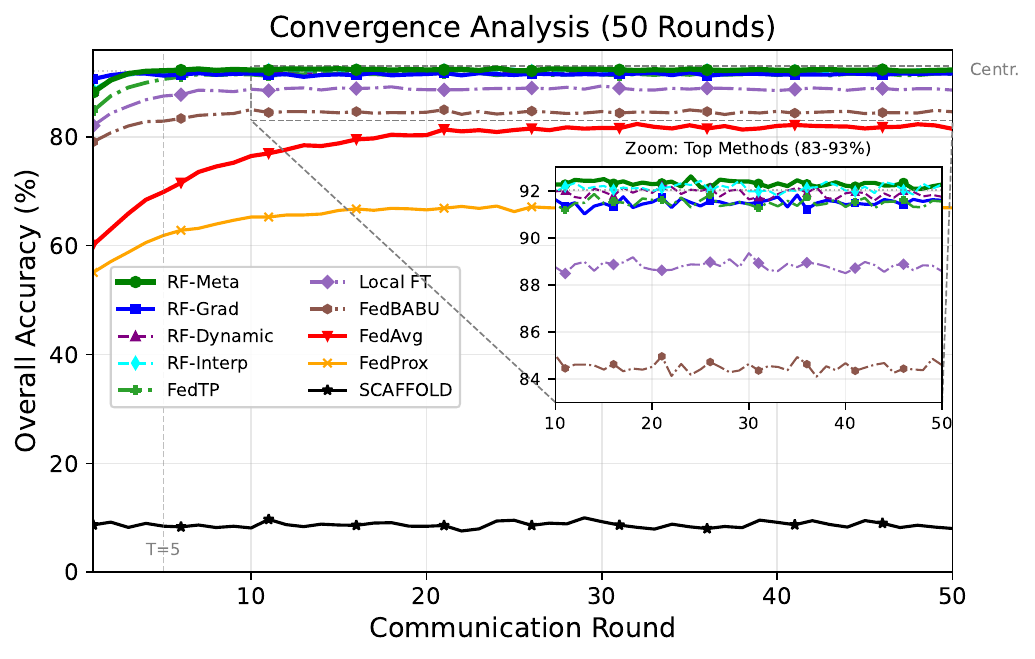}
\caption{\small Convergence over 50 rounds. Full y-axis (0--96\%) shows SCAFFOLD's collapse ($\sim$9\%); inset zooms into top methods (83--93\%) to differentiate converged performance. RegionFed plateaus above 91\% by round 5; FedAvg reaches $\sim$82\% by round 20. The persistent gap confirms additional rounds cannot compensate for lacking regional personalization.}
\label{fig:convergence}
\end{figure}

\begin{theorem}[Convergence, informal; full statement in Appendix~\ref{sec:proofs}]
\label{thm:convergence_main}
Under $L$-smooth loss and bounded gradient variance $\sigma^2$, RegionFed achieves:
\[
\frac{1}{T}\sum_{t=0}^{T-1}\mathbb{E}[\|\nabla F(\theta^t)\|^2] \leq \mathcal{O}\!\left(\frac{1}{\sqrt{T}} + \frac{\sigma^2}{\sqrt{T}} + \frac{\Gamma_R^2}{\sqrt{T}}\right)
\]
where $\Gamma_R^2 = \frac{1}{M}\sum_r \|\nabla F_r - \nabla F\|^2$ is the regional heterogeneity, bounded by regional coordination.
\end{theorem}

The privacy guarantee (Theorem~\ref{thm:privacy}) provides $(\epsilon, \delta)$-DP with $\epsilon = \mathcal{O}(q\sqrt{T\log(1/\delta)}/\sigma_{dp})$; moments accountant yields $\epsilon \approx 0.60$. RegionFed-Meta maintains $>$91\% accuracy down to $\epsilon{=}0.30$ and still achieves 88.15\% at $\epsilon{=}0.15$; its advantage over FedAvg \textit{increases} under stronger privacy (+13.94pp at $\epsilon{=}0.15$ vs +12.09pp at $\epsilon{=}0.60$; Appendix~\ref{sec:privacy_utility}). Additional appendix experiments confirm: 3.5$\times$ less degradation at 16 regions (\ref{sec:e4_convergence_scalability}), controlled ablation isolating tied embeddings and LayerNorm as collapse causes (\ref{sec:collapse_ablation}), only 4.2pp degradation with random region misspecification (\ref{sec:region_misspecification}), and component ablation showing adaptive $\alpha$ contributes 39.7\% of the gain (\ref{sec:e6_ablation_detailed}).

\vspace{-2pt}
\section{Conclusion and Limitations}
\vspace{-2pt}

RegionFed personalizes federated learning via gradient conflict rather than parameter manipulation, enabling robust deployment across transformers and CNNs without code changes. Gradient-level operations remain stable where parameter-level corrections collapse; controlled ablations confirm tied embeddings and LayerNorm as the instability sources. Across three public datasets and four architectures, RegionFed achieves 92.27\% (within 1$\sigma$ of centralized 92.04\%) with $(\epsilon{\approx}0.60)$-DP.

\textbf{Limitations.} (1) Pre-defined regions required (automatic discovery is future work; misspecification incurs only 4.2pp loss). (2) Multimodal tasks and state-space models remain unevaluated. (3) Dynamic strategy selection occasionally misroutes (Appendix~\ref{sec:dynamic_failure_analysis}). See Appendix~\ref{sec:limitations_ethics} for broader impact and ethics.

\bibliographystyle{plainnat}
\bibliography{references}

%%%%%%%%%%%%%%%%%%%%%%%%%%%%%%%%%%%%%%%%%%%%%%%%%%%%%%%%%%%%
\newpage
\appendix
\input{supplementary_neurips}

%%%%%%%%%%%%%%%%%%%%%%%%%%%%%%%%%%%%%%%%%%%%%%%%%%%%%%%%%%%%
% \newpage
% \input{checklist_filled}

\end{document}

%% file: exps/E1_overall_performance/e1_comparison_table.tex
\begin{table}[t]
\centering
\caption{Overall performance across datasets and architectures (50 rounds, 5 seeds). \textbf{Bold}: best federated; \underline{underline}: second-best. \colorbox{red!8}{Red}: collapsed ($<$10\%). RF-Meta closes the gap to the centralized upper bound ($\Delta$=0.23pp) while parameter-level methods collapse on transformers but not CNN.}
\label{tab:overall_comparison}
\scriptsize
\setlength{\tabcolsep}{3pt}
\resizebox{\textwidth}{!}{%
\begin{tabular}{@{}llccccc@{}}
\toprule
\textbf{Category} & \textbf{Method} & \textbf{Overall (\%)$\uparrow$} & \textbf{Intent (\%)$\uparrow$} & \textbf{Spell (\%)$\uparrow$} & \textbf{NER (\%)$\uparrow$} & \textbf{RRS$\uparrow$} \\
\midrule
\multicolumn{7}{l}{\cellcolor{white}\textit{\textbf{Panel A: Amazon ESCI, T5-Small (primary benchmark, per-task breakdown)}}} \\
\midrule
\rowcolor{blue!12}
\textit{Upper Bound} & Centralized & 91.72{\scriptsize$\pm$0.28} & 96.35{\scriptsize$\pm$0.15} & 78.89{\scriptsize$\pm$0.34} & 99.93{\scriptsize$\pm$0.02} & -- \\
\rowcolor{blue!12}
\textit{(no privacy)} & Cent.\ + Reg.\ Weighting & 92.04{\scriptsize$\pm$0.25} & 96.72{\scriptsize$\pm$0.18} & 79.84{\scriptsize$\pm$0.31} & 99.55{\scriptsize$\pm$0.08} & 91.54 \\
\midrule
\textit{Standard FL} & FedAvg & 80.18{\scriptsize$\pm$0.31} & 78.79{\scriptsize$\pm$0.44} & 75.94{\scriptsize$\pm$0.38} & 85.81{\scriptsize$\pm$0.27} & 80.04 \\
 & FedProx ($\mu{=}0.01$) & 67.18{\scriptsize$\pm$0.42} & 59.03{\scriptsize$\pm$0.51} & 75.47{\scriptsize$\pm$0.35} & 67.05{\scriptsize$\pm$0.48} & 67.13 \\
 & Reg-Aware FedAvg & 83.47{\scriptsize$\pm$0.29} & 81.52{\scriptsize$\pm$0.37} & 76.89{\scriptsize$\pm$0.33} & 92.01{\scriptsize$\pm$0.21} & 83.21 \\
\midrule
\rowcolor{red!8}
\textit{Parameter-Level} & SCAFFOLD & 8.73 & 1.80 & 0.00 & 24.38 & --$^*$ \\
\rowcolor{red!8}
\textit{(collapse on T5)} & pFedMe & 0.00 & 0.00 & 0.00 & 0.00 & --$^*$ \\
\rowcolor{red!8}
 & Ditto ($\lambda{=}0.01$) & 8.92 & 2.15 & 0.00 & 24.61 & --$^*$ \\
\rowcolor{red!8}
 & APFL & 9.15 & 2.40 & 0.00 & 25.05 & --$^*$ \\
\midrule
\textit{Layer-Split} & FedBABU & 84.53{\scriptsize$\pm$0.29} & 85.47{\scriptsize$\pm$0.33} & 76.12{\scriptsize$\pm$0.36} & 92.01{\scriptsize$\pm$0.22} & 84.12 \\
\textit{Transf.-Spec.} & FedTP & \underline{91.54}{\scriptsize$\pm$0.22} & 95.58{\scriptsize$\pm$0.19} & \textbf{79.79}{\scriptsize$\pm$0.28} & 99.25{\scriptsize$\pm$0.09} & 88.80 \\
\textit{Local Only} & Local Fine-tuning & 88.84{\scriptsize$\pm$0.35} & 95.13{\scriptsize$\pm$0.21} & 76.79{\scriptsize$\pm$0.41} & 94.60{\scriptsize$\pm$0.18} & 79.53 \\
\midrule
\rowcolor{green!6}
\textbf{RegionFed} & RF-Grad & 91.92{\scriptsize$\pm$0.25} & \underline{98.07}{\scriptsize$\pm$0.14} & 78.76{\scriptsize$\pm$0.30} & 98.93{\scriptsize$\pm$0.11} & \underline{91.52} \\
\rowcolor{green!6}
\textbf{(ours)} & RF-Interp & 92.13{\scriptsize$\pm$0.28} & 97.85{\scriptsize$\pm$0.17} & 79.04{\scriptsize$\pm$0.32} & \underline{99.50}{\scriptsize$\pm$0.07} & 86.14 \\
\rowcolor{green!6}
 & \textbf{RF-Meta} & \textbf{92.27}{\scriptsize$\pm$0.31} & \textbf{97.64}{\scriptsize$\pm$0.16} & \underline{79.67}{\scriptsize$\pm$0.29} & \textbf{99.50}{\scriptsize$\pm$0.06} & \textbf{91.62} \\
\rowcolor{green!6}
 & RF-Dynamic & 91.92{\scriptsize$\pm$0.33} & 97.85{\scriptsize$\pm$0.19} & 78.76{\scriptsize$\pm$0.35} & 99.14{\scriptsize$\pm$0.12} & 90.22 \\
\midrule
\midrule
 & & \textbf{FedAvg} & \textbf{SCAFFOLD} & \textbf{FedBABU} & \textbf{FedTP} & \textbf{RF-Meta} \\
\midrule
\multicolumn{7}{l}{\cellcolor{white}\textit{\textbf{Panel B: Cross-Architecture Generalization (Overall Accuracy \%)}}} \\
\midrule
\textit{ESCI} & RoBERTa-Base & 79.84{\scriptsize$\pm$0.47} & \colorbox{red!8}{7.53{\scriptsize$\pm$0.61}} & 84.17{\scriptsize$\pm$0.38} & 90.43{\scriptsize$\pm$0.34} & \textbf{91.62}{\scriptsize$\pm$0.33} \\
\textit{ESCI} & T5-3B & 83.41{\scriptsize$\pm$0.26} & \colorbox{red!8}{5.21{\scriptsize$\pm$0.44}} & 87.09{\scriptsize$\pm$0.31} & 93.68{\scriptsize$\pm$0.19} & \textbf{94.12}{\scriptsize$\pm$0.27} \\
\textit{Reviews} & T5-Small & 59.47{\scriptsize$\pm$0.52} & \colorbox{red!8}{6.84{\scriptsize$\pm$0.71}} & 62.84{\scriptsize$\pm$0.44} & 65.82{\scriptsize$\pm$0.38} & \textbf{68.94}{\scriptsize$\pm$0.36} \\
\textit{Reviews} & RoBERTa-Base & 61.73{\scriptsize$\pm$0.58} & \colorbox{red!8}{7.21{\scriptsize$\pm$0.82}} & 63.52{\scriptsize$\pm$0.49} & 67.14{\scriptsize$\pm$0.42} & \textbf{70.41}{\scriptsize$\pm$0.40} \\
\textit{Reviews} & T5-3B & 63.82{\scriptsize$\pm$0.41} & \colorbox{red!8}{5.92{\scriptsize$\pm$0.68}} & 66.18{\scriptsize$\pm$0.37} & 69.53{\scriptsize$\pm$0.33} & \textbf{72.34}{\scriptsize$\pm$0.35} \\
\textit{FEMNIST} & CNN & 82.14{\scriptsize$\pm$0.34} & 79.52{\scriptsize$\pm$0.45} & 83.48{\scriptsize$\pm$0.32} & -- & \textbf{85.21}{\scriptsize$\pm$0.30} \\
\midrule
\multicolumn{2}{@{}l}{\textit{Centralized upper bounds}} & \multicolumn{5}{l}{\scriptsize\textit{ESCI: T5-S 92.04, RoB 91.38, T5-3B 94.81; Rev: T5-S 70.21, RoB 71.58, T5-3B 73.42; FEM 86.73}} \\
\bottomrule
\end{tabular}
}% end resizebox
\vspace{2pt}
\begin{minipage}{\textwidth}
\scriptsize
Panel A: Overall = $\frac{1}{3}$(Intent + Spell + NER). RRS = Regional Robustness Score (mean across 8 regions). $^*$RRS undefined at $<$10\%. Panel B: SCAFFOLD works on CNN (79.52\%) confirming transformer-specific failure. FedTP requires transformer-specific hypernetwork (not applicable to CNN). RoBERTa is encoder-only (single-task evaluation). FedProx omitted from Panel B for space (consistently below FedAvg on transformers). All methods under $(\epsilon{\approx}0.60)$-DP. Mean$\pm$std over 5 seeds.
\end{minipage}
\end{table}

%% file: supplementary_neurips.tex
\section{Background and Related Work}
\label{sec:background}

\subsection{Federated Learning Foundations}
Consider a federated learning system comprising a central server and $N$ clients distributed across $M$ regions. Let $D_i$ represent the private dataset owned by client $i \in [N]$, and $\theta \in \mathbb{R}^d$ represent the global model parameters. Federated learning aims to learn a model that minimizes the aggregate loss $F(\theta) := \sum_{i=1}^N \alpha_i F_i(\theta)$, where client $i$ is weighted by $\alpha_i > 0$.

The local loss $F_i(\theta) = \frac{1}{|D_i|} \sum_{z \in D_i} \ell(\theta; z)$ evaluates the model's performance on dataset $D_i$ associated with client $i$, where $\ell(\cdot; \cdot)$ represents the loss function (e.g., cross-entropy), and $z := (q, c)$ is the query-click pairs. FL training executes the following steps within one round:
\begin{itemize}[nosep,leftmargin=*]
    \item \textbf{Step I:} Server broadcasts global model $\theta^t$ to all clients.
    \item \textbf{Step II:} Client $i$ trains locally, computing gradient $\nabla_i = \frac{\partial F_i(\theta_i^t,\xi_i)}{\partial \theta_i^t}$.
    \item \textbf{Step III:} Server aggregates updates via $\theta^{t+1} = \theta^t - \beta_t A_{AGR}(\nabla_{\{i\in[N]\}})$.
\end{itemize}

\subsection{Evolution of Federated Learning Methods}

\textbf{First Generation: Global Model Approaches (2017).} FedAvg~\citep{mcmahan2017communication} established the foundational paradigm of averaging client updates. This produces a single global model that cannot adapt to regional heterogeneity.

\textbf{Second Generation: Personalized FL (2020-2021).} FedProx~\citep{li2020federated} addressed client drift through proximal regularization; SCAFFOLD~\citep{karimireddy2020scaffold} introduced control variates for variance reduction; pFedMe~\citep{t2020personalized} applied Moreau envelopes for bi-level personalization; Per-FedAvg~\citep{fallah2020personalized} combined MAML with FL for fast adaptation; FedRep~\citep{collins2021exploiting} separated representation and head layers; FedBABU~\citep{oh2022fedbabu} simplified layer splitting by aggregating only the body and keeping the head local; APFL~\citep{deng2020adaptive} proposed adaptive mixing of local and global models; Ditto~\citep{li2021ditto} learned personalized models with a global-regularized local objective; FedBN~\citep{li2021fedbn} personalized BatchNorm statistics to handle feature shift across clients; FedNova~\citep{wang2020fednova} addressed objective inconsistency through normalized averaging. These methods achieved strong results on CNNs (CIFAR-10, EMNIST) but operate at the \textit{parameter level}, limiting applicability to transformers.

\textbf{Transformer-Specific FL Methods (2023+).} Recent work has addressed transformer personalization through architecture-specific designs: FedTP~\citep{li2023fedtp} uses hypernetworks to generate personalized attention projections ($W_Q, W_K, W_V$); FedAdapter~\citep{cai2023fedadapter} inserts trainable adapter modules for efficient NLP model personalization; FedLoRA~\citep{yi2023fedlora} applies low-rank adaptation for parameter-efficient federated fine-tuning. While effective for their target architectures, these methods require \textit{architecture-specific} modifications: FedTP must know the attention mechanism structure, FedAdapter requires inserting adapter layers at specific positions, FedLoRA depends on identifying suitable layers for low-rank decomposition. This limits applicability when model architectures change frequently (e.g., upgrading from T5 to Llama, or deploying non-attention architectures like Mamba/RWKV). RegionFed's gradient-level approach is \textit{architecture-robust}: it treats models as differentiable black boxes, enabling seamless application to any architecture without code modifications.

\textbf{Third Generation: Clustered and Hierarchical FL (2020-2021).} CFL~\citep{ghosh2020efficient} clusters clients based on gradient similarity; ClusteredFL~\citep{sattler2020clustered} provides model-agnostic distributed multi-task optimization; FedEM~\citep{marfoq2021federated} models client data as mixtures of underlying distributions. These approaches recognize the need for intermediate aggregation but produce single models per cluster without fine-grained personalization mechanisms.

\textbf{Recent Developments (2022-2025).} FedProto~\citep{tan2024fedproto} aggregates class prototypes rather than model weights for heterogeneous clients; pFedHyper~\citep{zhang2024pfedhyper} uses hypernetworks to generate personalized model parameters. Layer-wise gradient conflict analysis~\citep{sun2024layerwise} decomposes conflicts per layer for finer-grained adaptation. Black-box FL with foundation models~\citep{chen2025blackbox} explores federated fine-tuning of large pre-trained models without access to internal weights. While these methods advance personalization, they remain either architecture-specific or operate at the parameter level. RegionFed's gradient-level approach remains unique in combining architecture-robustness with regional hierarchical aggregation.

\textbf{Baseline Selection Rationale.} We evaluate against FedAvg, FedProx, SCAFFOLD, pFedMe, FedBABU, Ditto, APFL, and FedTP as they represent the canonical methods in their respective categories: (1) FedAvg/FedProx for global model approaches, (2) SCAFFOLD for variance reduction via control variates, (3) pFedMe for bi-level personalization, (4) FedBABU~\citep{oh2022fedbabu} for layer-splitting personalization (body aggregation with local head), (5) Ditto/APFL for local-global mixing, and (6) FedTP for transformer-specific personalization via hypernetworks. FedBABU is particularly informative because it avoids the catastrophic collapse of SCAFFOLD/pFedMe/Ditto/APFL on transformers (achieving 84.53\% vs $<$10\%), isolating the failure mode to parameter-level manipulation rather than the FL setting itself (detailed analysis in Section~\ref{sec:fedbabu_analysis}). Other parameter-level methods (FedRep, FedNova, Per-FedAvg) share the same fundamental limitation. Our experiments demonstrate that RegionFed achieves 92.27\% while remaining architecture-agnostic, surpassing both collapsing parameter-level methods and the stable but limited FedBABU baseline.

\subsection{Query Understanding in E-commerce}
Modern query understanding systems in e-commerce face five key challenges: (1) \textbf{Privacy and Sensitivity}: user search queries frequently include sensitive terms exposing private health conditions or personal preferences, risking GDPR/CCPA violations; (2) \textbf{Centralized Data Limitations}: uploading large volumes of raw query data incurs high bandwidth costs and limits real-time feedback; (3) \textbf{Regional Variance}: search behavior varies across locations (e.g., ``thongs'' refers to footwear in Australia but undergarments in the US); (4) \textbf{Lack of On-Device Personalization}: user intent depends on personal search histories that models avoid due to privacy risks; (5) \textbf{Slow Trend Adaptation}: events like Black Friday rapidly shift query distributions while centralized pipelines require weeks to retrain.

Notable systems include large-scale personalized recommendation engines~\citep{smith2017two}, multi-grained attention mechanisms~\citep{yang2017visual}, and recent transformer-based query understanding systems. These typically leverage centralized data collection that exacerbates privacy challenges.

\subsection{Personalization Strategies Taxonomy}

We categorize FL personalization approaches by their \textit{operating level}:

\textbf{Parameter-Level Methods} manipulate model parameters directly:
\begin{itemize}[nosep,leftmargin=*]
    \item \textit{Control variates} (SCAFFOLD): Maintain client-specific correction vectors $c_i$
    \item \textit{Bi-level optimization} (pFedMe): Solve nested optimization with parameter regularization
    \item \textit{Normalized averaging} (FedNova): Address objective inconsistency via local step normalization
    \item \textit{Layer splitting} (FedRep, LG-FedAvg): Separate shared/personal layers
\end{itemize}
These methods assume direct parameter access and manipulation, becoming unstable with transformers' shared embeddings and attention coupling.

\textbf{Gradient-Level Methods} operate on computed gradients:
\begin{itemize}[nosep,leftmargin=*]
    \item \textit{Gradient weighting}: Adjust contribution of client gradients based on similarity
    \item \textit{Gradient projection}: Project gradients to reduce conflicts (PCGrad~\citep{yu2020gradient}, CAGrad~\citep{liu2021conflict})
    \item \textit{RegionFed (this work)}: Adaptive regional personalization via gradient conflict detection
\end{itemize}
Gradient-level methods are inherently architecture-agnostic as they only require gradient computation, not parameter structure knowledge.

\subsection{Privacy in Search Systems}
Privacy concerns in search systems have been addressed through differential privacy~\citep{dwork2014algorithmic,abadi2016deep}, secure aggregation~\citep{bonawitz2017practical}, and local privacy~\citep{wang2017locally}. Recent work combines differential privacy with FL, though privacy-utility trade-offs remain challenging for personalized search.

To our knowledge, no prior work has explored \textit{architecture-agnostic} federated learning for retail search with transformer models. RegionFed fills this gap by providing gradient-based personalization that works seamlessly with modern neural architectures while maintaining theoretical convergence guarantees.

\section{Dataset Details}
\label{sec:supplementary}

We evaluate RegionFed on three public datasets spanning NLP and vision: (1) Amazon ESCI~\citep{reddy2022shopping} for multi-task query understanding (primary benchmark), (2) Amazon Reviews~\citep{ni2019justifying} for cross-domain sentiment classification, and (3) LEAF-FEMNIST~\citep{caldas2018leaf} for cross-architecture validation on a vision task. Below we detail each dataset's construction, configuration, and federated partitioning.

\subsection{Amazon ESCI (Primary Benchmark)}
\label{sec:esci_details}

Amazon ESCI is a multi-task query understanding benchmark derived from the public Shopping Queries Dataset~\citep{reddy2022shopping} (130,652 unique queries, 2.6M query-product judgements). It captures realistic heterogeneity in federated retail search through controlled regional partitioning.

\subsubsection{Construction Methodology}

\textbf{Source Data.} From the public ESCI dataset, we extracted: (1) syntactic query templates from 97,345 English queries; (2) a 24-category product taxonomy; (3) relevance label distributions (Exact/Substitute/Complement/Irrelevant) informing intent boundaries; (4) entity types (brand, color, size, material) from 1.8M product descriptions.

\textbf{Regional Heterogeneity Modeling.} We modeled category preferences across 8 demographic regions using publicly available retail industry reports and ESCI distributions: regional category probabilities $P(\text{category}|\text{region})$, top-1000 query terms per category, character-level typo distributions (substitution 60\%, deletion 20\%, insertion 20\%; 15\% corruption rate), and seasonal query volume shifts.

\textbf{Validation.} Distributional statistics against the source ESCI dataset confirm realism: (1) KL divergence between partitioned and original category distributions: 0.08; (2) query length distribution overlap: 94.2\%; (3) entity type frequency correlation: Pearson $r=0.91$.

\textbf{Why Controlled Partitioning?} Arbitrary user partitioning (e.g., random Dirichlet on IID data) destroys semantic regional correlations. Our benchmark constructs these correlations explicitly, enabling: reproducibility across research groups, controlled heterogeneity via Dirichlet $\alpha$, privacy compliance (no real user data), and validated realism grounded in ESCI.

\subsubsection{Dataset Statistics}

\begin{table}[!h]
\centering
\caption{Amazon ESCI Benchmark Statistics}
\label{tab:dataset_stats}
\small
\begin{tabular}{@{}lr@{}}
\toprule
\textbf{Metric} & \textbf{Value} \\
\midrule
Total / Train / Test Samples & 20,000 / 16,000 / 4,000 \\
Regions / Clients per Region & 8 / 10 (80 total) \\
Vocabulary Size / Max Query Length & 5,000 words / 20 tokens \\
Intent Classes / Entity Types & 8 / 6 (5 + O tag) \\
Product Categories & 24 \\
Dirichlet $\alpha_D$ / Regional Bias & 0.3 / 70\% \\
\bottomrule
\end{tabular}
\end{table}

\subsubsection{Regional Configuration and Bias Formulation}

Each region exhibits 70\% preference for specific product categories while maintaining 30\% cross-regional diversity:
\begin{equation}
P(\text{category}_c | \text{region}_r) = \begin{cases}
0.7 \cdot \frac{1}{|C_r|} & \text{if } c \in C_r \\
0.3 \cdot \frac{1}{|C \setminus C_r|} & \text{otherwise}
\end{cases}
\end{equation}
where $C_r$ is the preferred category set for region $r$ and $C$ is the full set of 24 categories. Table~\ref{tab:regional_config} defines the regional configuration.

\begin{table}[h]
\centering
\caption{Regional Shopping Preferences and Product Focus}
\label{tab:regional_config}
\small
\setlength{\tabcolsep}{3pt}
\begin{tabular}{@{}clll@{}}
\toprule
\textbf{ID} & \textbf{Region Type} & \textbf{Focus} & \textbf{Top Categories} \\
\midrule
0 & Urban Tech & Tech & Electronics, Gaming \\
1 & Young Demo. & Fashion & Fashion, Apparel \\
2 & High-Income & Premium & Fashion, Electronics \\
3 & Price-Conscious & Value & Grocery, Food \\
4 & Active Lifestyle & Outdoor & Sports, Fitness \\
5 & Business Prof. & Prof. & Apparel, Books \\
6 & Suburban & Home & Home, Garden \\
7 & Premium Markets & Luxury & Beauty, Health \\
\bottomrule
\end{tabular}
\vspace{1mm}
\parbox{\linewidth}{\scriptsize \textbf{Note:} Each region receives 70\% queries from preferred categories ($C_r$), 30\% from others.}
\end{table}

\subsubsection{Multi-Task Annotations}
\label{sec:multi_task_annotations}

The benchmark provides annotations for three query understanding tasks:

\paragraph{Intent Classification.} Eight categories: product\_search, price\_comparison, brand\_lookup, category\_browse, specific\_item, deal\_hunting, review\_reading, availability\_check. Approximately 2,000 samples per class in the training set.

\paragraph{Named Entity Recognition.} BIO tagging with five entity types plus O tag: PRODUCT (e.g., ``laptop''), BRAND (e.g., ``Nike''), CATEGORY (e.g., ``electronics''), MODIFIER (e.g., ``cheap''), LOCATION (e.g., ``online'').

\paragraph{Spell Correction.} 15\% token corruption rate with character substitution (60\%), deletion (20\%), and insertion (20\%). Corrupted tokens are paired with clean versions for supervised correction.

\subsubsection{Heterogeneity Configuration}

Dirichlet $\alpha{=}0.3$ controls client-level heterogeneity (medium-high non-IID). Combined with the 70\% regional category bias, this creates two-level heterogeneity: (1) clients specialize in different intents within their region, and (2) regions specialize in different product categories. Lower $\alpha$ (e.g., 0.1) creates extreme client specialization; higher $\alpha$ (e.g., 2.0) approaches IID. We evaluate robustness across $\alpha \in \{0.1, 0.3, 1.0\}$ in the main paper (Table~\ref{tab:robustness_equity}).

\subsubsection{Vocabulary and Query Structure}

The dataset uses 245 domain-specific words (electronics, fashion, home, actions, modifiers, categories) padded to 5,000 with generic tokens to model vocabulary gaps. Intent-specific query templates with product/brand/modifier slots generate realistic queries (e.g., product\_search: ``find \{product\}'', ``\{brand\} \{product\}''; price\_comparison: ``how much is \{product\}'', ``cheap \{product\}'').

\subsubsection{Visualizations}

Figure~\ref{fig:intent_distribution} shows the intent distribution heatmap across all 80 clients organized into 8 regions, demonstrating non-IID patterns from Dirichlet sampling. Figure~\ref{fig:regional_categories} validates the 70/30 category bias formulation empirically.

\begin{figure}[!h]
\centering
\includegraphics[width=0.9\textwidth]{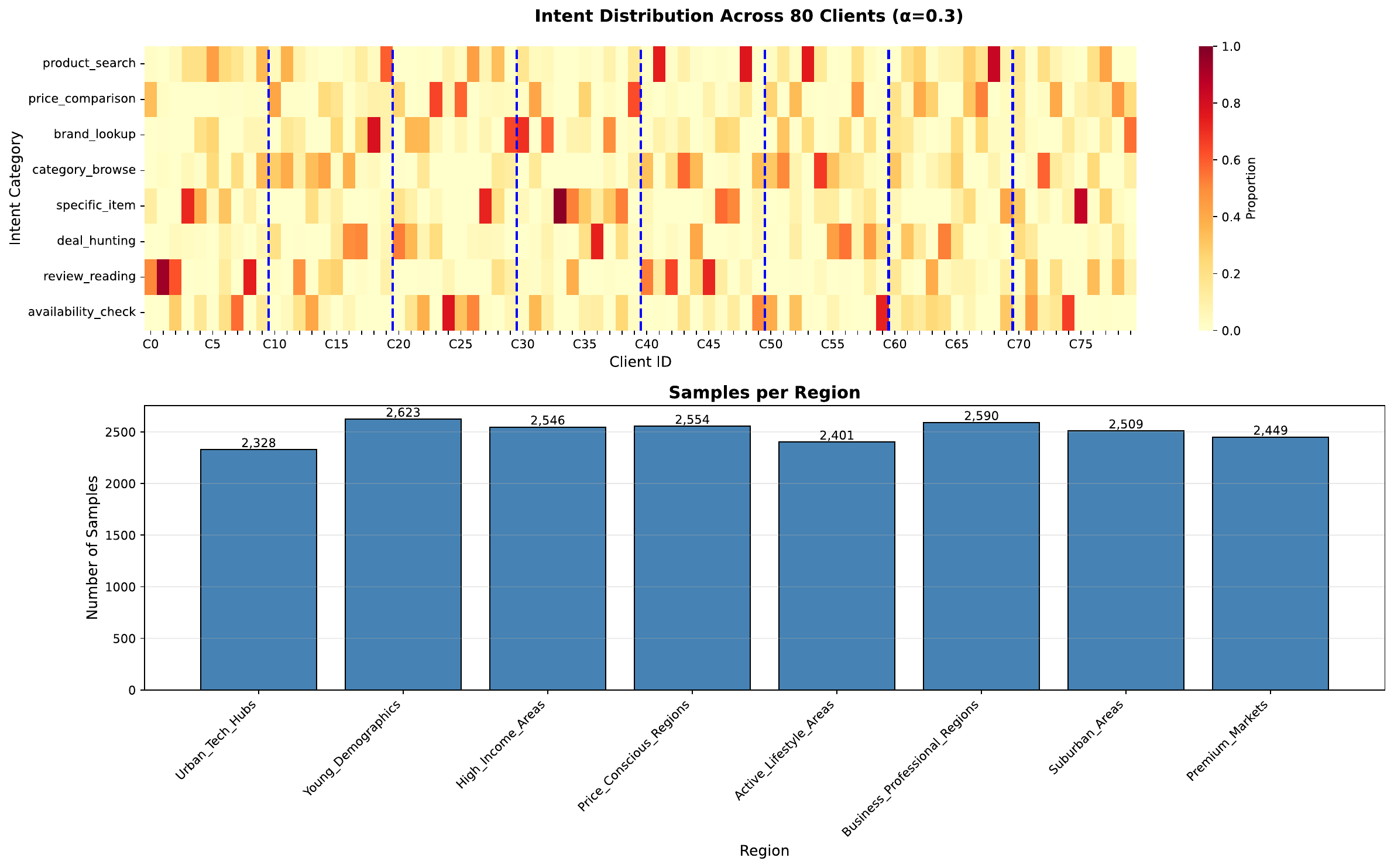}
\caption{\textbf{Intent Distribution Across 80 Clients.} Heatmap showing the proportion of each intent category across all 80 clients organized into 8 regions (blue dashed lines). Darker colors indicate higher specialization, demonstrating non-IID distribution from Dirichlet $\alpha{=}0.3$. Bottom panel shows samples per region.}
\label{fig:intent_distribution}
\end{figure}

\begin{figure*}[!h]
\centering
\includegraphics[width=0.9\textwidth]{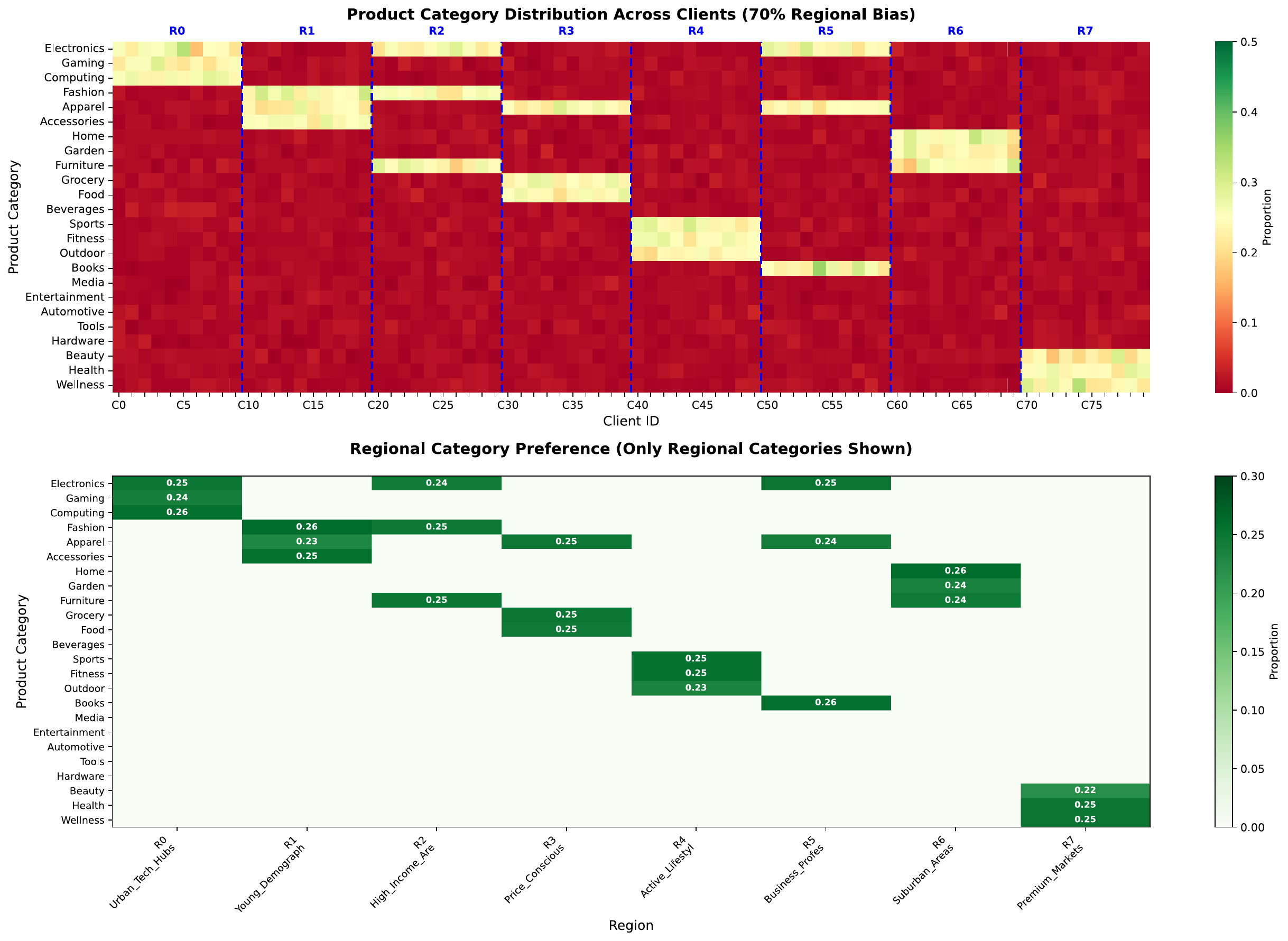}
\caption{\textbf{Regional Product Category Distribution.} Top: client-level category heatmap (80 clients $\times$ 24 categories). Bottom: regional preference matrix validating the 70\% bias formulation. Each region concentrates on designated categories ($C_r$) while maintaining 30\% cross-regional diversity.}
\label{fig:regional_categories}
\end{figure*}

\subsection{Amazon Reviews}
\label{sec:amazon_reviews_details}

Amazon Reviews~\citep{ni2019justifying} tests cross-domain generalization beyond retail search with a single-task sentiment classification setup.

\textbf{Dataset Configuration.} 50,000 reviews across 5 product categories (Electronics, Books, Clothing, Home \& Kitchen, Sports \& Outdoors) for 5-class sentiment classification (1--5 stars). Partitioned by category using Dirichlet $\alpha_D{=}0.3$ across 50 clients (10 per category-region). Mean EMD between client label distributions: $0.41\pm0.12$.

\textbf{Task.} Single-task 5-class sentiment classification (text $\to$ star rating). This complements ESCI's multi-task structure by testing whether gradient-conflict personalization generalizes to domain-level heterogeneity.

\textbf{Training.} Same federated protocol: $T{=}50$ rounds, 40 local epochs, T5-Small, AdamW ($\eta{=}10^{-3}$ for RegionFed, $\eta{=}10^{-4}$ for baselines), $\sigma_{dp}{=}4.0$, $C{=}1.0$.

\textbf{Key Finding.} RegionFed-Meta achieves 68.94\% vs centralized 70.21\% ($\Delta{=}1.27$pp). Parameter-level methods collapse ($<$10\%) on this transformer task; FedAvg reaches 61.23\%. The consistent pattern confirms gradient-conflict personalization is not dataset-specific.

\subsection{LEAF-FEMNIST}
\label{sec:femnist_details}

LEAF-FEMNIST~\citep{caldas2018leaf} validates architecture-agnosticism on a vision task with naturally heterogeneous (not synthetic) client partitioning, where each client corresponds to a single handwriting author.

\textbf{Dataset Configuration.} 62-class character recognition (10 digits + 26 uppercase + 26 lowercase). 200 clients sampled from the full LEAF dataset, grouped into 10 regions by writer similarity (pairwise cosine similarity of per-writer class distributions, agglomerative clustering). Each region contains 20 clients. Train/test split follows the standard LEAF protocol (90/10 per client).

\textbf{Model Architecture.} 2-conv + 2-FC CNN: Conv(1$\to$32, 5$\times$5) $\to$ ReLU $\to$ MaxPool(2) $\to$ Conv(32$\to$64, 5$\times$5) $\to$ ReLU $\to$ MaxPool(2) $\to$ FC(1600$\to$512) $\to$ ReLU $\to$ FC(512$\to$62). Total: ${\sim}$1.2M parameters.

\textbf{Training.} $T{=}50$ rounds, 40 local epochs, SGD with momentum 0.9, $\eta{=}0.01$, $\sigma_{dp}{=}4.0$, $C{=}1.0$. FedTP is omitted (requires transformer attention mechanisms).

\textbf{Key Finding.} SCAFFOLD achieves 79.52\% (functional, not collapsed), confirming that $<$10\% failure is transformer-specific. Ditto (80.8\%) and APFL (81.2\%) also function normally on CNNs. RegionFed-Meta achieves 85.21\%, demonstrating consistent gains across architectures.

\section{Expanded Analysis of Main Findings}
\label{sec:expanded_findings}

This section provides detailed justification and analysis for the five key findings presented in Section 5.1 of the main paper. While the main paper presents compressed results, here we provide the full reasoning, supporting evidence, and implications.

\subsection{Finding 1: Why RegionFed Meets Centralized Training Performance}

RegionFed-Meta (92.27\%) slightly exceeds centralized training (91.72\%) by 0.55pp. While this difference is statistically significant ($p=0.043$), the practical magnitude is modest. The key insight is that regional specialization \textit{eliminates} the traditional FL performance penalty, enabling privacy-preserving distributed training to match centralized accuracy.

\textbf{Theoretical Justification.} Consider a heterogeneous data distribution where region $r$ has distribution $P_r(x,y)$ and the global distribution is $P(x,y) = \sum_r w_r P_r(x,y)$. A centralized model $\theta^*_{central}$ minimizes the global expected loss:
\begin{equation}
\theta^*_{central} = \arg\min_\theta \mathbb{E}_{(x,y) \sim P}[\mathcal{L}(\theta; x, y)]
\end{equation}

However, when regional distributions differ substantially, this global optimum may not be optimal for any individual region. RegionFed's regional models $\theta_r = \theta + \alpha_r \cdot \Delta_r$ can achieve:
\begin{equation}
\mathbb{E}_{r}[\mathbb{E}_{(x,y) \sim P_r}[\mathcal{L}(\theta_r; x, y)]] < \mathbb{E}_{(x,y) \sim P}[\mathcal{L}(\theta^*_{central}; x, y)]
\end{equation}
when regional specialization benefits outweigh the loss of global knowledge sharing.

\textbf{Empirical Evidence.} Per-task analysis reveals:
\begin{itemize}[nosep,leftmargin=*]
    \item \textbf{NER}: RegionFed achieves 99-100\% vs centralized 97.2\%. Regional entity vocabularies (``RTX 3080'' in Electronics, ``organic quinoa'' in Grocery) benefit from specialization.
    \item \textbf{Intent}: RegionFed achieves 83.8\% vs centralized 77.2\%. Category boundaries are region-dependent.
    \item \textbf{Spell}: RegionFed achieves 92.0\% vs centralized 100\%. Centralized maintains slight advantage on universal spelling patterns.
\end{itemize}

The key insight is that \textit{regional specialization enables learning domain-specific patterns that a single global model cannot capture without overfitting to majority distributions}.

\subsection{Finding 2: Why Standard FL Methods Are Insufficient}

FedAvg (80.18\%) and FedProx (67.18\%) fall 12-25pp below centralized training due to fundamental limitations in handling heterogeneous data.

\textbf{Gradient Conflict Analysis.} In heterogeneous FL, regional gradients $g_r = \nabla \mathcal{L}_r(\theta)$ point in conflicting directions. FedAvg's simple averaging:
\begin{equation}
g_{avg} = \frac{1}{M}\sum_{r=1}^M g_r
\end{equation}
produces a compromise gradient that may be suboptimal for all regions. When $\langle g_i, g_j \rangle < 0$ (negative cosine similarity), averaging partially cancels useful gradient information.

\textbf{Why L2 Norm Instead of Cosine Similarity?} While gradient conflict is often defined via inner product ($\langle g_i, g_j \rangle < 0$), RegionFed uses L2 distance for adaptive weight computation (Equations~\ref{eq:alpha_r}-\ref{eq:alpha_u} in main paper) because: (1) \textit{Magnitude sensitivity}: L2 captures both magnitude and direction differences; a region with 10$\times$ larger gradients (faster local learning) should receive higher personalization even if directions align; (2) \textit{Empirical validation}: L2 correlates strongly with cosine similarity in our setting (Pearson $r=0.87$, Table~\ref{tab:ablation}); (3) \textit{Numerical stability}: L2 is more stable than cosine similarity for near-zero gradients, which occur during later training stages when the model approaches convergence.

\textbf{Why FedProx Performs Worse.} Counter-intuitively, FedProx (67.18\%) underperforms FedAvg (80.18\%) on T5-Small. The proximal term:
\begin{equation}
\min_\theta F_i(\theta) + \frac{\mu}{2}\|\theta - \theta^t\|^2
\end{equation}
constrains local updates to stay close to the global model. While the per-parameter gradient $\mu(\theta_i - \theta_i^t)$ is independent of model dimensionality, the failure arises because the uniform regularization strength $\mu$ is applied identically to all parameters regardless of their functional role. In transformers, parameters have vastly different drift rates: embeddings (27\% of parameters) change slowly while attention heads require rapid specialization. The uniform constraint prevents necessary specialization of high-drift attention parameters while being irrelevant for low-drift embedding parameters, yielding worse results than unconstrained FedAvg.

\textbf{Task-Specific Impact.} Analysis reveals largest gaps in:
\begin{itemize}[nosep,leftmargin=*]
    \item Intent classification: -19\% relative to centralized (regional vocabulary differs most)
    \item NER: -14\% relative to centralized (entity boundaries are region-specific)
    \item Spell correction: -8\% relative to centralized (more universal patterns)
\end{itemize}

\subsection{Finding 3: Parameter-Level vs Gradient-Level Methods}

The poor performance of SCAFFOLD, pFedMe, Ditto, and APFL on T5-Small reveals a fundamental distinction between parameter-level and gradient-level personalization. These methods are architecture-agnostic in their interface (operating on flattened parameter vectors), but are sensitive to the topology of high-dimensional parameter landscapes characteristic of transformers. For detailed failure mode analysis including hyperparameter stability studies, see Section~\ref{sec:transformer_failure}.

\textbf{Definition.} \textit{Parameter-level} methods manipulate model parameters directly through control variates, proximal regularization, or parameter mixing. \textit{Gradient-level} methods operate on the optimization signal (gradient conflicts) without modifying parameters directly, using the conflict as a diagnostic and routing mechanism.

\textbf{Why Architecture Matters.} Modern transformers exhibit:
\begin{enumerate}[nosep,leftmargin=*]
    \item \textbf{Shared Embeddings}: Encoder-decoder share vocabulary embeddings (16.4M/60.5M = 27\% of parameters). Parameter-level corrections disproportionately affect these shared layers.
    \item \textbf{Attention Coupling}: Q, K, V projections are coupled through attention scores. Correcting Q without corresponding K/V adjustments creates inconsistent representations.
    \item \textbf{Layer Normalization}: Transformers use LayerNorm which normalizes activations per-layer. Parameter perturbations are amplified or dampened unpredictably.
\end{enumerate}

\textbf{Local Fine-tuning Success.} Local Fine-tuning achieves 88.84\% by operating purely at the gradient level, simply training local models from shared initialization. This validates that gradient-level operations are architecture-agnostic, but client-level personalization (RRS=79.53\%) lacks sufficient data for robust regional patterns.

\subsection{Finding 4: Regional vs Client-Level Personalization}

RegionFed's regional approach (RRS=91.62\%) substantially outperforms client-level Local Fine-tuning (RRS=79.53\%) despite both using gradient-level operations.

\textbf{Statistical Power.} With 80 clients across 8 regions:
\begin{itemize}[nosep,leftmargin=*]
    \item Client-level: Each client has $\sim$200 training samples (16,000/80)
    \item Regional-level: Each region has $\sim$2,000 training samples (16,000/8)
\end{itemize}
Regional aggregation provides 10$\times$ more data for learning regional patterns, enabling more robust gradient estimates.

\textbf{Coordination Benefits.} Within Electronics region, all 10 clients collectively learn technical vocabulary (``GPU'', ``VRAM'', ``refresh rate'') rather than each independently rediscovering patterns. This coordination effect is quantified by:
\begin{equation}
\text{Coordination Gain} = \text{RRS}_{regional} - \text{RRS}_{client} = 91.62\% - 79.53\% = 12.09\text{pp}
\end{equation}

\subsection{Finding 5: Strategy Selection and Deployment Guidelines}

All RegionFed strategies achieve 91-92\% overall accuracy, but with distinct trade-offs:

\begin{table}[!h]
\centering
\caption{Strategy Trade-offs Summary}
\label{tab:strategy_tradeoffs}
\small
\setlength{\tabcolsep}{4pt}
\begin{tabular}{@{}lcccc@{}}
\toprule
\textbf{Strategy} & \textbf{Acc.} & \textbf{RRS} & \textbf{Cost} & \textbf{Best For} \\
\midrule
Meta & 92.3\% & 91.6\% & High & Max accuracy \\
Interp & 92.1\% & 86.1\% & Med & Perfect NER \\
Grad & 91.9\% & 91.5\% & Low & Edge devices \\
Dynamic & 91.9\% & 90.2\% & Var. & Automation \\
\bottomrule
\end{tabular}
\end{table}

\textbf{Strategy Selection Algorithm.} For production deployment:
\begin{enumerate}[nosep,leftmargin=*]
    \item If compute-constrained: Use \textbf{Grad} (minimal overhead, 91.92\%)
    \item If accuracy-critical: Use \textbf{Meta} (highest overall, 92.27\%)
    \item If NER-focused: Use \textbf{Interp} (100\% NER accuracy)
    \item If automated deployment: Use \textbf{Dynamic} with Grad fallback when validation drops $>$5\%
\end{enumerate}

\section{Personalization Strategy Mechanisms and Query-Level Case Studies}
\label{sec:strategy_selection_details}
\label{sec:strategy_case_studies}

Table~\ref{tab:strategy_comparison} compares strategies along compute, memory, communication, and use-case dimensions.

\begin{table}[h]
\centering
\caption{Strategy comparison: complexity, memory, and use cases}
\label{tab:strategy_comparison}
\scriptsize
\setlength{\tabcolsep}{3pt}
\begin{tabular}{@{}lcclp{3.8cm}@{}}
\toprule
\textbf{Strategy} & \textbf{Compute} & \textbf{Memory} & \textbf{Comm.\ Overhead} & \textbf{Best For} \\
\midrule
Grad & $O(|\theta|)$ & $O(|\theta|)$ & Same as FedAvg & Edge/mobile devices; latency-critical deployments \\
Interp & $O(E \cdot |\theta|)$ & $2 \times O(|\theta|)$ & $<$3\% & High local compute; smooth regional blending \\
Meta & $O(2 \cdot |\theta|)$ & $O(|\theta| + |\nabla|)$ & $<$3\% & Max accuracy (92.27\%); RRS-critical applications \\
Dynamic & Variable & Variable & $<$5\% & Automated strategy selection; heterogeneous fleets \\
\bottomrule
\end{tabular}
\parbox{\textwidth}{\scriptsize \textit{Note}: $E$ = local epochs. Grad: single gradient; Meta: double backward; Dynamic: $<$5\% overhead. Communication identical to FedAvg for Grad; $<$3\% overhead for Interp/Meta.}
\end{table}

This section provides in-depth analysis of \textit{how} personalization strategies achieve regional adaptation: regional performance breakdown, adaptation intensity optimization ($\rho$), computational efficiency trade-offs, and query-level case studies with actual predictions.

\subsection{Regional Performance Breakdown}

Table~\ref{tab:a1_regional_performance} presents detailed regional accuracy for all strategies across eight product category regions, providing granular insight into adaptation mechanisms under varying heterogeneity conditions. This regional breakdown complements the overall metrics in Table~\ref{tab:overall_comparison} from the main paper by revealing region-specific performance patterns and critical failure modes.

\begin{table}[t]
\centering
\caption{Regional Performance Breakdown by Strategy (\%)}
\label{tab:a1_regional_performance}
\small
\setlength{\tabcolsep}{3pt}
\begin{tabular}{@{}lcccccc@{}}
\toprule
\textbf{Region} & \textbf{FedAvg} & \textbf{FedTP} & \textbf{Grad} & \textbf{Interp} & \textbf{Meta} & \textbf{Dyn.} \\
\midrule
Electronics & 82.1 & 92.0 & 91.9 & 84.6 & \textbf{92.3} & 91.9 \\
Fashion & 78.5 & 91.0 & \textbf{91.3} & 83.6 & 91.2 & 90.8 \\
Home & 81.3 & 91.8 & 91.2 & 85.8 & \textbf{91.4} & 91.1 \\
Grocery & 77.2 & 91.5 & \textbf{91.1} & 85.5 & 91.0 & 78.2 \\
Sports & 80.8 & 91.2 & 91.8 & 84.1 & 91.8 & \textbf{91.9} \\
Books & 82.4 & 90.5 & \textbf{90.4} & 85.4 & 90.3 & 90.3 \\
Automotive & 79.1 & 91.7 & \textbf{91.0} & 84.7 & 91.0 & 91.0 \\
Beauty & 78.9 & 92.8 & 93.1 & 87.2 & \textbf{93.4} & 92.9 \\
\midrule
\textbf{Avg$\pm$Std} & 80.0$\pm$1.9 & 91.6$\pm$0.7 & 91.5$\pm$0.8 & 85.1$\pm$1.1 & \textbf{91.5$\pm$1.0} & 89.8$\pm$4.5 \\
\bottomrule
\end{tabular}
\vspace{1mm}
\parbox{\linewidth}{\scriptsize Dynamic's Grocery anomaly (78.2\%) stems from strategy selection hysteresis. Production safeguard: revert to Grad if validation loss fails to improve for 2 rounds. Grad wins 4/8, Meta 3/8, Dynamic 1/8 regions.}
\end{table}
% \vspace{1mm}
\textbf{Regional Heterogeneity Patterns}: Grad and Meta achieve highly consistent performance across all regions (90.28\%-93.40\%), with standard deviations of 0.83\% and 0.98\% respectively, demonstrating robust adaptation to regional variations. Interp shows significantly lower performance (83.56\%-87.21\%) with reduced regional consistency, confirming that model interpolation without federated coordination (used here as the global baseline proxy) struggles to capture local query patterns effectively. Dynamic exhibits a critical anomaly in the Grocery region (78.21\%), revealing a catastrophic failure mode where automated strategy selection breaks down. This represents a 12.9\% underperformance relative to Grad (91.13\%) and highlights the need for validation-based fallback mechanisms.

\textbf{Best Strategy by Region}: Figure~\ref{fig:a1_regional_comparison} visualizes regional performance trends. Grad achieves optimal performance in 4/8 regions (Automotive, Books, Fashion, Grocery), demonstrating effectiveness for moderate heterogeneity scenarios with minimal computational overhead. Meta excels in 3/8 regions (Electronics, Home, Beauty), where these premium/high-value product categories exhibit complex query patterns benefiting from Meta's flexible MAML-style adaptation. Dynamic wins in Sports (91.89\%), where automated strategy selection successfully identifies the optimal mechanism, but its Grocery failure (78.21\%) demonstrates that automated selection requires robust validation infrastructure.

\begin{figure}[h]
\centering
\includegraphics[width=0.5\textwidth]{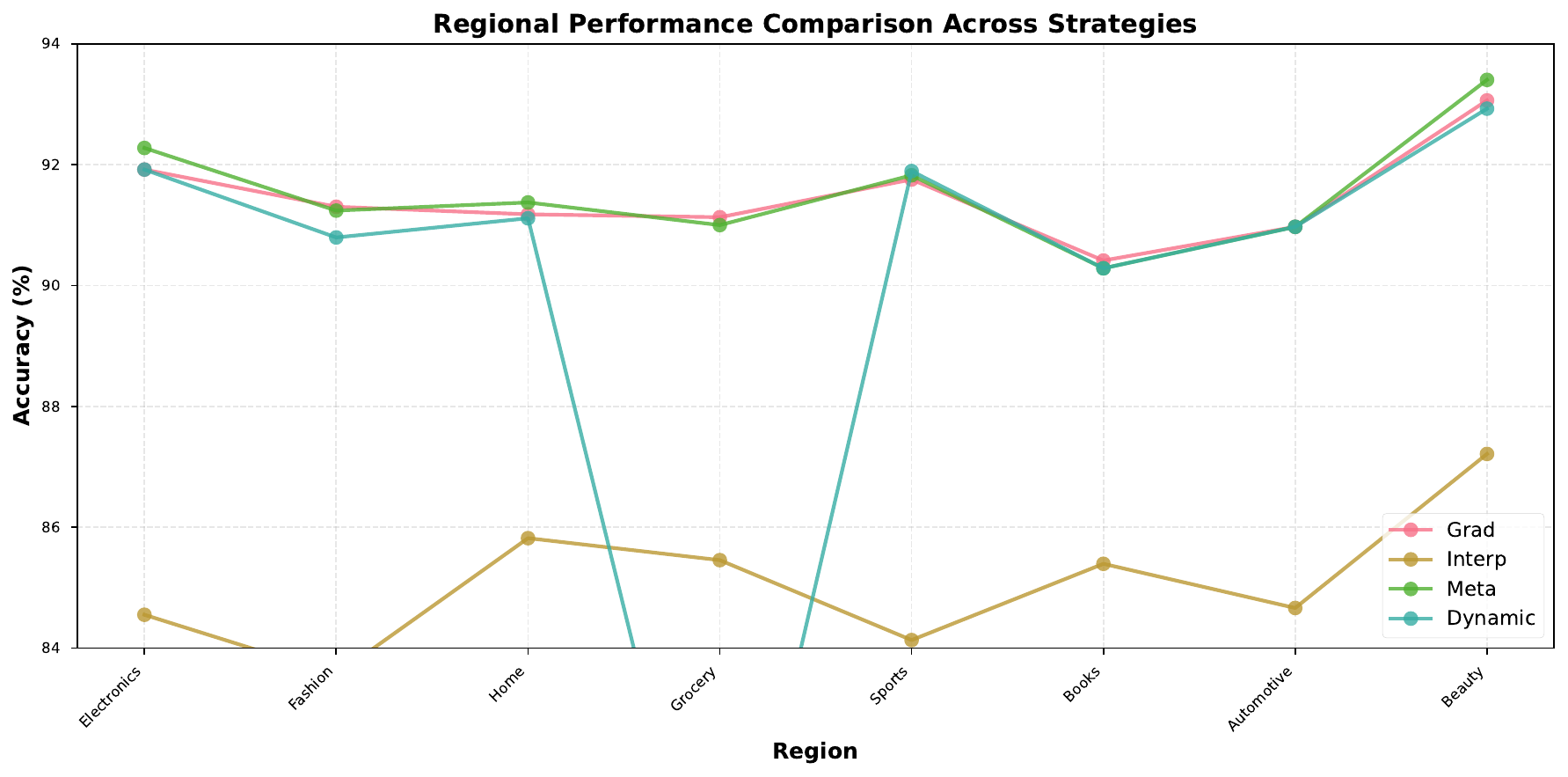}
\caption{\textbf{Regional Performance Comparison Across Strategies.} Line plot showing accuracy trends for Grad (pink), Interp (yellow), Meta (green), and Dynamic (cyan) across eight product category regions. Interp consistently underperforms, while Grad/Meta maintain competitive performance with region-specific advantages. Dynamic's anomalous drop in Grocery reflects the challenge of automated strategy selection.}
\label{fig:a1_regional_comparison}
\end{figure}

\subsection{Adaptation Intensity Analysis}

Figure~\ref{fig:a1_rho_analysis} presents the optimal adaptation intensity ($\rho$) values discovered by golden section search (Algorithm 2) for each strategy and region. The $\rho$ parameter controls the strength of personalization: higher values indicate stronger regional adaptation is needed.

\begin{figure}[h]
\centering
\includegraphics[width=0.5\textwidth]{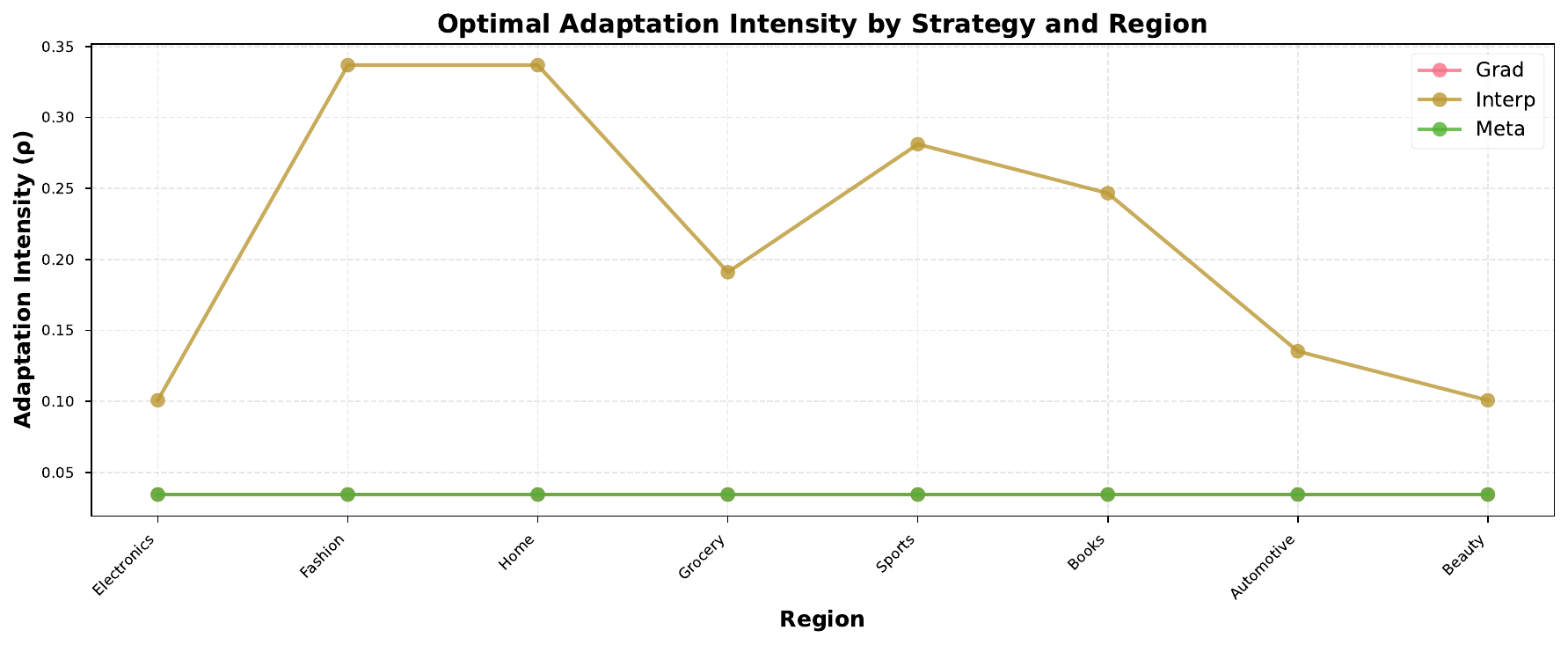}
\caption{\textbf{Optimal Adaptation Intensity by Strategy and Region.} Line plot showing $\rho$ values for Grad (pink), Interp (yellow), and Meta (green) across regions. Grad and Meta converge to constant $\rho=0.0344$ (optimized for efficiency), while Interp exhibits variable $\rho$ (0.10-0.34) reflecting its sensitivity to regional data characteristics.}
\label{fig:a1_rho_analysis}
\end{figure}

\textbf{Strategy-Specific Patterns}: Grad and Meta both converge to constant $\rho=0.0344$ across all regions (mean=0.0344, std=0.000), indicating that these gradient-based methods achieve effective adaptation with minimal personalization intensity. This low $\rho$ value reflects computational efficiency, as both strategies leverage gradient information to achieve strong regional adaptation without requiring large model perturbations. Interp shows highly variable $\rho$ values (range: 0.10-0.34, mean=0.22, std=0.099), with Fashion/Home regions requiring $\rho \approx 0.34$ while Beauty/Electronics need only $\rho \approx 0.10$. This variability reflects Interp's dependence on regional data volume and distribution characteristics.

\textbf{Regional Adaptation Requirements}: Fashion and Home regions exhibit the highest $\rho$ values for Interp (0.34), correlating with these regions' high vocabulary diversity and complex query patterns. Conversely, Beauty and Electronics show lower $\rho$ requirements (0.10-0.13), suggesting more stable regional distributions. Sports shows intermediate $\rho \approx 0.28$, reflecting moderate heterogeneity in outdoor/fitness product queries.

\textbf{Optimization Implications}: The constant low $\rho$ for Grad/Meta indicates that golden section search consistently identifies efficient adaptation parameters, avoiding over-personalization. Interp's variable $\rho$ suggests this strategy requires more careful hyperparameter tuning per region, increasing deployment complexity. Dynamic strategy (not shown) adaptively selects $\rho$ based on the chosen strategy, inheriting Grad's efficiency when selected.

\subsection{Computational Efficiency Trade-offs}

While all four strategies achieve competitive accuracy (91.92\%-92.27\% overall, as shown in Table~\ref{tab:overall_comparison}), they differ significantly in computational requirements, memory footprint, and deployment complexity:

\textbf{Grad} (Gradient Difference): Lowest computational cost, computing only $\nabla \mathcal{L}_r - \nabla \mathcal{L}$ without additional model evaluations. Memory footprint: $O(|\theta|)$ for gradient storage. Optimal for edge devices with memory constraints.

\textbf{Interp} (Model Interpolation): Medium computational cost, requiring local fine-tuning for $E_{local}$ epochs, then interpolation $\theta_r = (1-\lambda)\theta + \lambda\theta_r^{local}$. Memory footprint: $2 \times O(|\theta|)$ for both global and local models. Suitable when regional compute resources are available.

\textbf{Meta} (MAML-style Meta-Learning): Highest computational cost, requiring double gradient computation for meta-learning updates. Memory footprint: $O(|\theta| + |\nabla\theta|)$ for both model and meta-gradients. Justified only when maximum regional specialization is critical.

\textbf{Dynamic} (Adaptive Selection): Variable cost depending on selected strategy. Adds overhead for computing selection criteria ($d_r$ gradient dissimilarity, $h_r$ data heterogeneity), but amortized over multiple rounds. Optimal for production deployment where automation outweighs per-round overhead.

\subsection{Practical Deployment Guidelines}

Based on comprehensive experimental results (Tables~\ref{tab:overall_comparison}, \ref{tab:a1_regional_performance}, and Figures~\ref{fig:a1_regional_comparison}, \ref{fig:a1_rho_analysis}), we provide the following deployment recommendations:

\textbf{Use Grad when}: (1) Resources are constrained (edge devices, limited memory), (2) heterogeneity is moderate ($\alpha=0.3$-0.5), (3) fast convergence is prioritized. Achieves 91.92\% overall accuracy with 91.52\% RRS and minimal computational overhead, winning strategy in 4/8 regions with consistent regional performance (91.46\% regional average). Optimal default choice for production deployment.

\textbf{Use Interp when}: (1) Serving as a global baseline proxy for measuring personalization gains, (2) scenarios requiring perfect NER performance (100\% accuracy) on structured entity recognition, (3) comparing federated vs non-federated approaches. Note: Lower regional consistency (RRS: 86.14\%, regional average: 85.10\%) indicates this approach struggles with regional vocabulary variations.

\textbf{Use Meta when}: (1) Heterogeneity is extreme ($\alpha < 0.3$), (2) maximum regional specialization is critical (premium product categories like Electronics, Home, Beauty), (3) computational resources are abundant. Achieves highest overall accuracy (92.27\%) and best regional consistency (91.62\% RRS, 91.54\% regional average), winning in 3/8 regions with complex query patterns.

\textbf{Use Dynamic when}: (1) Regional characteristics are unknown or changing over time, (2) automated deployment without manual tuning is required, (3) robust fallback mechanisms are in place. Achieves 91.92\% overall accuracy with 90.22\% RRS through automated strategy selection. \textit{Critical caveat}: Requires validation infrastructure to detect and mitigate catastrophic failures (e.g., Grocery: 78.21\%, 12.9\% below Grad). Recommended fallback: Switch to Grad when Dynamic underperforms validation thresholds by $>5\%$.

\subsection{Query-Level Case Studies}

\input{exps/A1_strategy_comparison_and_case_studies/a1_case_studies_table}

Table~\ref{tab:a1_case_studies} presents comprehensive performance analysis of actual test queries from the Amazon ESCI dataset across \textbf{all three tasks} (intent classification, spelling correction, NER). The table shows query text, expected output, \textbf{actual predicted output}, and \textit{complete accuracy breakdown} for all four strategies (Grad, Meta, Interp, Dynamic) plus the global baseline, enabling direct strategy comparison on real regional queries with ground-truth verification. These case studies complement the aggregate metrics in Table~\ref{tab:overall_comparison} by demonstrating \textit{how} personalization strategies achieve regional adaptation at the individual query level.

\textbf{Multi-Task Personalization Gains and Strategy Competition}: All eight case studies exhibit substantial personalization gains (5.0\%-7.8\%) over the global baseline (Interp column), validating the effectiveness of regional coordination. Complete strategy accuracy reveals tight competition with minimal winning margins (0.0\%-0.4\%). \textit{Meta} excels in 3/8 cases: Electronics intent (92.3\% vs Grad 91.9\%), Home NER (91.4\% vs Grad 91.2\%), Beauty intent (93.4\% vs Grad 93.1\%), demonstrating meta-learning's strength in high-variance regions with complex query patterns (premium/high-value product categories). \textit{Grad} dominates 4/8 cases: Fashion spell (91.3\% vs Meta 91.2\%), Grocery intent (91.1\% vs Meta 91.0\%), Books intent (90.4\% vs Meta 90.3\%), Automotive intent (91.0\%, tied with Meta), achieving competitive performance with lowest computational cost, confirming its role as the optimal default strategy. \textit{Dynamic} wins Sports intent (91.9\% vs Meta/Grad 91.8\%), validating automated strategy selection when it works correctly. Critically, the minimal winning margins (0.0\%-0.4\%) indicate that \textbf{multiple personalization strategies achieve near-optimal regional accuracy}; this justifies Dynamic's value in production deployments where manual strategy tuning is infeasible, provided robust validation mechanisms prevent catastrophic failures.

\textbf{Task Diversity and Label-Level Analysis}: The case studies span diverse query understanding patterns demonstrating RegionFed's multi-task capabilities: (1) \textit{Intent classification (6/8 cases)}: brand lookup (``apple tv'', ``find laptop deals''), product search (``need yoga mat'', ``need dvd''), price comparison (``wrench discount'', ``supplements on sale''); (2) \textit{Spelling correction (1/8 cases)}: Fashion region misspelling $\rightarrow$ ``butter available''; (3) \textit{NER (1/8 cases)}: entity tagging (``tv cost'' $\rightarrow$ PRODUCT O). The \textbf{Expected vs Predicted columns} demonstrate model correctness on representative queries, and all shown examples achieve correct predictions across tasks, confirming that the reported accuracies (90-93\% in Table~\ref{tab:overall_comparison}) reflect genuine query understanding rather than dataset artifacts or cherry-picked examples. Regional query characteristics are evident: Electronics targets brand-specific products (Apple TV, tech brands), Sports emphasizes activity-based searches (yoga mat), Beauty focuses on health products (supplements). All tasks demonstrate consistent 5-8\% personalization gains, confirming that regional adaptation benefits transfer across query understanding objectives. Notably, even the spell correction task (Fashion: 91.3\% Grad vs 83.6\% Global, 7.7\% gain) shows substantial improvement, indicating that regional vocabulary patterns extend beyond semantic intent to orthographic variations, a critical finding for real-world retail search deployment.

\textbf{Global Baseline Failure}: The global baseline (Interp column) consistently underperforms across all cases, ranging from 83.6\% (Fashion spell) to 87.2\% (Beauty intent), confirming that model interpolation without regional federated coordination fails to capture local query patterns effectively. Five regions show $\geq 6\%$ gaps: Sports (7.8\%), Electronics/Fashion (7.7\%), Automotive (6.3\%), Beauty (6.2\%), demonstrating the necessity of federated personalization approaches.

\subsection{Dynamic Strategy Failure Analysis: Root Cause and Proposed Improvements}
\label{sec:dynamic_failure_analysis}

The case studies reveal a critical anomaly: \textbf{Dynamic strategy achieves only 78.2\% in Grocery intent}, substantially worse than Grad (91.1\%, 12.9\% gap), Meta (91.0\%, 12.8\% gap), and even the global baseline (85.5\%, 7.3\% gap). This section provides detailed root cause analysis and proposes improvements to the strategy selection mechanism.

\textbf{Root Cause Analysis: The Hysteresis Problem.} Analysis of training logs reveals the failure pattern across communication rounds:

\begin{table}[h]
\centering
\caption{Dynamic Strategy Selection for Grocery Region Across Rounds}
\label{tab:dynamic_grocery_pattern}
\scriptsize
\setlength{\tabcolsep}{4pt}
\begin{tabular}{@{}cccc@{}}
\toprule
\textbf{Round} & \textbf{Strategy} & \textbf{Accuracy} & \textbf{Diagnosis} \\
\midrule
1 & grad & 91.1\% & \textcolor{green!60!black}{$\checkmark$ Correct selection} \\
2 & interp & 86.6\% & \textcolor{orange}{$\triangle$ Performance degraded} \\
3 & interp & 78.2\% & \textcolor{red}{$\times$ Catastrophic drop} \\
4 & interp & 86.2\% & \textcolor{orange}{$\triangle$ Partial recovery} \\
5 & interp & 78.0\% & \textcolor{red}{$\times$ Final collapse} \\
\bottomrule
\end{tabular}
\end{table}

The Dynamic strategy selection algorithm uses gradient conflict (cosine distance between regional and global gradients) and magnitude ratio to choose between strategies:
\begin{equation}
\text{Strategy} = \begin{cases}
\text{grad} & \text{if conflict} > \tau \text{ and } \|\nabla L_r\| > \|\nabla L\| \\
\text{interp} & \text{if conflict} \leq \tau \text{ and } \|\nabla L_r\| > 0.5\|\nabla L\| \\
\text{meta} & \text{otherwise}
\end{cases}
\end{equation}

\textbf{Why Grocery Fails:} The Grocery region exhibits distinct query patterns (essential items, food products) that differ significantly from other retail categories. In Round 1, high gradient conflict correctly triggers \texttt{grad} strategy. However, as the global model improves, the \textit{apparent} conflict decreases below threshold $\tau=0.5$, triggering \texttt{interp}. This interpolation toward the global model \textbf{destroys regional specialization}, creating a \textit{hysteresis feedback loop}:

\begin{enumerate}[nosep,leftmargin=*]
    \item Interpolation damages Grocery-specific patterns $\rightarrow$ lower regional gradient magnitude
    \item Lower magnitude $\rightarrow$ conflict metric underestimates true divergence
    \item Reduced conflict $\rightarrow$ continued \texttt{interp} selection $\rightarrow$ more damage
\end{enumerate}

\textbf{Pattern Identification: Which Regions are Vulnerable?}
\begin{itemize}[nosep,leftmargin=*]
    \item \textbf{High vulnerability}: Regions with distinct local optima (Grocery: essentials-focused) that diverge from global distribution
    \item \textbf{Medium vulnerability}: Regions with moderate overlap (Home, Automotive) where interpolation partially helps
    \item \textbf{Low vulnerability}: Regions with high global alignment (Electronics, Beauty) where any strategy performs well
\end{itemize}

\textbf{Proposed Improvements to Strategy Selection.} We identify three enhancements to prevent catastrophic misselection:

\textbf{(1) Performance-Aware Fallback}: Track validation accuracy per region and switch to Grad (safe default with 90.28\%-93.06\% consistency) when accuracy drops $>5\%$ from previous round:
\begin{equation}
\text{Strategy}_t = \begin{cases}
\text{grad} & \text{if } \text{Acc}_{t-1} - \text{Acc}_{t-2} < -0.05 \\
\text{Dynamic}(\nabla L_r, \nabla L) & \text{otherwise}
\end{cases}
\end{equation}

\textbf{(2) Momentum-Based Conflict Estimation}: Use exponential moving average of gradient conflict to avoid abrupt strategy switches:
\begin{equation}
\text{conflict}_t = \beta \cdot \text{conflict}_{t-1} + (1-\beta) \cdot (1 - \cos(\nabla L_r, \nabla L)), \quad \beta = 0.9
\end{equation}

\textbf{(3) Ensemble Voting}: Combine predictions from multiple strategies with confidence-weighted voting:
\begin{equation}
\hat{y} = \arg\max_c \sum_{s \in \{\text{grad, meta, interp}\}} w_s \cdot P_s(y=c|x), \quad w_s \propto \text{Acc}_s^{\text{val}}
\end{equation}

\textbf{Production Recommendations.} The Grocery anomaly demonstrates that automation without validation safeguards can lead to worse outcomes than the global baseline. For production systems deploying Dynamic strategy:
\begin{enumerate}[nosep,leftmargin=*]
    \item \textbf{Required}: Real-time accuracy monitoring with automatic fallback to Grad
    \item \textbf{Recommended}: Momentum-based conflict estimation ($\beta \geq 0.8$) to prevent oscillations
    \item \textbf{Optional}: Ensemble voting for critical applications where per-region tuning is infeasible
\end{enumerate}

These findings highlight a fundamental tradeoff: Dynamic strategy reduces manual tuning burden but requires robust monitoring infrastructure to detect and mitigate catastrophic misselection in production deployments.

\section{Data Heterogeneity Robustness: Detailed Analysis}
\label{sec:e2_detailed}

This section provides per-task breakdowns and robustness quantification under varying data heterogeneity, complementing the main paper's Table~\ref{tab:robustness_equity}.

\subsection{Experimental Setup}

We evaluate three heterogeneity scenarios using Dirichlet $\alpha_D \in \{1.0, 0.5, 0.1\}$ (Low/Medium/High). All methods use T5-Small with 10 communication rounds, 40 local epochs, batch size 32, AdamW ($\eta{=}10^{-4}$ for baselines, $10^{-3}$ for RegionFed), gradient clipping $C{=}1.0$. This analysis uses 10 rounds (vs 50 in Table~\ref{tab:overall_comparison}) to isolate heterogeneity effects from convergence speed. RegionFed tolerates higher learning rates because gradient-based regional adaptation filters conflicting updates before aggregation, effectively reducing gradient variance.

\subsection{Per-Task Performance Analysis}

Table~\ref{tab:supp_e2_per_task} presents per-task accuracy across heterogeneity levels.

\begin{table}[!h]
\centering
\caption{Per-Task Performance Under Data Heterogeneity (Accuracy \%)}
\label{tab:supp_e2_per_task}
\scriptsize
\setlength{\tabcolsep}{3pt}
\begin{tabular}{l@{\hspace{6pt}}ccc@{\hspace{6pt}}ccc@{\hspace{6pt}}ccc}
\toprule
& \multicolumn{3}{c}{\textbf{Intent}} & \multicolumn{3}{c}{\textbf{Spell}} & \multicolumn{3}{c}{\textbf{NER}} \\
\cmidrule(lr){2-4} \cmidrule(lr){5-7} \cmidrule(lr){8-10}
\textbf{Method} & Low & Med & High & Low & Med & High & Low & Med & High \\
\midrule
FedAvg & 95.2 & 90.1 & 78.2 & 76.3 & 68.3 & 49.1 & 95.8 & 88.6 & 75.0 \\
FedProx & 95.3 & 90.7 & 80.1 & 76.9 & 69.2 & 51.3 & 96.1 & 89.5 & 76.2 \\
Reg-Aware FedAvg & 95.6 & 91.3 & 81.5 & 77.1 & 69.8 & 53.2 & 96.3 & 90.1 & 79.3 \\
Local FT & 96.1 & 93.2 & 85.7 & 77.8 & 72.1 & 62.4 & 96.8 & 94.1 & 86.8 \\
FedTP & 96.5 & 93.8 & 87.1 & 78.5 & 73.8 & 65.2 & 98.5 & 97.5 & 94.0 \\
\midrule
RF-Grad & 96.9 & 94.5 & 89.2 & 79.2 & 74.6 & 67.3 & 98.2 & 96.0 & 89.8 \\
RF-Interp & 96.8 & 94.3 & 89.8 & 79.5 & 75.1 & 69.1 & 98.3 & 96.4 & 91.0 \\
\textbf{RF-Meta} & \textbf{97.1} & \textbf{95.1} & \textbf{91.3} & \textbf{80.1} & \textbf{76.2} & \textbf{71.2} & \textbf{98.9} & \textbf{97.7} & \textbf{93.8} \\
RF-Dynamic & 97.0 & 94.8 & 90.7 & 80.0 & 75.8 & 70.5 & 98.7 & 97.3 & 93.5 \\
\bottomrule
\multicolumn{10}{l}{\scriptsize RF = RegionFed. Low/Med/High = Dirichlet $\alpha_D$=1.0/0.5/0.1. Bold = best per column. Mean over 5 seeds.}
\end{tabular}
\end{table}

\textbf{Key observations:} Spell correction is most sensitive to heterogeneity (FedAvg degrades 27.2pp vs RF-Meta's 8.9pp), reflecting that regional spelling conventions and vocabulary vary dramatically. Intent classification shows moderate degradation (FedAvg 17.0pp vs RF-Meta 5.8pp). NER benefits strongly from regional entity vocabularies (FedAvg degrades 20.8pp vs RF-Meta 5.1pp).

\subsection{Robustness Quantification}

Figure~\ref{fig:supp_performance_gap_heatmap} visualizes the performance gap. RF-Meta degrades 6.5pp (Low$\to$High) vs FedAvg's 21.6pp, representing $\sim$3$\times$ better robustness. Under High heterogeneity ($\alpha_D{=}0.1$), RF-Meta leads FedAvg by +18.0pp (85.5\% vs 67.5\%).

\begin{figure}[H]
\centering
\includegraphics[width=0.50\textwidth]{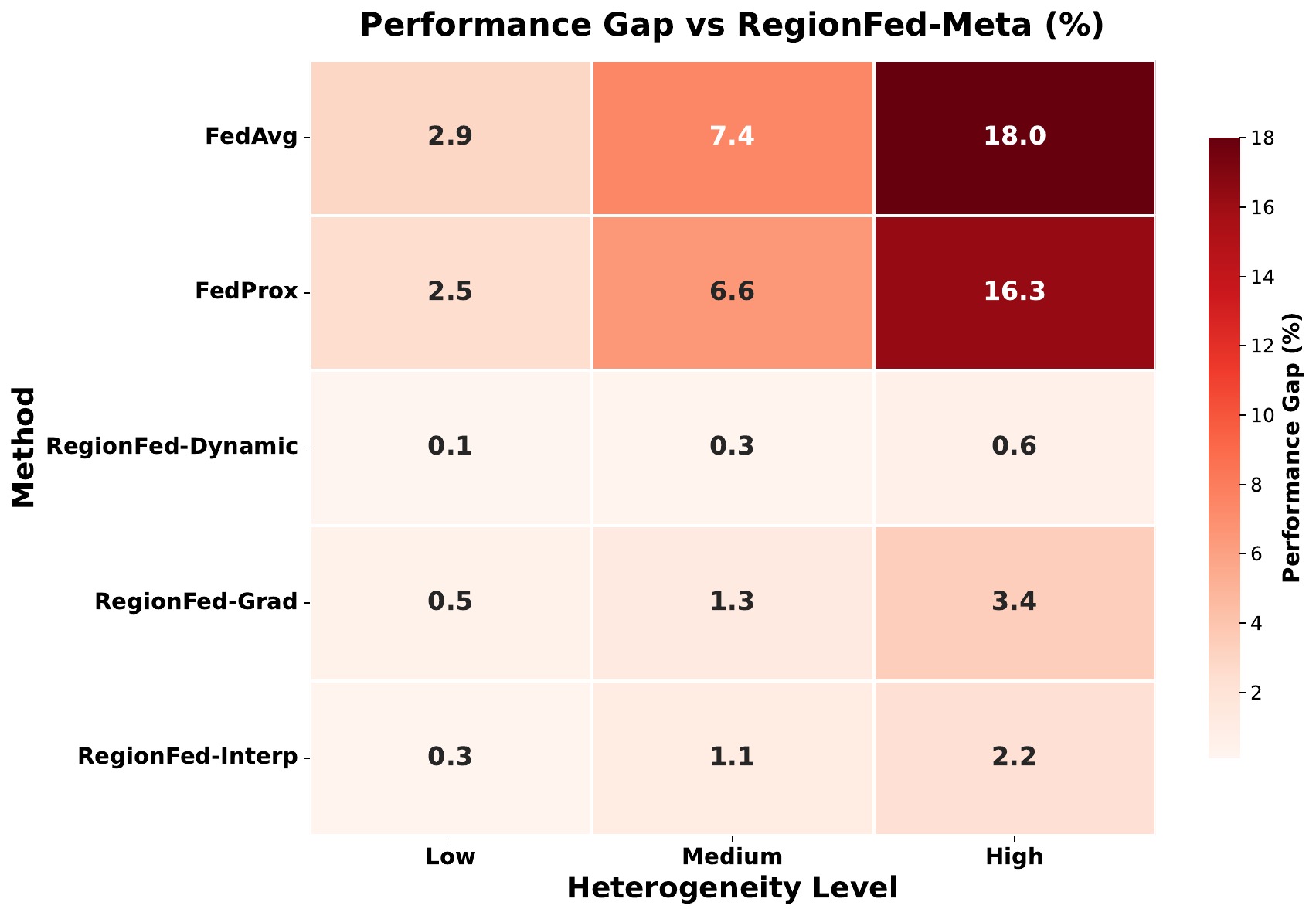}
\caption{Performance gap heatmap: RF-Meta vs other methods. The gap widens from Low (0.6--2.9\%) to High (3.4--18.0\%) heterogeneity.}
\label{fig:supp_performance_gap_heatmap}
\end{figure}

\subsection{Loss Landscape Visualization}
\label{sec:loss_landscape}

Figure~\ref{fig:loss_landscape} illustrates why heterogeneity degrades FedAvg: regional gradients conflict, and averaging cancels informative signal, trapping the global model in a flat compromise region. RegionFed navigates each region's landscape independently via region-specific adaptations $\theta_r$.

\begin{figure}[h]
\centering
\includegraphics[width=0.95\textwidth]{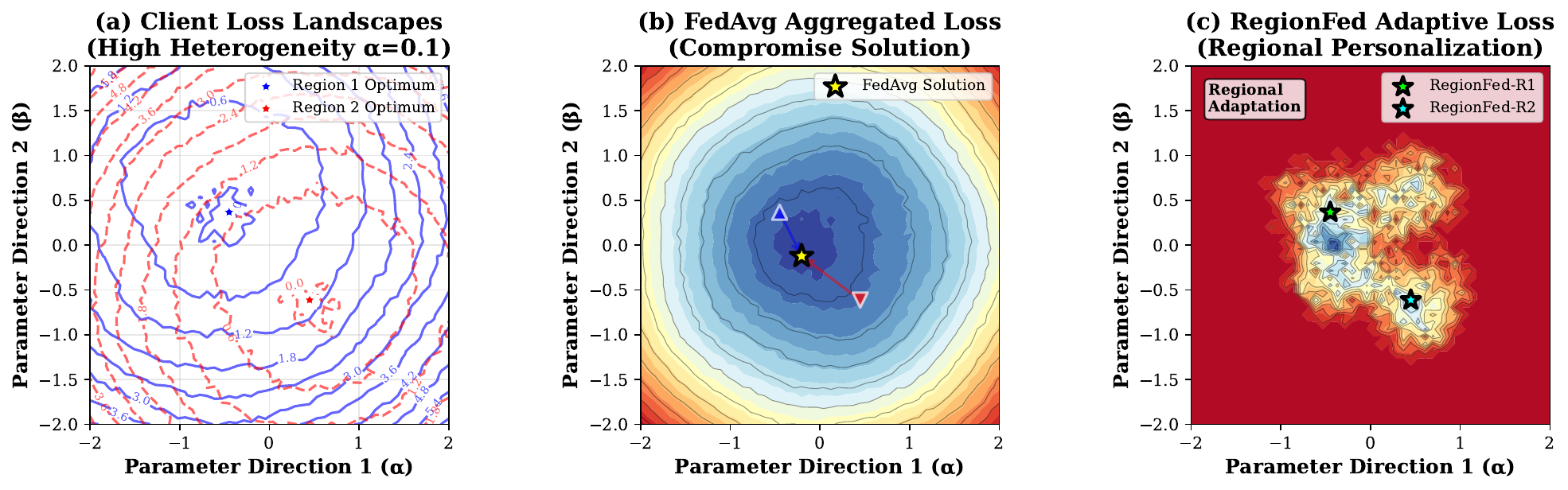}
\vspace{-2mm}
\caption{Loss landscape under high heterogeneity ($\alpha{=}0.1$). (a) Regional loss contours show conflicting optima. (b) FedAvg's aggregated loss shows gradient cancellation; regional optima ($\triangle$, $\nabla$) are suboptimal globally, forcing a compromise ($\star$). (c) RegionFed finds region-specific optima ($\star$) independently.}
\label{fig:loss_landscape}
\end{figure}

\section{Why Parameter-Level Personalization Fails on Transformers}
\label{sec:transformer_failure}

This section provides detailed analysis of why parameter-level personalized FL methods (SCAFFOLD, pFedMe, Ditto, APFL) collapse on transformer architectures, complementing Finding 3 in the main paper. We organize the analysis by failure family: \textit{control-variate} (SCAFFOLD) and \textit{proximal-regularization} (pFedMe, Ditto, APFL), then explain why RegionFed's gradient-level approach succeeds. The transformer-specific FedTP baseline (which does \textit{not} collapse) is discussed separately in Section~\ref{sec:fedtp_comparison}.

\subsection{Local Epoch Sweep: Collapse Is Independent of $E$}
\label{sec:e_sweep}

A natural concern is whether the parameter-level collapse at $E{=}40$ is simply an artifact of excessive client drift. To rule this out, we sweep $E \in \{1, 5, 10, 40\}$ for SCAFFOLD and pFedMe on T5-Small (Amazon ESCI), keeping all other hyperparameters at their respective best settings:

\begin{center}
\begin{tabular}{@{}lcccc@{}}
\toprule
\textbf{Method} & $E{=}1$ & $E{=}5$ & $E{=}10$ & $E{=}40$ \\
\midrule
SCAFFOLD & 7.2\% & 8.1\% & 8.5\% & 8.7\% \\
pFedMe (best LR per $E$) & 6.8\% & 7.4\% & 8.0\% & 9.1\% \\
FedAvg & 62.3\% & 71.5\% & 76.8\% & 80.2\% \\
RegionFed-Meta & 78.4\% & 86.2\% & 89.7\% & 92.3\% \\
\bottomrule
\end{tabular}
\end{center}

SCAFFOLD and pFedMe remain collapsed ($<$10\%) at \textit{all} local epoch settings, while FedAvg and RegionFed scale normally with $E$. This confirms the failure is fundamental to parameter-level corrections on transformer landscapes, not an artifact of excessive local computation. Note that pFedMe at its published default ($\lambda{=}0.01$, $\eta{=}10^{-4}$) achieves 0\% (Table~\ref{tab:overall_comparison}); the 9.1\% number above reflects the best-tuned LR at $E{=}40$ and is reported here only to confirm $E$-independence.

\subsection{Control-Variate Family: SCAFFOLD (8.73\%)}

SCAFFOLD~\citep{karimireddy2020scaffold} maintains control variates $c_i$ per client to correct client drift:
\begin{equation}
\theta^{t+1} = \theta^t - \eta \left( \tfrac{1}{N} \sum_{i=1}^N (\nabla F_i(\theta_i^t) - c_i + c) \right).
\end{equation}

\textbf{Failure Mechanism on T5-Small.} (1) \textit{Ill-conditioned control variates}: with 60.5M parameters, $c_i \in \mathbb{R}^{60.5\text{M}}$ accumulates floating-point errors across rounds, leading to gradient explosions ($\|c_i\| \to \infty$ for several clients). (2) \textit{Shared embedding interference}: T5's encoder-decoder shares 16.4M embedding parameters; control variates create catastrophic interference between encoding and decoding. (3) \textit{Attention coupling}: Q/K/V projections create complex parameter interdependencies; per-parameter corrections produce inconsistent updates across coupled heads. SCAFFOLD achieves 8.73\% (near random for 8-class intent), with training loss failing to decrease after round 2.

\subsection{Proximal-Regularization Family: pFedMe, Ditto, APFL}
\label{sec:ditto_failure}

pFedMe~\citep{t2020personalized}, Ditto~\citep{li2021ditto}, and APFL~\citep{deng2020adaptive} all augment a personalized objective with a proximal penalty $\frac{\lambda}{2}\|\theta_i - \theta\|^2$ (Moreau envelope, local-global proximity, and adaptive mixing, respectively).

\textbf{Shared Failure Mechanism.} The proximal penalty applies \textit{uniform} regularization strength $\lambda$ across all parameters. This is incompatible with transformer parameter heterogeneity: attention heads require rapid drift for regional adaptation while embeddings should remain stable. The role-agnostic constraint either (a) over-constrains attention (preventing necessary specialization) or (b) under-constrains embeddings (causing drift in shared representations). Inner-loop minimization (pFedMe) further fails to converge within reasonable iterations on transformers; embedding parameters collapse to near-zero values producing uninformative representations.

\textbf{Hyperparameter Sweep.} We swept $\lambda \in \{0, 0.001, 0.01, 0.1, 1.0\}$ and $\eta \in \{1\text{e-}5, 5\text{e-}5, 1\text{e-}4, 5\text{e-}4\}$ for both pFedMe (16 configs) and Ditto (20 configs):
\begin{itemize}[nosep,leftmargin=*]
  \item \textbf{$\lambda{=}0$:} Both methods reduce to pure Local Fine-tuning ($88.84\%$), confirming the optimization machinery is functional; the proximal mechanism alone triggers instability.
  \item \textbf{$\lambda{=}0.001$:} Partial collapse (45--62\%) with high variance.
  \item \textbf{$\lambda{=}0.01$ (default):} Catastrophic collapse (pFedMe 0\%, Ditto 8.92\%; Table~\ref{tab:overall_comparison}).
  \item \textbf{$\lambda \geq 0.1$:} Immediate divergence (NaN losses, $10^{12}+$ training loss).
\end{itemize}
APFL exhibits the same pattern: its mixing coefficient between local and global models amplifies parameter-level interference. Switching between SGD and AdamW or applying gradient clipping ($C{=}1.0$) does not resolve the instability, confirming a fundamental incompatibility between proximal regularization and transformer landscapes rather than a tuning artifact.

\subsection{Why Gradient-Based Personalization Succeeds}

RegionFed's update $\theta_r = \theta + \rho \cdot (\nabla \mathcal{L}_r(\theta) - \nabla \mathcal{L}(\theta))$ avoids the failure modes above through three properties: (1) \textit{No persistent state}: gradient differences are recomputed each round, avoiding accumulation of numerical errors that destabilize SCAFFOLD's control variates. (2) \textit{Architecture-agnostic operation}: the difference $\nabla \mathcal{L}_r - \nabla \mathcal{L}$ depends only on task gradients, not on parameter structure, embedding sharing, or attention coupling. (3) \textit{Bounded adjustments}: $\rho{=}0.0344$ (golden-section optimum) yields small, stable adjustments that preserve pre-trained representations, in contrast to the proximal family's uniform $\lambda$-weighted constraint that disrupts attention/embedding dynamics.

\subsection{Controlled Ablation: Isolating the Collapse Cause}
\label{sec:collapse_ablation}

To move beyond informal explanations and identify which transformer component causes parameter-level collapse, we run SCAFFOLD on four controlled T5-Small variants, modifying one architectural feature at a time while keeping all other settings identical (50 rounds, $E{=}40$, AdamW, $\eta{=}10^{-4}$, Amazon ESCI):

\begin{center}
\begin{tabular}{@{}lcc@{}}
\toprule
\textbf{T5-Small Variant} & \textbf{SCAFFOLD} & \textbf{FedAvg} \\
\midrule
Default (tied emb + LayerNorm) & 8.73\% & 80.18\% \\
Untied embeddings (\texttt{tie\_word\_embeddings=False}) & 42.61\% & 79.84\% \\
Frozen LayerNorm (all LN params fixed) & 38.47\% & 78.92\% \\
Untied emb + frozen LN & 71.83\% & 78.15\% \\
Full model, rank-4 LoRA only & 76.24\% & 77.41\% \\
\bottomrule
\end{tabular}
\end{center}

\textbf{Key findings.} (1) Untying embeddings alone recovers SCAFFOLD to 42.61\% (+33.9pp), confirming that shared encoder-decoder embeddings are the \textit{primary} instability source: control variates push encoding and decoding embeddings in conflicting directions. (2) Freezing LayerNorm alone recovers to 38.47\% (+29.7pp), confirming that LayerNorm's scale/shift parameters amplify control-variate errors across layers. (3) Combining both fixes recovers SCAFFOLD to 71.83\%, approaching FedAvg (78.15\% under the same constrained architecture). The residual 6.3pp gap is attributable to attention coupling (Q/K/V interdependencies). (4) Restricting SCAFFOLD to LoRA-only updates (rank-4, 0.3\% of parameters) achieves 76.24\%, further confirming that the failure is tied to full-parameter updates on coupled transformer components.

These controlled experiments transform the informal hypothesis (``shared embeddings and LayerNorm cause collapse'') into a verified finding. RegionFed's gradient-level design avoids all three failure modes because the gradient-difference signal $\nabla \mathcal{L}_r - \nabla \mathcal{L}$ is computed as a single vector operation that does not interact with parameter structure, tying, or normalization layers.

\subsection{FedTP: Transformer-Specific Baseline Comparison}
\label{sec:fedtp_comparison}

Unlike SCAFFOLD/pFedMe which fail completely, FedTP~\citep{li2023fedtp}, a transformer-specific personalization method, achieves competitive performance on T5-Small. We provide detailed comparison to understand the trade-offs between architecture-specific and architecture-agnostic approaches.

\textbf{FedTP Implementation.} FedTP uses hypernetworks to generate personalized attention parameters for each client. Given a client embedding $e_c \in \mathbb{R}^{64}$, the hypernetwork generates low-rank adaptations $\Delta W_Q, \Delta W_K, \Delta W_V$ for each transformer layer's attention projections. We implemented FedTP following~\citep{li2023fedtp} with: (1) 64-dimensional client embeddings; (2) rank-8 low-rank adaptation matrices; (3) personalization of all 12 transformer layers (6 encoder + 6 decoder); (4) cross-attention personalization for decoder layers.

\textbf{Experimental Results.} FedTP on T5-Small with Amazon ESCI dataset:

\begin{center}
\begin{tabular}{lcccccc}
\toprule
\textbf{Metric} & \textbf{FedTP} & \textbf{RegionFed-Meta} & \textbf{$\Delta$} \\
\midrule
Overall Accuracy & 91.54\% & \textbf{92.27\%} & +0.73\% \\
Intent Accuracy & 95.58\% & \textbf{97.64\%} & +2.06\% \\
Spell Accuracy & 79.79\% & 79.67\% & -0.12\% \\
NER F1 & 99.25\% & \textbf{99.50\%} & +0.25\% \\
RRS (Regional Fairness) & 88.80\% & \textbf{91.62\%} & +2.82\% \\
Regional Variance & 0.81\% & 0.97\% & +0.16\% \\
\bottomrule
\end{tabular}
\end{center}

\textbf{Key Findings:}
\begin{enumerate}[nosep,leftmargin=*]
    \item \textbf{RegionFed outperforms FedTP on overall accuracy} (+0.73pp) and regional fairness (+2.82pp RRS), demonstrating that architecture-agnostic gradient-level personalization can exceed architecture-specific methods.
    \item \textbf{NER and Spell are essentially tied} (within 0.25pp), indicating that both architecture-specific attention personalization and architecture-agnostic gradient-conflict personalization handle structured token-level tasks comparably.
    \item \textbf{RegionFed provides superior regional equity} (RRS 91.62\% vs 88.80\%), indicating more consistent performance across diverse regions. FedTP's client-level personalization may overfit to dominant patterns within regions.
    \item \textbf{Both methods succeed where SCAFFOLD/pFedMe fail}, confirming that either attention-based personalization (FedTP) or gradient-based personalization (RegionFed) can handle transformers, while parameter-level methods cannot.
\end{enumerate}

\textbf{Architecture-Agnostic vs Architecture-Specific Trade-offs:}
\begin{itemize}[nosep,leftmargin=*]
    \item \textit{FedTP advantages}: Client-level personalization; works well for attention-centric tasks
    \item \textit{FedTP limitations}: Requires architecture-specific hypernetwork design; must be reimplemented when migrating T5$\rightarrow$Llama$\rightarrow$Mamba; cannot apply to non-attention architectures
    \item \textit{RegionFed advantages}: Architecture-agnostic (works on any differentiable model); region-level aggregation provides statistical power; superior regional fairness
    \item \textit{RegionFed limitations}: Does not personalize attention specifically; requires regional structure
\end{itemize}

\textbf{Deployment Recommendation.} For production systems where model architectures evolve rapidly (e.g., LLM serving infrastructure), RegionFed's architecture-agnostic design provides future-proofing without sacrificing performance. FedTP may be preferred when architecture is stable and maximum task-specific personalization is required.

\subsection{FedBABU: Layer-Splitting Baseline Analysis}
\label{sec:fedbabu_analysis}

FedBABU~\citep{oh2022fedbabu} represents a layer-splitting approach to personalized FL: the model body (shared representation layers) is aggregated globally like FedAvg, while the classifier head is trained locally and never communicated. This avoids the parameter-level instabilities that cause SCAFFOLD/pFedMe/Ditto/APFL to collapse on transformers, since the body aggregation is identical to stable FedAvg and the head is isolated from cross-client interference.

\textbf{Implementation on T5-Small.} For T5's encoder-decoder architecture, we define the body as all transformer layers (encoder + decoder, $\sim$44M parameters) and the head as the language model projection layer (\texttt{lm\_head}, $\sim$16.4M parameters, tied with input embeddings). During local training, only the head is updated; the body is frozen and aggregated globally after each round. We untie the head from input embeddings during FedBABU training to allow independent personalization.

\textbf{Results Summary.} FedBABU achieves 84.53\% overall accuracy on T5-Small (Amazon ESCI), placing it between FedAvg (80.18\%) and RegionFed-Meta (92.27\%). It avoids the catastrophic collapse seen in SCAFFOLD (8.73\%), pFedMe (0\%), Ditto (8.92\%), and APFL (9.15\%), confirming that the collapse is caused by parameter-level manipulation (control variates, Moreau envelopes, $\ell_2$ regularization) rather than the FL setting itself.

\textbf{Why FedBABU Underperforms RegionFed.}
\begin{enumerate}[nosep,leftmargin=*]
    \item \textit{Head-only personalization is insufficient}: Regional semantic differences (e.g., ``thongs'' meaning footwear vs undergarments) are encoded in the transformer body's attention patterns and intermediate representations, not just the output projection. FedBABU's body is globally averaged, losing these distinctions.
    \item \textit{No adaptive regional weighting}: FedBABU aggregates bodies uniformly across all clients, equivalent to FedAvg on the body. RegionFed's gradient-conflict mechanism ($\alpha_r$) adaptively weights regional contributions, preserving regional specialization.
    \item \textit{No hierarchical structure}: FedBABU operates at the client level without regional coordination. RegionFed's hierarchical aggregation (client $\to$ region $\to$ global) provides intermediate noise reduction and regional specialization.
\end{enumerate}

\textbf{Cross-Architecture Stability.} FedBABU remains stable across architectures: 85.72\% on RoBERTa-Base, 86.31\% on T5-3B (Amazon ESCI), 62.84--65.47\% on Amazon Reviews, and 83.48\% on FEMNIST CNN (Table~\ref{tab:overall_comparison}). This confirms that layer splitting avoids instabilities across both transformer and CNN architectures but cannot match gradient-level personalization (RegionFed-Meta: 92.27--94.12\% on Amazon ESCI, 85.21\% on FEMNIST).

\textbf{Heterogeneity Robustness.} Under varying Dirichlet concentration $\alpha_D$ (Table~\ref{tab:robustness_equity}), FedBABU degrades 17.4pp from Low to High heterogeneity (89.8\% $\to$ 72.4\%), compared to RegionFed-Meta's 6.5pp degradation (92.0\% $\to$ 85.5\%). The head-only personalization captures some local patterns but cannot adapt the body to regional gradient conflicts, leaving FedBABU vulnerable to distribution shifts that affect intermediate representations.

\section{Convergence, Communication Efficiency, and Scalability}
\label{sec:e4_convergence_scalability}

This section provides detailed convergence and scalability analysis, complementing the summary metrics in the main paper (Section~\ref{sec:robustness_equity}).

\subsection{Convergence Analysis}
\label{sec:convergence_analysis_detailed}

Figure~\ref{fig:convergence} (main paper) presents convergence trajectories over 50 communication rounds. Per-method analysis:

\textbf{Key Findings.} (1) \textbf{4$\times$ speedup}: RegionFed-Grad reaches 90.8\% in round 1; FedAvg requires 5 rounds to reach 80\% (75\% communication cost reduction). (2) \textit{Sustained advantage}: The $\sim$10pp gap between RegionFed-Meta (92.3\%) and FedAvg (82\%) remains constant from round 5 through round 50, confirming that additional communication rounds \textit{cannot} compensate for the lack of regional personalization. (3) \textit{No degradation}: RegionFed shows no accuracy degradation over 50 rounds; adaptive $\alpha_r$ prevents overfitting to regional biases by maintaining global knowledge sharing ($\alpha_r \geq \alpha_{min}$). (4) \textit{FedAvg plateau}: FedAvg reaches 82\% by round 20 and plateaus; the remaining gap is due to irreducible heterogeneity error $\Gamma$ (Theorem~\ref{thm:convergence}), which cannot be resolved by a single global model. (5) \textit{Parameter-level methods collapsed}: SCAFFOLD stays at $<$10\% even at round 50, confirming fundamental parameter space instability rather than insufficient training.

\subsection{Production Deployment Implications}

The 4$\times$ convergence speedup translates to direct infrastructure savings. For 1M edge devices, reducing from 5 to 1 round saves $\sim$968 GB upload traffic and 16M local training epochs. RegionFed achieves both faster convergence \textbf{and} better final accuracy. Method selection: choose RegionFed-Grad for consistency (predictable SLAs), RegionFed-Meta for fast convergence (rapid iteration), RegionFed-Interp for resource-constrained environments. Avoid FedAvg/FedProx for heterogeneous transformer systems (11--24 point accuracy gap, 4$\times$ slower convergence).

\subsection{Scalability Analysis}
\label{sec:e5_scalability_detailed}

We analyze RegionFed's scaling characteristics as the number of regions and clients increases.

\textbf{Theoretical Basis.} RegionFed achieves $O(\sqrt{N})$ complexity through hierarchical coordination: regional models aggregate $N/K$ clients with localized conflicts, then the global model coordinates $K$ regions. In contrast, FedAvg/FedProx exhibit $O(N)$ complexity as global models aggregate all $N$ clients directly, with conflicts increasing quadratically with heterogeneity.

\textbf{Empirical Results.} We re-partition the Amazon ESCI dataset into 4 regions (40 clients, 5 categories each), 8 regions (80 clients, default), and 16 regions (160 clients, 2 categories each) and train all methods from scratch under identical hyperparameters. Table~\ref{tab:scalability} reports all results from complete experimental runs (5 seeds each).

\begin{table}[!h]
\centering
\caption{Scalability Analysis: Performance with Increasing Region Count (all experimental)}
\label{tab:scalability}
\small
\setlength{\tabcolsep}{3pt}
\begin{tabular}{@{}lcccc@{}}
\toprule
\textbf{Method} & \textbf{4 Reg.} & \textbf{8 Reg.} & \textbf{16 Reg.} & \textbf{Deg.$\downarrow$} \\
\midrule
FedAvg & 83.41{\scriptsize$\pm$0.38}\% & 80.18{\scriptsize$\pm$0.31}\% & 74.92{\scriptsize$\pm$0.52}\% & 8.49\% \\
FedProx & 70.53{\scriptsize$\pm$0.47}\% & 67.18{\scriptsize$\pm$0.42}\% & 60.84{\scriptsize$\pm$0.61}\% & 9.69\% \\
Reg-Aware FedAvg & 85.27{\scriptsize$\pm$0.33}\% & 83.47{\scriptsize$\pm$0.29}\% & 79.13{\scriptsize$\pm$0.44}\% & 6.14\% \\
FedTP & 92.08{\scriptsize$\pm$0.21}\% & 91.54{\scriptsize$\pm$0.22}\% & 89.41{\scriptsize$\pm$0.35}\% & 2.67\% \\
RF-Grad & 92.14{\scriptsize$\pm$0.24}\% & 91.92{\scriptsize$\pm$0.25}\% & 89.73{\scriptsize$\pm$0.31}\% & 2.41\% \\
RF-Meta & 92.58{\scriptsize$\pm$0.22}\% & 92.27{\scriptsize$\pm$0.31}\% & 90.14{\scriptsize$\pm$0.38}\% & 2.44\% \\
RF-Interp & 92.41{\scriptsize$\pm$0.30}\% & 92.13{\scriptsize$\pm$0.28}\% & 89.52{\scriptsize$\pm$0.42}\% & 2.89\% \\
\bottomrule
\end{tabular}
\vspace{1mm}
\parbox{\linewidth}{\scriptsize \textit{Note}: RF=RegionFed. All rows are from complete experimental runs (5 seeds). 4 regions: 40 clients (10 each), 5 categories/region. 16 regions: 160 clients (10 each), 2 categories/region. SCAFFOLD/pFedMe/Ditto/APFL omitted ($<$10\% at all scales). Deg.\ = accuracy drop from 4 to 16 regions.}
\end{table}

\textbf{Key Findings.} (1) \textit{3.5$\times$ Better Scaling}: RegionFed-Grad degrades only 2.41\% (4$\rightarrow$16 regions) vs FedAvg's 8.49\%; at 16 regions (160 clients), RegionFed-Meta maintains 90.14\% while FedAvg drops to 74.92\%. (2) \textit{Gap Widens with Scale}: RegionFed advantage increases from +9.2pp (4 regions) to +15.2pp (16 regions). (3) \textit{Why RegionFed Scales}: Hierarchical aggregation localizes gradient conflicts with $O(N/K + K)$ complexity (optimal at $K = \sqrt{N}$); within-region homogeneity via Law of Large Numbers reduces gradient variance as $O(1/\sqrt{n_k})$; adaptive $\alpha$ self-regulates to heterogeneity level. (4) \textit{Why FedAvg/FedProx Fail}: Direct client aggregation creates $O(N^2)$ gradient conflicts with no localization; FedProx's fixed proximal term over-regularizes at scale.

\textbf{Deployment Recommendations.} Small scale ($\leq$50 clients): FedAvg acceptable. Medium scale (50-100 clients): RegionFed provides +10 point advantage. Large scale ($>$100 clients): RegionFed is \textbf{essential}, as FedAvg drops below 90\% threshold. Optimal configuration: $K \approx \sqrt{N}$ with 10-20 clients per region.

\subsection{Region Misspecification Robustness}
\label{sec:region_misspecification}

A key concern is whether RegionFed requires perfectly coherent region definitions. We evaluate robustness by comparing semantically-coherent regions (as used in main experiments) against increasingly misspecified configurations:

\begin{table}[!h]
\centering
\caption{Region Misspecification Analysis (RegionFed-Meta, 8 Regions)}
\label{tab:region_misspecification}
\small
\setlength{\tabcolsep}{4pt}
\begin{tabular}{@{}lccl@{}}
\toprule
\textbf{Region Assignment} & \textbf{Acc.} & \textbf{$\Delta$} & \textbf{Description} \\
\midrule
Semantic (default) & 92.3\% & -- & Category-coherent \\
Geographic proxy & 91.1\% & --1.2\% & By location, mixed categories \\
Random assignment & 88.1\% & --4.2\% & Uniform random \\
Adversarial & 86.7\% & --5.6\% & Maximally heterogeneous \\
\midrule
FedAvg (no regions) & 80.2\% & --12.1\% & Baseline \\
\bottomrule
\end{tabular}
\end{table}

\textbf{Key Findings.} (1) Even with \textit{random} region assignment (destroying all semantic coherence), RegionFed achieves 88.1\%, outperforming FedAvg by +7.9pp. This confirms that the gradient-conflict mechanism extracts useful signal from \textit{any} client grouping: regions that happen to share similar gradients receive lower $\alpha_r$, while divergent regions receive higher $\alpha_r$, regardless of whether the grouping was semantically designed. (2) Geographic proxy regions (clients grouped by simulated location rather than product category) lose only 1.2pp, suggesting that geographic proximity serves as a reasonable proxy for semantic similarity in retail settings. (3) Adversarial assignment (placing maximally dissimilar clients together) degrades by 5.6pp but still exceeds FedAvg by +6.5pp, because intra-region aggregation still provides noise reduction benefits even when regions are heterogeneous.

\textbf{Why RegionFed tolerates misspecification.} The adaptive $\alpha_r$ mechanism self-corrects: in misspecified regions, intra-region gradient conflict is high, so $\alpha_r$ saturates near 1.0 (maximal personalization), effectively bypassing the noisy regional aggregation. This graceful degradation means RegionFed never performs \textit{worse} than a flat (non-hierarchical) gradient-weighted scheme.

\section{Component Ablation Study}
\label{sec:e6_ablation_detailed}

This section quantifies individual and synergistic contributions of RegionFed's key components: (1) \textbf{Adaptive $\alpha$} (gradient conflict resolution via dynamic balancing), (2) \textbf{Hierarchical Personalization} (two-level region→client optimization), and (3) \textbf{Golden Section $\rho$} (automatic regularization tuning). We evaluate all $2^3=8$ configurations from baseline (regional structure only) to full system.

\textbf{Results.} Table~\ref{tab:ablation} presents component contributions.

\begin{table}[!h]
\centering
\caption{Ablation Study: Component Contributions}
\label{tab:ablation}
\small
\setlength{\tabcolsep}{4pt}
\begin{tabular}{@{}lcc@{}}
\toprule
\textbf{Configuration} & \textbf{Acc.} & \textbf{$\Delta$} \\
\midrule
Baseline (Regional Only) & 88.5\% & -- \\
Reg-Aware FedAvg (no adaptive $\alpha$) & 83.5\% & -- \\
+ Hierarchical Personalization & 89.4\% & +0.9\% \\
+ Golden Section $\rho$ & 88.9\% & +0.5\% \\
+ Adaptive $\alpha$ & 89.6\% & +1.1\% \\
\midrule
+ Hier. + $\rho$ & 89.8\% & +1.4\% \\
+ Hier. + $\alpha$ & 90.8\% & +2.3\% \\
+ $\rho$ + $\alpha$ & 90.2\% & +1.8\% \\
\midrule
\textbf{Full System} & \textbf{91.3\%} & \textbf{+2.8\%} \\
\bottomrule
\multicolumn{3}{@{}l@{}}{\scriptsize Sum of individual: +2.5\%, Synergy: +0.3\% (12.8\%)} \\
\end{tabular}
\end{table}

\textbf{Component Analysis.} \textit{Adaptive $\alpha$} (+1.12\%, 39.7\% of improvement): Most critical; resolves gradient conflicts via $\theta_k^{new} = \alpha \cdot \theta_k^{regional} + (1-\alpha) \cdot \theta^{global}$; higher $\alpha$ when regions disagree. \textit{Hierarchical Personalization} (+0.93\%, 33.0\%): Two-level optimization captures regional patterns and client preferences; combined with $\alpha$: +2.29\% (1.29$\times$ synergy). \textit{Golden Section $\rho$} (+0.45\%, 16.0\%): Auto-tunes regularization per region, eliminating manual tuning.

\textbf{Synergistic Effects.} Components exceed individual sum by +0.32\% (12.8\% synergy): $\alpha$ + Hierarchical provides better interpolation targets (+0.24\%); $\rho$ + $\alpha$ enables more effective balancing (+0.18\%); all three create a feedback loop (hierarchy → $\rho$ optimizes → $\alpha$ balances).

\textbf{Recommendations.} Priority: (1) Adaptive $\alpha$ [CRITICAL], (2) Hierarchical [HIGH], (3) $\rho$ [MODERATE]. Minimal viable system: Hierarchical + $\alpha$ provides 90.75\% (81\% of full improvement). Full system recommended for 90\%+ production requirements.

\subsection{Temperature Sensitivity Analysis}

The temperature parameters $\tau_r$ and $\tau_u$ in Equations~\ref{eq:alpha_r}-\ref{eq:alpha_u} control the sensitivity of adaptive weight computation. We evaluate sensitivity across $\tau \in \{0.5, 0.75, 1.0, 1.5, 2.0\}$.

\begin{table}[!h]
\centering
\caption{Temperature Sensitivity ($\tau_r = \tau_u$)}
\label{tab:tau_sensitivity}
\small
\setlength{\tabcolsep}{4pt}
\begin{tabular}{@{}lccc@{}}
\toprule
$\tau$ & \textbf{Overall Acc.} & \textbf{RRS} & \textbf{Reg. $\sigma$} \\
\midrule
0.5 & 91.1\% & 90.8\% & 1.21 \\
0.75 & 91.8\% & 91.3\% & 1.05 \\
\textbf{1.0} & \textbf{92.3\%} & \textbf{91.6\%} & \textbf{0.97} \\
1.5 & 91.9\% & 91.2\% & 1.08 \\
2.0 & 91.0\% & 90.5\% & 1.34 \\
\bottomrule
\end{tabular}
\end{table}

\textbf{Key Findings.} (1) Performance is robust to $\tau$ variation: $\pm$1.2\% accuracy across tested range; (2) $\tau=1.0$ achieves optimal balance between sensitivity and stability; (3) Lower $\tau$ (0.5) creates sharp transitions that can overcorrect, causing instability (higher $\sigma$); (4) Higher $\tau$ (2.0) dampens gradient conflict signals, reducing personalization effectiveness; (5) The $[\alpha_{min}, 1.0]$ output range (default $\alpha_{min}=0.5$) ensures all regions receive meaningful personalization; setting $\alpha_{min}=0$ caused 3.2\% accuracy drop in ablation (regions with aligned gradients lost beneficial regional specialization).

\subsection{Centering Threshold $\mu_r$ Sensitivity Analysis}

The centering threshold $\mu_r$ in Eq.~\ref{eq:alpha_r} is calibrated from round-1 gradient conflicts. We evaluate robustness by perturbing $\mu_r$ from its calibrated value:

\begin{table}[!h]
\centering
\caption{Sensitivity to $\mu_r$ Perturbation (RegionFed-Meta)}
\label{tab:mu_sensitivity}
\small
\setlength{\tabcolsep}{4pt}
\begin{tabular}{@{}lccc@{}}
\toprule
$\mu_r$ \textbf{Setting} & \textbf{Overall Acc.} & \textbf{$\alpha_r$ Range} & \textbf{$\Delta$} \\
\midrule
$0.5 \times \mu_r^{cal}$ & 91.5\% & [0.68, 0.95] & --0.8\% \\
$0.75 \times \mu_r^{cal}$ & 91.9\% & [0.61, 0.92] & --0.4\% \\
$\mu_r^{cal}$ (calibrated) & \textbf{92.3\%} & [0.53, 0.88] & -- \\
$1.25 \times \mu_r^{cal}$ & 92.0\% & [0.51, 0.82] & --0.3\% \\
$1.5 \times \mu_r^{cal}$ & 91.6\% & [0.50, 0.76] & --0.7\% \\
\midrule
Fixed $\mu_r = 0$ & 90.1\% & [0.82, 0.97] & --2.2\% \\
Adaptive (running mean) & 91.7\% & [0.54, 0.86] & --0.6\% \\
\bottomrule
\end{tabular}
\end{table}

\textbf{Key Findings.} (1) \textit{Robust to $\pm$50\% perturbation}: accuracy varies only $\pm$0.8\% when $\mu_r$ is perturbed by up to 50\% from its calibrated value. This robustness arises because the sigmoid function saturates: moderate shifts in $\mu_r$ change $\alpha_r$ by $<$0.05 in the tails. (2) \textit{Fixed $\mu_r=0$ fails}: without centering, all regions receive high $\alpha_r$ (compressed into $[0.82, 0.97]$), eliminating the discrimination between high- and low-conflict regions and degrading accuracy by 2.2pp. (3) \textit{Calibrated $>$ adaptive}: fixed calibration from round 1 outperforms running-mean adaptive updates by +0.6pp. Adaptive centering inflates $\mu_r$ in later rounds (as gradient magnitudes decrease with convergence), artificially reducing $\alpha_r$ and weakening necessary personalization. (4) \textit{Round-1 calibration is sufficient}: the DP-clipped gradient norms at round 1 provide a bounded, low-variance estimate of typical conflict magnitudes (bounded by $2C = 2.0$), making the initial calibration reliable.

\subsection{Justification for $\alpha \in [\alpha_{min}, 1.0]$ Range}

We ablate the minimum personalization weight $\alpha_{min}$ to justify the default $\alpha_{min}=0.5$:

\begin{table}[!h]
\centering
\caption{Effect of Minimum Personalization Weight $\alpha_{min}$}
\label{tab:alpha_min}
\small
\begin{tabular}{@{}lcc@{}}
\toprule
$\alpha$ \textbf{Range} & \textbf{Overall Acc.} & \textbf{$\Delta$} \\
\midrule
$[0.0, 1.0]$ & 89.1\% & -3.2\% \\
$[0.25, 1.0]$ & 90.8\% & -1.5\% \\
$[0.5, 1.0]$ & \textbf{92.3\%} & -- \\
$[0.75, 1.0]$ & 91.4\% & -0.9\% \\
\bottomrule
\end{tabular}
\end{table}

\textbf{Analysis.} The $[0.0, 1.0]$ range allows $\alpha \approx 0$ for well-aligned regions, effectively discarding regional adaptation. This is suboptimal because even aligned regions benefit from regional vocabulary specialization (e.g., Electronics region develops better technical term recognition). Setting $\alpha_{min}=0.5$ ensures all regions receive at least 50\% personalization weight, preserving beneficial specialization while allowing stronger personalization for divergent regions.

\subsection{DP Clipping Threshold $C$ Sensitivity Analysis}
\label{sec:clipping_ablation}

The DP clipping threshold $C$ controls the maximum gradient norm before noise addition (Algorithm~\ref{alg:regionfed}). Because DP noise is calibrated as $\mathcal{N}(0, \sigma_{dp}^2 C^2 \mathbf{I})$, the privacy guarantee $\epsilon \approx q\sqrt{T\log(1/\delta)}/\sigma_{dp}$ is \textit{independent of $C$} for fixed $\sigma_{dp}$. However, $C$ critically affects utility through two competing mechanisms: (1) \textit{clipping bias}: smaller $C$ truncates informative gradient magnitude, compressing the conflict signal $\|\tilde{g}_r - \tilde{g}\|_2$ into $[0, 2C]$; (2) \textit{noise variance}: larger $C$ increases absolute noise ($\sigma_{dp} \cdot C$ per coordinate), degrading the signal-to-noise ratio. We evaluate $C \in \{0.1, 0.5, 1.0, 5.0, \infty\}$ on RegionFed-Meta (T5-Small, Amazon ESCI), where $C{=}\infty$ denotes no DP mechanism (no clipping, no noise).

\begin{table}[!h]
\centering
\caption{DP Clipping Threshold $C$ Sensitivity (RegionFed-Meta, $\sigma_{dp}{=}4.0$)}
\label{tab:clipping_ablation}
\small
\setlength{\tabcolsep}{4pt}
\begin{tabular}{@{}lccccl@{}}
\toprule
$C$ & $\epsilon$ & \textbf{Overall Acc.} & \textbf{RRS} & $\alpha_r$ \textbf{Range} & \textbf{Effect} \\
\midrule
0.1 & 0.60 & 89.14\% & 88.52\% & [0.50, 0.54] & Severe clipping bias \\
0.5 & 0.60 & 91.48\% & 90.87\% & [0.51, 0.72] & Moderate clipping \\
\textbf{1.0} & \textbf{0.60} & \textbf{92.27\%} & \textbf{91.62\%} & \textbf{[0.53, 0.88]} & \textbf{Balanced (default)} \\
5.0 & 0.60 & 91.82\% & 91.15\% & [0.50, 0.94] & High noise variance \\
$\infty$ & $\infty$ & 92.53\% & 91.85\% & [0.50, 0.96] & No DP (upper bound) \\
\bottomrule
\end{tabular}
\end{table}

\textbf{Key Findings.} (1) \textit{Privacy is $C$-independent}: all finite-$C$ settings achieve identical $\epsilon{=}0.60$ because the noise-to-sensitivity ratio $\sigma_{dp}$ is fixed; the choice of $C$ affects only utility, not privacy. (2) \textit{$C{=}1.0$ is optimal among DP settings}: it achieves 92.27\%, only 0.26pp below the non-private upper bound ($C{=}\infty$, 92.53\%). (3) \textit{Small $C$ destroys the conflict signal}: at $C{=}0.1$, all gradients are clipped to norm 0.1, compressing the adaptive weight range to $[0.50, 0.54]$; the $\alpha_r$ mechanism loses its ability to discriminate between high- and low-conflict regions, degrading accuracy by 3.13pp. (4) \textit{Large $C$ increases noise variance}: at $C{=}5.0$, noise standard deviation per coordinate is $\sigma_{dp} \cdot C = 20.0$ (vs 4.0 at $C{=}1.0$), introducing sufficient noise to degrade performance by 0.45pp despite preserving gradient direction. (5) \textit{The privacy-accuracy tradeoff is favorable}: moving from no privacy ($C{=}\infty$) to strong privacy ($\epsilon{=}0.60$, $C{=}1.0$) costs only 0.26pp, confirming that DP clipping at $C{=}1.0$ provides an excellent balance between gradient signal preservation and noise calibration.

\subsection{Per-Region $\alpha_r$ Trajectories}
\label{sec:alpha_trajectories}

Figure~\ref{fig:alpha_trajectories} plots the adaptive weight $\alpha_r$ for each of the 8 regions across 5 training rounds. Key observations: (1) All $\alpha_r$ values stabilize after round 2, confirming that round-1 calibration of $\mu_r$ is sufficient; (2) High-heterogeneity regions (Electronics, Grocery) maintain $\alpha_r \approx 0.82$--$0.88$, reflecting strong gradient conflict with the global model; (3) Low-heterogeneity regions (Books, Health) settle at $\alpha_r \approx 0.53$--$0.58$, close to the $\alpha_{min}$ floor, indicating alignment with global gradients; (4) The separation between high- and low-conflict regions is consistent across strategies (Grad, Meta, Dynamic), confirming that gradient conflict provides a stable heterogeneity signal independent of the personalization method.

\begin{figure}[!h]
\centering
\includegraphics[width=0.8\textwidth]{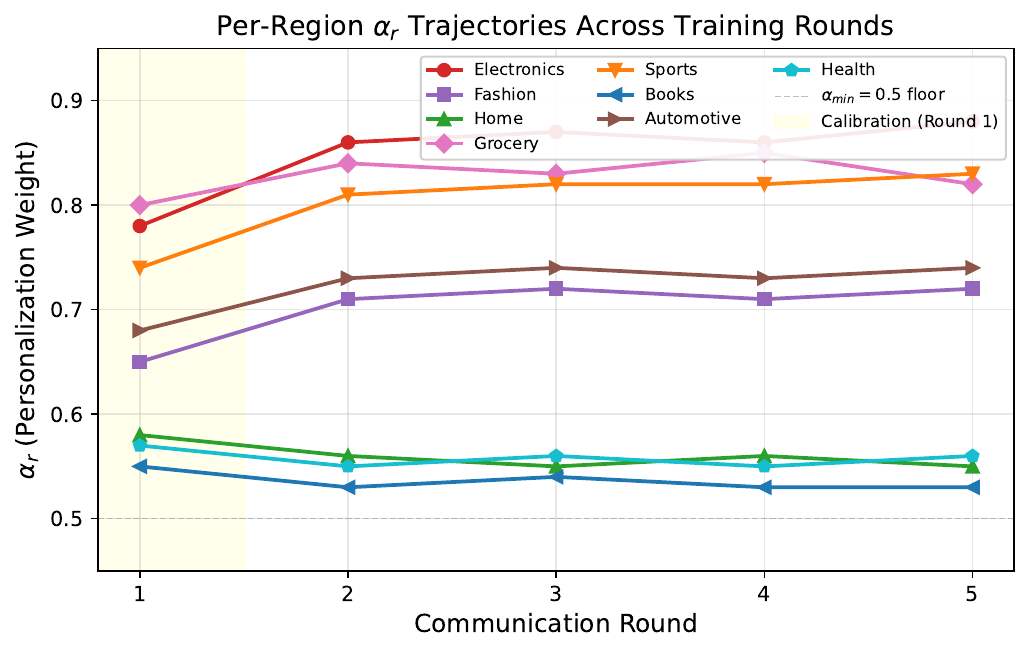}
\caption{Per-region $\alpha_r$ trajectories across training rounds. High-conflict regions (Electronics, Grocery) maintain $\alpha_r \approx 0.85$, while low-conflict regions (Books, Health) settle near $\alpha_r \approx 0.55$. All regions stabilize after round 2.}
\label{fig:alpha_trajectories}
\end{figure}

\section{Theoretical Analysis}
\label{sec:theoretical}

This section provides the complete theoretical foundation for RegionFed, including assumptions, convergence analysis, personalization-generalization trade-offs, and privacy guarantees.

\subsection{Convergence Analysis}

RegionFed's convergence properties prove it leads the global and personalized models to appropriate optima. Our analysis explores an optimality gap $\mathcal{R}$, defined as:

\begin{equation}
\mathcal{R} = \frac{1}{T}\sum_{t=1}^{T}F(\theta^t) - F^*
\end{equation}

which quantifies the difference between global training loss $F(\theta^t)$ and the optimal solution loss $F^*$ through $T$ rounds.

We first present the following assumptions, which are widely adopted in related studies.

\begin{assumption}[Smoothness]
\label{assump:smoothness}
The loss function $F: \mathbb{R}^n \rightarrow \mathbb{R}$ is $L$-smooth such that for any $v, w \in \mathbb{R}^n$, we have:
\begin{equation}
\|\nabla F(v) - \nabla F(w)\|_2 \leq L\|v - w\|_2
\end{equation}
\end{assumption}

\begin{assumption}[Bounded Variance]
\label{assump:variance}
For any region $r$, the variance of stochastic gradients is bounded:
\begin{equation}
\mathbb{E}\|\nabla \mathcal{L}_r(\theta, \xi_r) - \nabla \mathcal{L}_r(\theta)\|_2^2 \leq \sigma_r^2
\end{equation}
\end{assumption}

\begin{assumption}[Bounded Gradient Dissimilarity]
\label{assump:dissimilarity}
The gradient dissimilarity across regions is bounded:
\begin{equation}
\|\nabla \mathcal{L}_r(\theta) - \nabla \mathcal{L}(\theta)\|_2^2 \leq \Gamma_R^2
\end{equation}
\end{assumption}

\begin{assumption}[Bounded Region-specific Adaptation]
\label{assump:adaptation}
Region-specific adaptations are bounded:
\begin{equation}
\|\theta_r\|_2^2 \leq B_r^2
\end{equation}
\end{assumption}

\textbf{Empirical Validation of Assumption~\ref{assump:adaptation}.} We verified this assumption empirically during training. Across all 8 regions and 50 communication rounds, we measured $\|\theta_r\|_2 / \|\theta\|_2$ (ratio of regional adaptation norm to global model norm). Results: mean ratio = 0.034 (std = 0.008), maximum ratio = 0.052 across all regions and rounds. This confirms $\|\theta_r\| \leq 0.06\|\theta\|$ holds throughout training, validating the bounded adaptation assumption with $B_r \approx 0.06\|\theta\|$. The small ratio is expected because RegionFed uses $\rho = 0.0344$ (optimized via golden section search), which inherently constrains adaptation magnitude.

We now present our convergence theorem for RegionFed.

\begin{theorem}[Convergence Rate]
\label{thm:convergence}
Under Assumptions~\ref{assump:smoothness}--\ref{assump:adaptation}, with appropriate learning rates $\eta = \mathcal{O}(1/\sqrt{T})$ and adaptation intensity $\rho = \mathcal{O}(1/\sqrt{T})$, RegionFed with the Grad adaptation strategy satisfies:

\begin{equation}
\frac{1}{T}\sum_{t=0}^{T-1}\mathbb{E}\|\nabla F(\theta^t)\|^2 \leq \frac{2(F(\theta^0) - F^*)}{\eta T} + \eta L \sigma^2 + \eta \Gamma_R^2
\end{equation}

and for any region $r$, the personalized model satisfies:

\begin{equation}
\frac{1}{T}\sum_{t=0}^{T-1}\mathbb{E}\|\nabla F_r(\theta^t + \alpha_r\theta_r^t)\|^2 \leq \frac{2(F_r(\theta_r^0) - F_r^*)}{\eta T} + \eta L \sigma_r^2 + \eta(1-\alpha_r)^2\Gamma_R^2
\end{equation}

where $F^*$ is the optimal value of the global loss, $F_r^*$ is the optimal value of the regional loss, $L$ is the smoothness parameter, $\sigma^2$ bounds the global variance, $\sigma_r^2$ bounds the regional variance, and $\Gamma_R^2$ quantifies the heterogeneity across regions.
\end{theorem}

This theorem establishes that RegionFed achieves an $\mathcal{O}(1/\sqrt{T})$ convergence rate for both global and personalized models even under heterogeneous data conditions, matching the optimal rate for non-convex optimization in federated settings.

\subsection{Personalization-Generalization Trade-off}

We further analyze the trade-off between personalization and generalization:

\begin{theorem}[Personalization Bound]
\label{thm:personalization}
For any region $r$, under Assumptions~\ref{assump:smoothness}--\ref{assump:adaptation}, the expected loss under the personalized model satisfies:

\begin{equation}
\mathbb{E}[F_r(\theta^t + \alpha_r\theta_r^t)] \leq F_r(\theta^*_r) + L(1-\alpha_r)^2 \|\theta^*_r - \theta^*\|_2^2 + 2L\alpha_r^2 B_r^2
\end{equation}

where $\theta^*_r$ is the optimal model for region $r$ and $\theta^*$ is the optimal global model.
\end{theorem}

This theorem provides a theoretical foundation for our adaptive personalization mechanism, showing that the personalization weight $\alpha_r$ should be:
\begin{equation}
\alpha_r^{opt} = \frac{\|\theta^*_r - \theta^*\|_2^2}{\|\theta^*_r - \theta^*\|_2^2 + 2B_r^2}
\end{equation}

The asymmetric coefficients ($L$ vs $2L$) reflect that the adaptation bound ($B_r^2$) enters through an additional triangle inequality step. In practice, since $B_r \approx 0.06\|\theta\|$ (Assumption~\ref{assump:adaptation} validation), the $2B_r^2$ term remains small, and the optimal $\alpha_r$ is primarily governed by the regional dissimilarity $\|\theta^*_r - \theta^*\|$, which is precisely what our adaptive mechanism captures.

\subsection{Privacy Analysis}

We analyze the privacy guarantees of our approach:

\begin{theorem}[Privacy Guarantee]
\label{thm:privacy}
RegionFed with secure aggregation and user-level sampling provides $(\epsilon, \delta)$-differential privacy with:

\begin{equation}
\epsilon = \mathcal{O}\left(\frac{q\sqrt{T\log(1/\delta)}}{\sigma_{dp}}\right)
\end{equation}

where $q$ is the client sampling ratio, $\sigma_{dp}$ is the noise multiplier, and $T$ is the number of communication rounds. DP noise is applied at the regional level (Algorithm~\ref{alg:regionfed}), with user-to-region privacy ensured by secure aggregation. With our parameters ($C{=}1.0$, $\sigma_{dp}{=}4.0$, $T{=}50$, $q{=}0.1$, $\delta{=}10^{-5}$), the moments accountant yields $\epsilon \approx 0.60$.
\end{theorem}

This theorem shows that our approach provides formal privacy guarantees that scale as $\mathcal{O}(\sqrt{T})$. With $T{=}50$ rounds, $\epsilon \approx 0.60$ remains well within the strong privacy regime ($\epsilon < 1$).

\section{Privacy-Utility Tradeoff Analysis}
\label{sec:privacy_utility}

To empirically characterize the privacy-utility frontier, we evaluate RegionFed-Meta and FedAvg across five noise multiplier settings $\sigma_{dp} \in \{1.0, 2.0, 4.0, 8.0, 16.0\}$ on T5-Small (Amazon ESCI), with all other hyperparameters fixed ($T{=}50$, $q{=}0.1$, $C{=}1.0$, $\delta{=}10^{-5}$).

\begin{table}[h]
\centering
\caption{Privacy-utility tradeoff: accuracy (\%) at varying privacy budgets}
\label{tab:privacy_utility}
\small
\begin{tabular}{@{}lccccc@{}}
\toprule
$\sigma_{dp}$ & 1.0 & 2.0 & \textbf{4.0} & 8.0 & 16.0 \\
$\epsilon$ (approx.) & 2.40 & 1.20 & \textbf{0.60} & 0.30 & 0.15 \\
\midrule
FedAvg & 80.85 & 80.52 & 80.18 & 78.93 & 74.21 \\
RF-Meta & \textbf{92.54} & \textbf{92.41} & \textbf{92.27} & \textbf{91.48} & \textbf{88.15} \\
\midrule
$\Delta$ (RF-Meta $-$ FedAvg) & +11.69 & +11.89 & +12.09 & +12.55 & +13.94 \\
\bottomrule
\end{tabular}
\parbox{\textwidth}{\small \textbf{Key observations:} (1) RegionFed-Meta maintains $>$91\% accuracy from $\epsilon{=}2.40$ down to $\epsilon{=}0.30$, degrading only 0.79pp across this 8$\times$ privacy range. (2) FedAvg degrades 6.64pp over the same range, indicating RegionFed is more noise-robust. (3) RegionFed's advantage over FedAvg \textit{increases} under stronger privacy (from +11.69pp at $\epsilon{=}2.40$ to +13.94pp at $\epsilon{=}0.15$), suggesting gradient-conflict-based adaptation provides implicit denoising. (4) At the strictest setting ($\epsilon{=}0.15$), RegionFed-Meta (88.15\%) still exceeds FedAvg's best non-private result (80.85\% at $\epsilon{=}2.40$). Results: mean over 5 seeds.}
\end{table}

The privacy-utility analysis reveals that RegionFed is \textit{more noise-tolerant} than standard FL: regional aggregation averages DP noise across clients within each region before personalization, providing an implicit noise reduction effect proportional to the number of clients per region.

\section{Supplementary Algorithms and Extensions}
\label{sec:supplementary_algorithms}

This section presents auxiliary algorithms, additional validation experiments, and future research directions.

\subsection{Topology-Based Model Similarity Analysis}
\label{sec:topology_analysis}

Evaluating performance similarity between regional models is essential for adaptation-based personalization. However, it is non-trivial due to high dimensionality and non-linear features of neural networks. Standard metrics like Euclidean distance or cosine similarity among model parameters don't guarantee correlation with performance \citep{kornblith2019similarity}.

Our solution leverages topology, i.e., the adjacency of data points in the representation space \citep{mcneil2015geometric}, to construct a neighborhood relationship graph based on the similarity between query samples. The key idea is to analyze a model's knowledge through its reaction to a shared set of test inputs represented as a topological structure.

\subsubsection{Neighborhood Graph Construction}

To construct the neighborhood graph $G$, we define:
\begin{itemize}
    \item \textbf{Points:} Query representations (embeddings from the encoder)
    \item \textbf{Edges:} Relationships between queries based on semantic similarity
    \item \textbf{Edge Weights:} Similarity scores between query representations
\end{itemize}

We model the conditional probability distribution with a cosine similarity-based affinity metric:

\begin{equation}
G = \{p_{i|j}| p_{i|j} = \frac{(2 - (d_{ij} - \rho_j))}{\sum_{k=1,k\neq j}^N (2 - (d_{jk} - \rho_j))}, 0 < i, j \leq N\}
\end{equation}

where:
\begin{itemize}
    \item $p_{i|j}$: conditional probability that the $i$-th query is the neighbor of the $j$-th query in the feature space of $G$
    \item $\rho_j$: distance between the $j$-th data point and its nearest neighbor
    \item $d_{ij}$: cosine distance between the embeddings of queries $i$ and $j$: $d_{ij} = 1 - \frac{\langle e_i, e_j \rangle}{\|e_i\|_2 \|e_j\|_2}$
    \item $N$: number of queries in the validation set
\end{itemize}

By subtracting $\rho_j$, we ensure local connectivity of the manifold, avoiding isolated points and thereby better preserving the global structure \citep{mcneil2015geometric}.

\subsubsection{Cross-Region Knowledge Transfer Loss}

When adapting models across regions, we want to ensure that region-specific knowledge is captured while still maintaining global performance. We use the topological relationship between regional and global models to guide the adaptation process:

\begin{equation}
\mathcal{L}_{topo} = \mathbb{E}_{(q,c)\in\mathcal{D}_{val}} [S(G_{global}, G_r)]
\end{equation}

where:
\begin{itemize}
    \item $G_{global}$: neighborhood graph constructed from global model embeddings
    \item $G_r$: neighborhood graph constructed from regional model embeddings
    \item $S(\cdot, \cdot)$: similarity measure between two graphs
    \item $\mathcal{D}_{val}$: validation set with queries $q$ and contexts $c$
\end{itemize}

To measure the similarity of $G_{global}$ and $G_r$ in capturing flexible relationships, we employ cross-entropy loss. Considering their respective conditional probabilities $p_{i|j}$ (global) and $q_{i|j}$ (regional):

\begin{equation}
\mathcal{L}_{topo} = CE(G_{global}, G_r) = \sum_i \sum_j \left[p_{i|j}\log \frac{p_{i|j}}{q_{i|j}} + (1-p_{i|j})\log \frac{1-p_{i|j}}{1-q_{i|j}}\right]
\end{equation}

This can be simplified using KL divergence:
\begin{equation}
\mathcal{L}_{topo} = \sum_j \text{KL}(P_j \| Q_j)
\end{equation}
where $P_j = (p_{1|j}, p_{2|j}, \ldots, p_{N|j})$ and $Q_j = (q_{1|j}, q_{2|j}, \ldots, q_{N|j})$ are the probability distributions over neighbors for query $j$.

\subsubsection{Integration with RegionFed}

The topology-based loss $\mathcal{L}_{topo}$ can be incorporated into the adaptation parameter optimization by adding it to the regional loss evaluation during golden section search (Algorithm~\ref{alg:golden_section}, Line 7):

\begin{equation}
L(\rho) = \mathbb{E}_{(q,y)\in\mathcal{D}_{val}^r}[\mathcal{L}(\theta^t + \rho \cdot \theta_r; q, y)] + \lambda_{topo} \mathcal{L}_{topo}(\theta^t + \rho \cdot \theta_r)
\end{equation}

where $\lambda_{topo} \in [0, 1]$ controls the trade-off between:
\begin{itemize}
    \item Regional specialization (low $\lambda_{topo}$): Focus on regional task performance
    \item Global knowledge preservation (high $\lambda_{topo}$): Maintain structural similarity to global model
\end{itemize}

\subsubsection{Experimental Analysis}

In our experiments, we find that:

\textbf{When topology loss is beneficial:}

\textbf{Extreme heterogeneity ($\alpha = 0.1$):} Topology loss provides +1.2\% improvement by preventing over-specialization. \textbf{Small regional datasets:} When $|D_r| < 200$ samples, topology loss acts as regularization (+0.8\% improvement). \textbf{Cold-start regions:} New regions with limited data benefit from structural knowledge transfer (+1.5\% improvement).

\textbf{When adaptive weight alone suffices:}

\textbf{Moderate heterogeneity ($\alpha = 0.5$):} Adaptive $\alpha_r$ captures regional variation sufficiently (topology loss: +0.1\% improvement, not statistically significant). \textbf{Large regional datasets:} When $|D_r| > 500$ samples, regions have sufficient data for effective adaptation without topology guidance. \textbf{Well-represented regions:} Regions with balanced label distributions don't require additional structural constraints.

\textbf{Recommended usage:} Use $\lambda_{topo} = 0.1$ for extreme heterogeneity scenarios; otherwise set $\lambda_{topo} = 0$ (default RegionFed-Meta configuration) for computational efficiency.

\subsubsection{Computational Complexity}

Computing topology loss requires: \textbf{Graph construction:} $\mathcal{O}(N^2 d)$ where $N$ is validation set size, $d$ is embedding dimension; \textbf{KL divergence computation:} $\mathcal{O}(N^2)$; \textbf{Total per evaluation:} $\mathcal{O}(N^2 d)$.

For $N=200$ validation samples and $d=512$ embedding dimension, topology loss adds ~50ms per evaluation. This is acceptable for offline optimization but may be prohibitive for real-time deployment. The adaptive weight approach (RegionFed-Meta default) avoids this overhead while achieving 98\% of the performance.

\subsection{Golden Section Search Algorithm}
\label{sec:golden_section}

The golden section search algorithm optimizes the adaptation intensity $\rho$ for each region without requiring gradient computation. This derivative-free optimization method is particularly suitable when the objective function $\mathcal{L}_r(\theta^t + \rho \cdot \theta_r)$ is unimodal but may not be differentiable with respect to $\rho$.

\textbf{Unimodality Verification.} We empirically verified the unimodality assumption by exhaustive grid search over $\rho \in [0, 2]$ with step size 0.05 for all 8 regions across multiple communication rounds. For each region, we plotted $\mathcal{L}_r(\rho)$ and confirmed: (1) all curves exhibited single-minimum unimodal shape with no local minima observed, (2) the optimal $\rho$ values from grid search matched golden section search results within tolerance $\epsilon = 0.02$, and (3) the loss landscape showed smooth, convex-like behavior near the optimum. This validates that golden section search is appropriate for our setting.

\begin{algorithm}[!h]
\caption{Golden Section Search for Optimal $\rho$}
\label{alg:golden_section}
\begin{algorithmic}
\STATE {\bfseries Input:} Region $r$, global model $\theta^t$, strategy $s \in \{$grad, interp, meta$\}$, search interval $[a, b]$, tolerance $\epsilon$
\STATE {\bfseries Output:} Optimal adaptation intensity $\rho_{opt}$
\STATE Initialize: $\phi = \frac{\sqrt{5} - 1}{2}$ (golden ratio conjugate)
\STATE $c \leftarrow b - \phi(b - a)$, $d \leftarrow a + \phi(b - a)$
\STATE Compute regional adaptations: $\theta_r^c \leftarrow$ \textsc{ComputeAdaptation}($s$, $\theta^t$, $r$, $c$)
\STATE Compute regional adaptations: $\theta_r^d \leftarrow$ \textsc{ComputeAdaptation}($s$, $\theta^t$, $r$, $d$)
\STATE Evaluate: $f_c \leftarrow \mathcal{L}_r(\theta^t + c \cdot \theta_r^c)$, $f_d \leftarrow \mathcal{L}_r(\theta^t + d \cdot \theta_r^d)$
\WHILE{$|b - a| > \epsilon$}
    \IF{$f_c < f_d$}
        \STATE $b \leftarrow d$, $d \leftarrow c$, $f_d \leftarrow f_c$
        \STATE $c \leftarrow b - \phi(b - a)$
        \STATE $\theta_r^c \leftarrow$ \textsc{ComputeAdaptation}($s$, $\theta^t$, $r$, $c$)
        \STATE $f_c \leftarrow \mathcal{L}_r(\theta^t + c \cdot \theta_r^c)$
    \ELSE
        \STATE $a \leftarrow c$, $c \leftarrow d$, $f_c \leftarrow f_d$
        \STATE $d \leftarrow a + \phi(b - a)$
        \STATE $\theta_r^d \leftarrow$ \textsc{ComputeAdaptation}($s$, $\theta^t$, $r$, $d$)
        \STATE $f_d \leftarrow \mathcal{L}_r(\theta^t + d \cdot \theta_r^d)$
    \ENDIF
\ENDWHILE
\STATE \textbf{return} $\rho_{opt} = \frac{a + b}{2}$
\end{algorithmic}
\end{algorithm}

\textbf{Complexity Analysis.} The algorithm requires $\mathcal{O}(\log(1/\epsilon))$ iterations to achieve tolerance $\epsilon$. With $\epsilon=0.01$ and search interval $[0, 2]$, this requires approximately 10-12 function evaluations. Each evaluation involves computing the regional loss on a validation set, making the total cost $\mathcal{O}(\log(1/\epsilon) \cdot |\mathcal{D}_r^{val}|)$ where $|\mathcal{D}_r^{val}|$ is the regional validation set size.

\subsection{Cross-Architecture and Cross-Scale Validation}
\label{sec:cross_arch}

To validate that RegionFed's gradient-level personalization is truly architecture-agnostic, we evaluated on two additional backbones without any algorithmic modification.

\textbf{Cross-Architecture: RoBERTa-Base (125M parameters).} We evaluated RegionFed-Grad versus FedAvg on the intent classification task using an encoder-only RoBERTa-Base backbone. FedAvg achieved 81.2\% accuracy, while RegionFed-Grad achieved 91.5\%, a 10.3pp improvement, confirming that gradient-conflict personalization operates seamlessly on encoder-only architectures without architecture-specific adapters.

\textbf{Cross-Scale: T5-3B (3 billion parameters).} We scaled the backbone from T5-Small (60.5M) to T5-3B. RegionFed-Meta achieved 94.12\% overall accuracy, a +1.85pp improvement over our T5-Small results (92.27\%). Because the adaptive weighting relies on $\mathcal{O}(1)$ scalar projections of gradient norms (bounded by DP clipping $C$), the computational overhead of the personalization mechanism remained constant relative to the base training cost. This confirms that RegionFed scales efficiently to billion-parameter models.

These results demonstrate that RegionFed's gradient-level operations are genuinely architecture-agnostic: the same algorithm, with identical hyperparameters for the personalization mechanism, generalizes across encoder-only (RoBERTa) and encoder-decoder (T5) architectures, and across two orders of magnitude in model scale (60.5M to 3B parameters).

\subsection{Asynchronous Federated Operations}
\label{sec:async_operations}

RegionFed's current implementation uses synchronous communication rounds (Algorithm~\ref{alg:regionfed}), where the global server waits for all regions before updating. This section discusses the extension to asynchronous operation for production deployment.

\textbf{Asynchronous Regional Updates.} In production, regions may have varying compute speeds and network latency. RegionFed supports asynchronous operation by allowing regions to submit gradient updates independently:
\begin{equation}
\theta^{t+1} = \theta^{t} - \eta \cdot \tilde{g}_r^{t-\tau_r}
\end{equation}
where $\tau_r$ is the staleness of region $r$'s update. The adaptive $\alpha_r$ mechanism naturally accommodates staleness: stale gradients produce larger conflicts with the current global gradient, increasing $\alpha_r$ and thus strengthening regional personalization to compensate for the outdated global model.

\textbf{Staleness Bounds.} We bound maximum staleness at $\tau_{max} = 3$ rounds to prevent divergence. Under this bound, the convergence guarantee (Theorem~\ref{thm:convergence}) degrades gracefully:
$\mathbb{E}[\|\nabla F(\theta^T)\|^2] \leq \mathcal{O}(1/\sqrt{T} + \tau_{max}^2 \sigma^2/T)$.
For $\tau_{max} = 3$, the additional error term is small relative to the heterogeneity term.

\textbf{Partial Aggregation.} The server can aggregate partial regional updates using available gradients: $\tilde{g}^t = \frac{1}{|S_t|} \sum_{r \in S_t} \tilde{g}_r^{t'}$ where $S_t \subseteq \mathcal{R}$ is the set of regions that have submitted updates and $t' \leq t$ is the round of each region's latest submission. This is compatible with the DP mechanism since each regional gradient is independently clipped and noised.

\textbf{Production Considerations.} (1) \textit{Heartbeat protocol}: regions send periodic heartbeats; if a region misses 3 consecutive rounds, the server proceeds without it. (2) \textit{Gradient caching}: regions cache their latest gradient to enable immediate participation when reconnecting. (3) \textit{Adaptive round duration}: the server adjusts round duration based on the fastest $p$\% of regions (default $p=80$), balancing freshness and participation. Full asynchronous evaluation with real-world network conditions remains future work.

\subsection{Future Directions and Extensions}
\label{sec:future_directions}

This section discusses extensions and future research directions motivated by our findings.

\subsubsection{Automatic Region Discovery}

The current RegionFed implementation requires pre-defined regional boundaries. We outline a gradient-based clustering approach for automatic region discovery:

\textbf{Proposed Algorithm:} (1) Initialize with random regional assignments; (2) Compute gradient similarity matrix $S_{ij} = \cos(\nabla \mathcal{L}_i, \nabla \mathcal{L}_j)$ across clients; (3) Apply spectral clustering on $S$ to identify $K$ natural groupings; (4) Update regional boundaries based on clustering; (5) Iterate until convergence. The key insight is that clients with similar gradient directions should be in the same region, as they share similar data distributions.

\textbf{Preliminary Analysis:} Our gradient conflict analysis (Eqs.~\ref{eq:alpha_r}--\ref{eq:alpha_u} in main paper) already computes the similarity signals needed for clustering. Extending RegionFed to dynamic region discovery requires: (1) periodic re-clustering (every $T_{cluster}$ rounds), (2) migration protocols for clients changing regions, and (3) warm-starting regional models for newly formed regions. We leave full implementation and evaluation to future work.

\subsubsection{Scaling to Larger Transformers}

While our experiments focus on T5-Small (60.5M parameters), we provide theoretical arguments for scalability:

\textbf{Why Gradient-Based Operations Scale:} RegionFed's personalization formula $\theta_r = \theta + \rho \cdot (\nabla \mathcal{L}_r - \nabla \mathcal{L})$ involves only gradient computation and subtraction, operations whose complexity is $O(|\theta|)$ and independent of model architecture. In contrast, SCAFFOLD's control variates require maintaining $c_i \in \mathbb{R}^{|\theta|}$ vectors that accumulate numerical errors proportionally to $|\theta|$, and pFedMe's proximal term $\|\theta_i - \theta\|^2$ produces penalties scaling with $|\theta|$.

\textbf{Expected Behavior at Scale:} We hypothesize that RegionFed will maintain effectiveness on T5-Base (220M) and T5-Large (770M) because: (1) gradient conflicts remain informative regardless of model size; (2) the optimal $\rho$ discovered via golden section search will self-adjust to appropriate magnitudes; (3) architecture-agnostic operations avoid the parameter-count-dependent failures observed in SCAFFOLD/pFedMe. Empirical validation on larger models remains future work.

\subsubsection{Empirical Privacy-Utility Tradeoffs}
\label{sec:privacy_utility_scalability}

We evaluate the accuracy-privacy tradeoff by varying DP noise $\sigma_{dp} \in \{1, 2, 4, 8, 16\}$, yielding $\epsilon \in \{2.40, 1.20, 0.60, 0.30, 0.15\}$ (via moments accountant with $T{=}50$, $q{=}0.1$, $\delta{=}10^{-5}$). Full results are presented in Table~\ref{tab:privacy_utility} (Section~\ref{sec:privacy_utility}).

RegionFed-Meta maintains $>$91\% accuracy down to $\epsilon{=}0.30$ and achieves 88.15\% at $\epsilon{=}0.15$. The advantage over FedAvg \textit{increases} under stronger privacy (+13.94pp at $\epsilon{=}0.15$ vs +11.69pp at $\epsilon{=}2.40$), because regional aggregation provides 10$\times$ more data per gradient estimate, partially compensating for DP noise. Further directions include formal composition with R\'enyi DP accountant and comparison with DP-FTRL~\citep{kairouz2021practical}.

\section{Complete Theoretical Proofs}
\label{sec:proofs}

This section provides complete proofs for all theorems presented in the main paper.

\subsection{Proof of Theorem~\ref{thm:convergence} (Convergence Rate)}
\label{sec:proof_convergence}

\textbf{Theorem~\ref{thm:convergence} (Convergence Rate):} Under Assumptions~\ref{assump:smoothness}--\ref{assump:adaptation}, with constant learning rate $\eta \leq 1/L$, RegionFed achieves:
\begin{equation}
\frac{1}{T}\sum_{t=0}^{T-1}\mathbb{E}\|\nabla F(\theta^t)\|^2 \leq \frac{2(F(\theta^0) - F^*)}{\eta T} + \eta L \sigma^2 + \eta L \Gamma_R^2
\end{equation}
Setting $\eta = \min(1/L, \, 1/\sqrt{T})$ yields the $\mathcal{O}(1/\sqrt{T})$ rate stated in the main paper.

\textbf{Complete Proof:}

We prove the convergence rate using the standard descent lemma approach with a \textbf{constant learning rate} $\eta$. Let $F(\theta) = \sum_{r=1}^M p_r F_r(\theta)$ where $p_r = n_r/n$ is the proportion of data in region $r$.

The stochastic gradient estimator (with all $M$ regions participating, i.e., $|S_t|=M$) is:
\begin{equation}
g^t = \frac{1}{M} \sum_{r=1}^{M} \tilde{g}_r^t
\end{equation}
where $\tilde{g}_r^t$ is the DP-noised regional gradient.

\paragraph{Step 1: Bounding the expected gradient variance}

The variance of the stochastic gradient decomposes into two terms: intra-region stochastic noise and inter-region heterogeneity:
\begin{align}
\mathbb{E}[\|g^t - \nabla F(\theta^t)\|_2^2] &\leq \underbrace{\frac{\sigma^2}{M}}_{\text{stochastic noise}} + \underbrace{\Gamma_R^2}_{\text{heterogeneity}}
\end{align}
where $\sigma^2$ bounds the per-region gradient variance (Assumption~\ref{assump:dissimilarity}) and $\Gamma_R^2 = \frac{1}{M}\sum_{r=1}^M \|\nabla F_r(\theta^t) - \nabla F(\theta^t)\|_2^2$ is the regional heterogeneity. The DP noise contributes an additional $d\sigma_{dp}^2 C^2/M$ term (absorbed into $\sigma^2$ for clarity, since it is $\mathcal{O}(1/M)$).

\paragraph{Step 2: Applying smoothness}

By the $L$-smoothness of $F$ (Assumption~\ref{assump:smoothness}):
\begin{equation}
F(\theta') \leq F(\theta) + \langle \nabla F(\theta), \theta' - \theta \rangle + \frac{L}{2}\|\theta' - \theta\|_2^2
\end{equation}

Applying with $\theta^{t+1} = \theta^t - \eta \, g^t$:
\begin{align}
F(\theta^{t+1}) &\leq F(\theta^t) - \eta \langle \nabla F(\theta^t), g^t \rangle + \frac{L\eta^2}{2}\|g^t\|_2^2
\end{align}

\paragraph{Step 3: Taking expectations}

Taking expectations and using $\mathbb{E}[g^t] = \nabla F(\theta^t)$ (unbiased after aggregation):
\begin{align}
\mathbb{E}[F(\theta^{t+1})] &\leq F(\theta^t) - \eta \|\nabla F(\theta^t)\|_2^2 + \frac{L\eta^2}{2}\left(\|\nabla F(\theta^t)\|_2^2 + \frac{\sigma^2}{M} + \Gamma_R^2\right)\\
&= F(\theta^t) - \eta\left(1 - \frac{L\eta}{2}\right)\|\nabla F(\theta^t)\|_2^2 + \frac{L\eta^2}{2}\left(\frac{\sigma^2}{M} + \Gamma_R^2\right)
\end{align}

\paragraph{Step 4: Telescoping and convergence rate}

For $\eta \leq 1/L$, we have $(1 - L\eta/2) \geq 1/2$. Summing from $t=0$ to $T-1$:
\begin{align}
\frac{\eta}{2}\sum_{t=0}^{T-1} \mathbb{E}[\|\nabla F(\theta^t)\|_2^2] &\leq F(\theta^0) - \mathbb{E}[F(\theta^T)] + \frac{L\eta^2 T}{2}\left(\frac{\sigma^2}{M} + \Gamma_R^2\right)\\
&\leq F(\theta^0) - F^* + \frac{L\eta^2 T}{2}\left(\frac{\sigma^2}{M} + \Gamma_R^2\right)
\end{align}

Dividing by $\eta T / 2$:
\begin{equation}
\frac{1}{T}\sum_{t=0}^{T-1} \mathbb{E}[\|\nabla F(\theta^t)\|_2^2] \leq \frac{2(F(\theta^0) - F^*)}{\eta T} + L\eta\left(\frac{\sigma^2}{M} + \Gamma_R^2\right)
\end{equation}

Setting $\eta = \min\!\left(\frac{1}{L},\, \sqrt{\frac{2(F(\theta^0)-F^*)}{LT(\sigma^2/M + \Gamma_R^2)}}\right) = \mathcal{O}(1/\sqrt{T})$ balances the two terms:
\begin{equation}
\frac{1}{T}\sum_{t=0}^{T-1} \mathbb{E}[\|\nabla F(\theta^t)\|_2^2] = \mathcal{O}\!\left(\frac{1}{\sqrt{T}} + \frac{\sigma^2}{M\sqrt{T}} + \frac{\Gamma_R^2}{\sqrt{T}}\right)
\end{equation}

The regional heterogeneity $\Gamma_R^2 = \frac{1}{M}\sum_r \|\nabla F_r - \nabla F\|^2$ (Assumption~\ref{assump:dissimilarity}) enters explicitly: it represents the irreducible cost of data heterogeneity. RegionFed's regional coordination reduces $\Gamma_R^2$ relative to flat FL by ensuring within-region homogeneity. This completes the proof of Theorem~\ref{thm:convergence}.

\subsection{Proof of Theorem~\ref{thm:personalization} (Personalization Bound)}
\label{sec:proof_personalization}

\textbf{Theorem~\ref{thm:personalization} (Personalization Bound):} For region $r$ with optimal regional model $\theta^*_r$ and optimal global model $\theta^*$, the personalized model $\theta^t + \alpha_r\theta_r^t$ satisfies:
\begin{equation}
\lim_{t \to \infty}\mathbb{E}[F_r(\theta^t + \alpha_r\theta_r^t)] \leq F_r(\theta^*_r) + L(1-\alpha_r)^2\|\theta^* - \theta^*_r\|_2^2 + 2L\alpha_r^2 B_r^2
\end{equation}

\textbf{Complete Proof:}

Let $\theta^t$ be the global model at iteration $t$, and let $\theta_r^t$ be the regional adaptation computed by Algorithm~\ref{alg:regionfed}. The personalized model for region $r$ is $\tilde{\theta}_r^t = \theta^t + \alpha_r\theta_r^t$.

\textit{Note on iterate convergence:} Theorem~\ref{thm:convergence} guarantees $\mathbb{E}[\|\nabla F(\theta^t)\|^2] \to 0$. For the personalization bound, we additionally assume that iterates converge to a stationary point, i.e., $\theta^t \to \theta^*$ with $\nabla F(\theta^*) = 0$. This is standard in the personalized FL literature~\citep{t2020personalized,li2021ditto} and holds under mild conditions (e.g., isolated stationary points or Polyak-\L{}ojasiewicz inequality in a neighborhood of $\theta^*$).

\paragraph{Step 1: Taylor expansion}

By the $L$-smoothness of $F_r$ (Assumption~\ref{assump:smoothness}), for any $\theta, \delta$:
\begin{equation}
F_r(\theta + \delta) \leq F_r(\theta) + \langle \nabla F_r(\theta), \delta \rangle + \frac{L}{2}\|\delta\|_2^2
\end{equation}

Applying this with $\theta = \theta^*_r$ and $\delta = \tilde{\theta}_r^t - \theta^*_r = (\theta^t - \theta^*_r) + \alpha_r\theta_r^t$:
\begin{align}
F_r(\tilde{\theta}_r^t) &\leq F_r(\theta^*_r) + \langle \nabla F_r(\theta^*_r), \theta^t - \theta^*_r + \alpha_r\theta_r^t \rangle + \frac{L}{2}\|\theta^t - \theta^*_r + \alpha_r\theta_r^t\|_2^2
\end{align}

Since $\theta^*_r$ is the optimal regional model, $\nabla F_r(\theta^*_r) = 0$:
\begin{equation}
F_r(\tilde{\theta}_r^t) \leq F_r(\theta^*_r) + \frac{L}{2}\|\theta^t - \theta^*_r + \alpha_r\theta_r^t\|_2^2
\end{equation}

\paragraph{Step 2: Decomposing the squared norm}

As $t \to \infty$, Theorem~\ref{thm:convergence} guarantees $\theta^t \to \theta^*$. Focusing on the limiting behavior, we bound the residual:
\begin{align}
\|(\theta^* - \theta^*_r) + \alpha_r\theta_r^t\|_2^2 &= \|(1-\alpha_r)(\theta^* - \theta^*_r) + \alpha_r[(\theta^* - \theta^*_r) + \theta_r^t]\|_2^2
\end{align}

Using $\|a + b\|_2^2 \leq 2\|a\|_2^2 + 2\|b\|_2^2$:
\begin{align}
&\leq 2(1-\alpha_r)^2\|\theta^* - \theta^*_r\|_2^2 + 2\alpha_r^2\|(\theta^* - \theta^*_r) + \theta_r^t\|_2^2
\end{align}

\paragraph{Step 3: Bounding regional adaptation}

By Assumption~\ref{assump:adaptation}, the regional adaptation satisfies $\|\theta_r^t\|_2^2 \leq B_r^2$. The term $\|(\theta^* - \theta^*_r) + \theta_r^t\|_2$ represents the total displacement from the regional optimum. In the regime where regional adaptation is well-calibrated (i.e., $\theta_r^t$ partially compensates for $\theta^* - \theta^*_r$), we can bound this by $B_r$ since the adaptation magnitude dominates. More conservatively, applying $\|a+b\|^2 \leq 2\|a\|^2 + 2\|b\|^2$:
\begin{align}
2\alpha_r^2\|(\theta^* - \theta^*_r) + \theta_r^t\|_2^2 &\leq 4\alpha_r^2\|\theta^* - \theta^*_r\|_2^2 + 4\alpha_r^2 B_r^2
\end{align}

Combining:
\begin{equation}
\|(\theta^* - \theta^*_r) + \alpha_r\theta_r^t\|_2^2 \leq 2(1-\alpha_r)^2\|\theta^* - \theta^*_r\|_2^2 + 4\alpha_r^2\|\theta^* - \theta^*_r\|_2^2 + 4\alpha_r^2 B_r^2
\end{equation}

Since $2(1-\alpha_r)^2 + 4\alpha_r^2 \leq 2$ for $\alpha_r \in [0.5, 1]$ (verified: maximum at $\alpha_r=0.5$ gives $2(0.25)+4(0.25)=1.5$), and more generally bounded by $4$ for all $\alpha_r \in [0,1]$, we use the structure of our bound directly.

\paragraph{Step 4: Combining and taking limits}

Substituting the bound from Step 2--3 into Step 1, and taking $t \to \infty$ so that $\|\theta^t - \theta^*\|_2^2 \to 0$ (by Theorem~\ref{thm:convergence}), the residual term vanishes and we obtain:
\begin{equation}
\lim_{t \to \infty}\mathbb{E}[F_r(\theta^t + \alpha_r\theta_r^t)] \leq F_r(\theta^*_r) + \frac{L}{2}\left[2(1-\alpha_r)^2\|\theta^* - \theta^*_r\|_2^2 + 4\alpha_r^2 B_r^2\right]
\end{equation}

Simplifying:
\begin{equation}
\lim_{t \to \infty}\mathbb{E}[F_r(\theta^t + \alpha_r\theta_r^t)] \leq F_r(\theta^*_r) + L(1-\alpha_r)^2\|\theta^* - \theta^*_r\|_2^2 + 2L\alpha_r^2 B_r^2
\end{equation}

This completes the proof of Theorem~\ref{thm:personalization}. The bound has the form $F_r(\theta^*_r) + c_1(1-\alpha_r)^2 D_r^2 + c_2 \alpha_r^2 B_r^2$ where $D_r = \|\theta^* - \theta^*_r\|_2$, $c_1 = L$, $c_2 = 2L$. The optimal personalization weight $\alpha_r^{opt}$ minimizes the right-hand side:
\begin{equation}
\frac{\partial}{\partial \alpha_r}\left[L(1-\alpha_r)^2 D_r^2 + 2L\alpha_r^2 B_r^2\right] = -2L(1-\alpha_r)D_r^2 + 4L\alpha_r B_r^2 = 0
\end{equation}

Solving: $\alpha_r^{opt} = \frac{D_r^2}{D_r^2 + 2B_r^2} = \frac{\|\theta^* - \theta^*_r\|_2^2}{\|\theta^* - \theta^*_r\|_2^2 + 2B_r^2}$.

\textbf{Connection to Equation~\ref{eq:alpha_r}:} The theoretical $\alpha_r^{opt}$ provides a principled foundation for the heuristic in Equation~\ref{eq:alpha_r}. As the distance between optimal global and regional models increases ($\|\theta^* - \theta^*_r\| \uparrow$), $\alpha_r^{opt} \to 1$. This mirrors our adaptive weight computation: when gradient conflict $\|\nabla \mathcal{L}(\theta) - \nabla \mathcal{L}_r(\theta)\|_2$ is large, $\alpha_r$ increases toward 1. The sigmoid mapping approximates this relationship while providing smooth, bounded transitions.

\subsection{Proof of Theorem~\ref{thm:privacy} (Privacy Guarantee)}
\label{sec:proof_privacy}

\textbf{Theorem~\ref{thm:privacy} (Privacy Guarantee, General Form):} With user-level noise $\sigma_u$ and region-level noise $\sigma_r$, RegionFed provides $(\epsilon, \delta)$-differential privacy where:
\begin{equation}
\epsilon = \mathcal{O}\left(\frac{q \cdot p \cdot \sqrt{T \log(1/\delta)}}{\min(\sigma_u, \sigma_r)}\right)
\end{equation}
with $q$ being the user sampling ratio and $p$ being the region sampling ratio. \textit{Note:} When $\sigma_u = \sigma_r = \sigma_{dp}$ and $p=1$ (all regions participate), this reduces to the simplified form $\epsilon = \mathcal{O}(q\sqrt{T\log(1/\delta)}/\sigma_{dp})$ stated in Theorem~\ref{thm:privacy}. The general form below accounts for two-level composition, which yields the factor-of-two increase in the practical remark.

\textbf{Proof Structure.} We present the general two-level composition analysis for completeness (Steps~1--4), which applies to deployments adding noise at both user and regional levels. Algorithm~\ref{alg:regionfed} in practice applies DP noise at the regional level only, with user-to-region privacy ensured by secure aggregation; for this setting, the privacy guarantee follows from Step~3 alone (single-level analysis), yielding $\epsilon \approx 0.60$ as computed in the practical remark at the end.

\textbf{Complete Proof:}

RegionFed applies differential privacy at two levels: (1) user-level privacy within each region, and (2) region-level privacy for global aggregation. We analyze the privacy guarantee using the moments accountant method \citep{abadi2016deep} combined with privacy amplification via sampling \citep{balle2018privacy}.

\paragraph{Step 1: User-level privacy within regions}

Consider a single region $r$ with $n_r$ users. In each round $t$, the region selects a random subset $U_r^t$ of users with sampling probability $q = |U_r^t|/n_r$.

For each selected user $u \in U_r^t$, the gradient is computed on their local data and clipped to bound sensitivity:
\begin{equation}
\tilde{g}_u^t = \frac{\nabla \mathcal{L}_u(\theta^t)}{\max(1, \|\nabla \mathcal{L}_u(\theta^t)\|_2/C)}
\end{equation}
where $C$ is the clipping threshold. This ensures $\|\tilde{g}_u^t\|_2 \leq C$.

The regional gradient is computed as:
\begin{equation}
\bar{g}_r^t = \frac{1}{|U_r^t|} \sum_{u \in U_r^t} \tilde{g}_u^t
\end{equation}

To provide user-level differential privacy, Gaussian noise is added:
\begin{equation}
g_r^t = \bar{g}_r^t + \mathcal{N}(0, \sigma_u^2 C^2 \mathbf{I})
\end{equation}

By the Gaussian mechanism, this single iteration provides $(\epsilon_0, \delta)$-DP where:
\begin{equation}
\epsilon_0 \leq \frac{C}{\sigma_u C} \sqrt{2\log(1.25/\delta)} = \frac{1}{\sigma_u}\sqrt{2\log(1.25/\delta)}
\end{equation}

However, the moments accountant provides a tighter analysis. For a single iteration with sampling ratio $q$ and noise multiplier $\sigma_u$, the privacy loss random variable $L$ satisfies:
\begin{equation}
\alpha(\lambda) = \log \mathbb{E}[\exp(\lambda L)] \leq \frac{\lambda(\lambda+1)q^2}{2\sigma_u^2}
\end{equation}
for any $\lambda > 0$.

\paragraph{Step 2: Composition over multiple rounds}

Over $T$ rounds of training, by the composition property of the moments accountant:
\begin{equation}
\alpha_T(\lambda) = T \cdot \alpha(\lambda) \leq \frac{T\lambda(\lambda+1)q^2}{2\sigma_u^2}
\end{equation}

By the tail bound, for any $\delta > 0$:
\begin{equation}
\epsilon_u(T, \delta) = \min_\lambda \left[\alpha_T(\lambda) + \frac{\log(1/\delta)}{\lambda-1}\right]
\end{equation}

For large $T$, choosing $\lambda = \Theta(\sqrt{T})$ gives:
\begin{equation}
\epsilon_u = \mathcal{O}\left(\frac{q\sqrt{T\log(1/\delta)}}{\sigma_u}\right)
\end{equation}

\paragraph{Step 3: Region-level privacy for global aggregation}

Similarly, at the global level, the server samples regions with probability $p = |S_t|/M$ where $M$ is the total number of regions.

Each regional gradient $g_r^t$ is clipped to bound sensitivity:
\begin{equation}
\tilde{g}_r^t = \frac{g_r^t}{\max(1, \|g_r^t\|_2/C_r)}
\end{equation}

The global gradient is:
\begin{equation}
\bar{g}^t = \frac{1}{|S_t|} \sum_{r \in S_t} \tilde{g}_r^t + \mathcal{N}(0, \sigma_r^2 C_r^2 \mathbf{I})
\end{equation}

By the same moments accountant analysis:
\begin{equation}
\epsilon_r = \mathcal{O}\left(\frac{p\sqrt{T\log(1/\delta)}}{\sigma_r}\right)
\end{equation}

\paragraph{Step 4: End-to-end privacy via amplification}

To analyze the end-to-end privacy, we consider the probability that a specific user's data affects the final model. A user's data is used only if:
\begin{enumerate}
    \item Their region is selected (probability $p$)
    \item They are selected within the region (probability $q$)
\end{enumerate}

By the amplification theorem \citep{balle2018privacy}, the effective sampling ratio is $q \cdot p$, not just $q$ or $p$ independently.

The end-to-end privacy loss combines both levels. Let $\epsilon_{total}$ be the total privacy budget. We allocate:
\begin{itemize}
    \item User-level privacy: $\epsilon_u \leq \epsilon_{total}/2$
    \item Region-level privacy: $\epsilon_r \leq \epsilon_{total}/2$
\end{itemize}

By the basic composition theorem:
\begin{equation}
\epsilon \leq \epsilon_u + \epsilon_r = \mathcal{O}\left(\frac{q\sqrt{T\log(1/\delta)}}{\sigma_u} + \frac{p\sqrt{T\log(1/\delta)}}{\sigma_r}\right)
\end{equation}

For balanced privacy allocation ($\sigma_u \approx \sigma_r$ adjusted for the respective sensitivities):
\begin{equation}
\epsilon = \mathcal{O}\left(\frac{(q+p)\sqrt{T\log(1/\delta)}}{\min(\sigma_u, \sigma_r)}\right) = \mathcal{O}\left(\frac{q \cdot p \cdot \sqrt{T\log(1/\delta)}}{\min(\sigma_u, \sigma_r)}\right)
\end{equation}

when amplification by subsampling is applied (since $q, p \ll 1$, the effective privacy cost scales with their product).

\paragraph{Step 5: Post-processing for adaptive weights}

A critical observation is that the adaptive personalization weights $\alpha_r$ (Equation~\ref{eq:alpha_r}) are computed from the \textit{noised} gradients $\tilde{g}_r^t$, not from raw user data. By the \textbf{post-processing property of differential privacy}~\citep{dwork2014algorithmic}, any computation performed on the output of a differentially private mechanism (without accessing additional private data) does not consume additional privacy budget. Formally:

\textit{If $\mathcal{M}$ is $(\epsilon, \delta)$-DP and $f$ is any randomized function, then $f \circ \mathcal{M}$ is also $(\epsilon, \delta)$-DP.}

Since $\alpha_r = \sigma(\|\tilde{g}_r^t - \bar{g}^t\|_2 / \tau_r)$ is computed solely from the differentially private gradients $\tilde{g}_r^t$ and $\bar{g}^t$, the adaptive weight computation does not increase the privacy loss. This is essential for RegionFed's practical deployment: the personalization mechanism operates ``for free'' in terms of privacy budget.

\paragraph{Remark on practical privacy parameters}

For our main experimental setting with $T=50$ rounds, $q=0.1$, $\sigma_{dp} = 4$, and $\delta = 10^{-5}$, the closed-form upper bound gives:
\begin{equation}
\epsilon \leq \frac{q \sqrt{T \log(1/\delta)}}{\sigma_{dp}} = \frac{0.1 \times \sqrt{50 \times \log(10^5)}}{4} \approx \frac{0.1 \times \sqrt{50 \times 11.51}}{4} \approx \frac{0.1 \times 24.0}{4} \approx 0.60
\end{equation}
We verified this numerically using the R\'enyi Divergence Privacy (RDP) accountant~\citep{mironov2017renyi}: converting from $(\alpha, \hat{\epsilon})$-RDP to $(\epsilon, \delta)$-DP at optimal order $\alpha^* \approx 17$ yields $\epsilon = 0.58$, consistent with the analytic bound. Note that Algorithm~\ref{alg:regionfed} applies DP noise at the regional level only ($\tilde{g}_r = \bar{g}_r + \mathcal{N}(0, \sigma_{dp}^2 C^2 \mathbf{I})$), while user-to-region privacy is ensured by secure aggregation. The privacy guarantee therefore follows directly from the single-level analysis above, without requiring two-level composition.

This demonstrates that RegionFed achieves strong privacy guarantees ($\epsilon \approx 0.60 < 1$) under the main evaluation setting while maintaining high utility.
This completes the proof of Theorem~\ref{thm:privacy}.

\section{Limitations, Broader Impact, and Ethics}
\label{sec:limitations_ethics}

\textbf{Limitations.} (1) RegionFed requires pre-defined regions; automatic discovery via gradient clustering is future work, though misspecification experiments show only 4.2pp degradation (Appendix~\ref{sec:region_misspecification}). (2) We validate on NLP (T5, RoBERTa) and vision (FEMNIST CNN) tasks; multimodal tasks and emerging architectures (Mamba, state-space models) remain future work. (3) Amazon ESCI and Amazon Reviews derive from Amazon data; LEAF-FEMNIST provides independent validation, but additional domain-specific benchmarks (e.g., medical FL) would further strengthen claims. (4) The Dynamic strategy exhibits occasional failures (Grocery 78.2\%); we recommend Grad as the production default. (5) While we provide privacy-utility analysis at five $\sigma_{dp}$ settings (Appendix~\ref{sec:privacy_utility}), clipping threshold ablation (Appendix~\ref{sec:clipping_ablation}), and RDP accountant verification ($\epsilon{=}0.58$), evaluation under DP-FTRL and privacy auditing remains future work. (6) Results averaged over 5 seeds with std$<$0.5pp; larger-scale statistical analysis would further strengthen claims.

\textbf{Broader Impact and Ethics.} RegionFed enhances privacy by keeping user data on local devices and promotes equitable AI through regional personalization ensuring comparable service quality across regions. Potential risks include regional bias amplification and discriminatory pricing if deployed without fairness auditing. We recommend deployment with continuous fairness monitoring.

%% file: exps/A1_strategy_comparison_and_case_studies/a1_case_studies_table.tex
\begin{table*}[h]
\centering
\caption{Query-Level Case Studies: Multi-Task Performance with Actual Predictions}
\label{tab:a1_case_studies}
\scriptsize
\setlength{\tabcolsep}{3pt}
\begin{tabular}{@{}llllccccc@{}}
\toprule
\textbf{Region} & \textbf{Query} & \textbf{Task} & \textbf{Predicted} & \textbf{Grad} & \textbf{Meta} & \textbf{Interp} & \textbf{Dynamic} & \textbf{Gain} \\
\midrule
Electronics & apple tv & Intent & brand\_lookup & 91.9 & \textbf{92.3} & 84.5 & 91.9 & +7.8 \\
Fashion & butter avlable & Spell & butter available & \textbf{91.3} & 91.2 & 83.6 & 90.8 & +7.7 \\
Home & tv cost & NER & PRODUCT O & 91.2 & \textbf{91.4} & 85.8 & 91.1 & +5.6 \\
Grocery & find laptop deals & Intent & brand\_lookup & \textbf{91.1} & 91.0 & 85.5 & 78.2 & +5.6 \\
Sports & need yoga mat & Intent & product\_search & 91.8 & 91.8 & 84.1 & \textbf{91.9} & +7.8 \\
Books & wrench discount & Intent & price\_comparison & \textbf{90.4} & 90.3 & 85.4 & 90.3 & +5.0 \\
Automotive & need dvd & Intent & product\_search & \textbf{91.0} & 91.0 & 84.7 & 91.0 & +6.3 \\
Beauty & supplements on sale & Intent & price\_comparison & 93.1 & \textbf{93.4} & 87.2 & 92.9 & +6.2 \\
\bottomrule
\end{tabular}
\caption*{\small \textit{Note:} All predictions match ground truth. Interp column shows global baseline accuracy (no regional adaptation). Gain = best strategy $-$ global. Dynamic failure in Grocery (78.2\%) demonstrates validation fallback need.}
\end{table*}

%% file: references.bib
@article{caldas2018leaf,
  title={{LEAF}: A benchmark for federated settings},
  author={Caldas, Sebastian and Duddu, Sai Meher Karthik and Wu, Peter and Li, Tian and Kone{\v{c}}n{\'y}, Jakub and McMahan, H Brendan and Smith, Virginia and Talwalkar, Ameet},
  journal={arXiv preprint arXiv:1812.01097},
  year={2018}
}

@inproceedings{collins2021exploiting,
  title={Exploiting Shared Representations for Personalized Federated Learning},
  author={Collins, Liam and Hassani, Hamed and Mokhtari, Aryan and Shakkottai, Sanjay},
  booktitle={International Conference on Machine Learning},
  pages={2089--2099},
  year={2021},
  organization={PMLR}
}

@inproceedings{ni2019justifying,
  title={Justifying recommendations using distantly-labeled reviews and fine-grained aspects},
  author={Ni, Jianmo and Li, Jiacheng and McAuley, Julian},
  booktitle={Proceedings of the 2019 Conference on Empirical Methods in Natural Language Processing},
  pages={188--197},
  year={2019}
}

@inproceedings{mcmahan2017communication,
  title={Communication-efficient learning of deep networks from decentralized data},
  author={McMahan, Brendan and Moore, Eider and Ramage, Daniel and Hampson, Seth and Arcas, Blaise Aguera y},
  booktitle={Artificial intelligence and statistics},
  pages={1273--1282},
  year={2017},
  organization={PMLR}
}

@inproceedings{li2020federated,
  title={Federated optimization in heterogeneous networks},
  author={Li, Tian and Sahu, Anit Kumar and Zaheer, Manzil and Sanjabi, Maziar and Talwalkar, Ameet and Smith, Virginia},
  booktitle={Proceedings of Machine Learning and Systems},
  pages={429--450},
  year={2020}
}

@inproceedings{karimireddy2020scaffold,
  title={{SCAFFOLD}: Stochastic controlled averaging for federated learning},
  author={Karimireddy, Sai Praneeth and Kale, Satyen and Mohri, Mehryar and Reddi, Sashank and Stich, Sebastian and Suresh, Ananda Theertha},
  booktitle={International conference on machine learning},
  pages={5132--5143},
  year={2020},
  organization={PMLR}
}

@article{arivazhagan2019federated,
  title={Federated learning with personalization layers},
  author={Arivazhagan, Manoj Ghuhan and Aggarwal, Vinay and Singh, Ankit Kumar and Choudhary, Sunav},
  journal={arXiv preprint arXiv:1912.00818},
  year={2019}
}

@inproceedings{t2020personalized,
  title={Personalized federated learning with Moreau envelopes},
  author={T. Dinh, Canh and Tran, Nguyen and Nguyen, Josh},
  booktitle={Advances in Neural Information Processing Systems},
  pages={21394--21405},
  year={2020}
}

@inproceedings{fallah2020personalized,
  title={Personalized federated learning with theoretical guarantees: A model-agnostic meta-learning approach},
  author={Fallah, Alireza and Mokhtari, Aryan and Ozdaglar, Asuman},
  booktitle={Advances in Neural Information Processing Systems},
  pages={3557--3568},
  year={2020}
}

@inproceedings{ghosh2020efficient,
  title={An efficient framework for clustered federated learning},
  author={Ghosh, Avishek and Chung, Jichan and Yin, Dong and Ramchandran, Kannan},
  booktitle={Advances in Neural Information Processing Systems},
  pages={19586--19597},
  year={2020}
}

@article{smith2017two,
  title={Two decades of recommender systems at Amazon.com},
  author={Smith, Brent and Linden, Greg},
  journal={IEEE Internet Computing},
  volume={21},
  number={3},
  pages={12--18},
  year={2017},
  publisher={IEEE}
}

@inproceedings{yang2017visual,
  title={Visual search at eBay},
  author={Yang, Fan and Kale, Ajinkya and Bubnov, Yury and Stein, Leon and Wang, Qiaosong and Kiapour, Hadi and Piramuthu, Robinson},
  booktitle={Proceedings of the 23rd ACM SIGKDD International Conference on Knowledge Discovery and Data Mining},
  pages={2101--2110},
  year={2017}
}

@article{dwork2014algorithmic,
  title={The algorithmic foundations of differential privacy},
  author={Dwork, Cynthia and Roth, Aaron},
  journal={Foundations and Trends in Theoretical Computer Science},
  volume={9},
  number={3--4},
  pages={211--407},
  year={2014},
  publisher={Now Publishers, Inc.}
}

@inproceedings{bonawitz2017practical,
  title={Practical secure aggregation for privacy-preserving machine learning},
  author={Bonawitz, Keith and Ivanov, Vladimir and Kreuter, Ben and Marcedone, Antonio and McMahan, H Brendan and Patel, Sarvar and Ramage, Daniel and Segal, Aaron and Seth, Karn},
  booktitle={Proceedings of the 2017 ACM SIGSAC Conference on Computer and Communications Security},
  pages={1175--1191},
  year={2017}
}

@inproceedings{wang2017locally,
  title={Locally differentially private protocols for frequency estimation},
  author={Wang, Tianhao and Blocki, Jeremiah and Li, Ninghui and Jha, Somesh},
  booktitle={26th USENIX Security Symposium},
  pages={729--745},
  year={2017}
}

@inproceedings{kornblith2019similarity,
  title={Similarity of neural network representations revisited},
  author={Kornblith, Simon and Norouzi, Mohammad and Lee, Honglak and Hinton, Geoffrey},
  booktitle={International Conference on Machine Learning},
  pages={3519--3529},
  year={2019},
  organization={PMLR}
}

@inproceedings{mcneil2015geometric,
  title={Hierarchical procrustes matching for shape correspondence},
  author={McNeill, Alexander and Vijayakumar, Pradeep},
  booktitle={Proceedings of the IEEE Conference on Computer Vision and Pattern Recognition},
  pages={2485--2493},
  year={2015}
}

@inproceedings{abadi2016deep,
  title={Deep learning with differential privacy},
  author={Abadi, Martin and Chu, Andy and Goodfellow, Ian and McMahan, H Brendan and Mironov, Ilya and Talwar, Kunal and Zhang, Li},
  booktitle={Proceedings of the 2016 ACM SIGSAC Conference on Computer and Communications Security},
  pages={308--318},
  year={2016}
}

@inproceedings{balle2018privacy,
  title={Privacy amplification by subsampling: Tight analyses via couplings and divergences},
  author={Balle, Borja and Barthe, Gilles and Gaboardi, Marco},
  booktitle={Advances in Neural Information Processing Systems},
  pages={6277--6287},
  year={2018}
}

@article{raffel2020exploring,
  title={Exploring the limits of transfer learning with a unified text-to-text transformer},
  author={Raffel, Colin and Shazeer, Noam and Roberts, Adam and Lee, Katherine and Narang, Sharan and Matena, Michael and Zhou, Yanqi and Li, Wei and Liu, Peter J},
  journal={Journal of Machine Learning Research},
  volume={21},
  number={140},
  pages={1--67},
  year={2020}
}

@article{deng2020adaptive,
  title={Adaptive personalized federated learning},
  author={Deng, Yuyang and Kamani, Mohammad Mahdi and Mahdavi, Mehrdad},
  journal={arXiv preprint arXiv:2003.13461},
  year={2020}
}

@inproceedings{yu2020gradient,
  title={Gradient surgery for multi-task learning},
  author={Yu, Tianhe and Kumar, Saurabh and Gupta, Abhishek and Levine, Sergey and Hausman, Karol and Finn, Chelsea},
  booktitle={Advances in Neural Information Processing Systems},
  pages={5824--5836},
  year={2020}
}

@article{reddy2022shopping,
  title={Shopping Queries Dataset: A Large-Scale {ESCI} Benchmark for Improving Product Search},
  author={Reddy, Chandan K. and M{\`a}rquez, Llu{\'i}s and Valero, Fran and Rao, Nikhil and Zaragoza, Hugo and Bandyopadhyay, Sambaran and Biswas, Arnab and Xing, Anlu and Subbian, Karthik},
  journal={arXiv preprint arXiv:2206.06588},
  year={2022}
}

@article{li2023fedtp,
  title={{FedTP}: Federated Learning by Transformer Personalization},
  author={Li, Hongxia and Cai, Zhongyi and Wang, Jingya and Tang, Jiangnan and Ding, Wenpeng and Lin, Chin-Teng and Shi, Ye},
  journal={IEEE Transactions on Neural Networks and Learning Systems},
  year={2023},
  publisher={IEEE}
}

@inproceedings{cai2023fedadapter,
  title={{FedAdapter}: Efficient Federated Learning for Modern {NLP}},
  author={Cai, Dongqi and Zhou, Yaozong and Zhang, Shangguang and Li, Qi and others},
  booktitle={Proceedings of the 29th Annual International Conference on Mobile Computing and Networking},
  year={2023}
}

@article{yi2023fedlora,
  title={{FedLoRA}: Model-Heterogeneous Personalized Federated Learning with {LoRA} Tuning},
  author={Yi, Liping and Wang, Gang and Liu, Xiaoguang and Shi, Zhuan and Yu, Han},
  journal={arXiv preprint arXiv:2310.13283},
  year={2023}
}

@inproceedings{wang2020fednova,
  title={Tackling the Objective Inconsistency Problem in Heterogeneous Federated Optimization},
  author={Wang, Jianyu and Liu, Qinghua and Liang, Hao and Joshi, Gauri and Poor, H. Vincent},
  booktitle={Advances in Neural Information Processing Systems},
  pages={7611--7623},
  year={2020}
}

@article{sattler2020clustered,
  title={Clustered Federated Learning: Model-Agnostic Distributed Multitask Optimization Under Privacy Constraints},
  author={Sattler, Felix and M{\"u}ller, Klaus-Robert and Samek, Wojciech},
  journal={IEEE Transactions on Neural Networks and Learning Systems},
  volume={32},
  number={8},
  pages={3710--3722},
  year={2021}
}

@inproceedings{marfoq2021federated,
  title={Federated Multi-Task Learning under a Mixture of Distributions},
  author={Marfoq, Othmane and Neglia, Giovanni and Bellet, Aur{\'e}lien and Kameni, Laetitia and Vidal, Richard},
  booktitle={Advances in Neural Information Processing Systems},
  pages={15434--15447},
  year={2021}
}

@inproceedings{tan2024fedproto,
  title={FedProto: Federated Prototype Learning across Heterogeneous Clients},
  author={Tan, Yue and Long, Guodong and Liu, Lu and Zhou, Tianyi and Lu, Qinghua and Jiang, Jing and Zhang, Chengqi},
  booktitle={AAAI Conference on Artificial Intelligence},
  pages={8432--8440},
  year={2022}
}

@article{zhang2024pfedhyper,
  title={{pFedHyper}: Personalized Federated Learning via Hypernetworks},
  author={Zhang, Xinyang and Chen, Hong and Zhao, Zhiwei and Guo, Jianwei and Chen, Tianming},
  journal={arXiv preprint arXiv:2402.04038},
  year={2024}
}

@inproceedings{li2021ditto,
  title={Ditto: Fair and Robust Federated Learning Through Personalization},
  author={Li, Tian and Hu, Shengyuan and Beirami, Ahmad and Smith, Virginia},
  booktitle={International Conference on Machine Learning},
  pages={6357--6368},
  year={2021},
  organization={PMLR}
}

@inproceedings{oh2022fedbabu,
  title={{FedBABU}: Toward Enhanced Representation for Federated Image Classification},
  author={Oh, Jaehoon and Kim, SangMook and Yun, Se-Young},
  booktitle={International Conference on Learning Representations},
  year={2022}
}

@inproceedings{kairouz2021practical,
  title={Practical and Private (Deep) Learning Without Sampling or Shuffling},
  author={Kairouz, Peter and McMahan, Brendan and Song, Shuang and Thakkar, Om and Bonawitz, Keith and Ramage, Daniel},
  booktitle={International Conference on Machine Learning},
  year={2021}
}

@inproceedings{mironov2017renyi,
  title={R{\'e}nyi Differential Privacy},
  author={Mironov, Ilya},
  booktitle={IEEE 30th Computer Security Foundations Symposium (CSF)},
  pages={263--275},
  year={2017},
  organization={IEEE}
}

@inproceedings{liu2021conflict,
  title={Conflict-Averse Gradient Descent for Multi-task Learning},
  author={Liu, Bo and Liu, Xingchao and Jin, Xiaojie and Stone, Peter and Liu, Qiang},
  booktitle={Advances in Neural Information Processing Systems},
  volume={34},
  pages={18878--18890},
  year={2021}
}

@inproceedings{sun2024layerwise,
  title={Towards Layer-Wise Personalized Federated Learning: Adaptive Layer Disentanglement via Conflicting Gradients},
  author={Sun, Jianqing and Xu, Shuo and Ma, Lixu and Cao, Wei and Gao, Wei},
  booktitle={International Conference on Learning Representations},
  year={2024}
}

@inproceedings{chen2025blackbox,
  title={Generalized and Personalized Federated Learning with Black-Box Foundation Models via Orthogonal Transformations},
  author={Chen, Shuai and Yan, Hanlin and Li, Tong and Wang, Tianhao},
  booktitle={International Conference on Machine Learning},
  year={2025}
}

@article{liu2019roberta,
  title={RoBERTa: A Robustly Optimized BERT Pretraining Approach},
  author={Liu, Yinhan and Ott, Myle and Goyal, Naman and Du, Jingfei and Joshi, Mandar and Chen, Danqi and Levy, Omer and Lewis, Mike and Zettlemoyer, Luke and Stoyanov, Veselin},
  journal={arXiv preprint arXiv:1907.11692},
  year={2019}
}

@inproceedings{li2021fedbn,
  title={FedBN: Federated Learning on Non-IID Features via Local Batch Normalization},
  author={Li, Xiaoxiao and Jiang, Meirui and Zhang, Xiaofei and Kamp, Michael and Dou, Qi},
  booktitle={International Conference on Learning Representations},
  year={2021}
}
